\documentclass[final,onefignum,onetabnum]{siamart171218}

\usepackage{lipsum}
\usepackage{amsfonts}
\usepackage{graphicx}
\usepackage{epstopdf}
\usepackage{algorithmic}

\ifpdf
  \DeclareGraphicsExtensions{.eps,.pdf,.png,.jpg}
\else
  \DeclareGraphicsExtensions{.eps}
\fi

\newsiamremark{remark}{Remark}
\newsiamremark{assumption}{Assumption}
\crefname{assumption}{Assumption}{Assumptions}
\newsiamthm{claim}{Claim}

\headers{Convergence and Dynamical Behavior of the ADAM Algorithm}{A. Barakat and P. Bianchi}

\title{Convergence and Dynamical Behavior of the ADAM Algorithm for Non-Convex Stochastic Optimization}% \thanks{Submitted to the editors. \funding{}}}

\author{Anas Barakat \and Pascal Bianchi
\footnotemark[1]\thanks{LTCI, T\'el\'ecom Paris, Institut polytechnique de Paris, France. (\email{firstname.name@telecom-paristech.fr})
  }}

\usepackage{amsopn}
\DeclareMathOperator{\diag}{diag}

\usepackage{preamble}

\ifpdf
\hypersetup{
  pdftitle={Convergence and Dynamical Behavior of the ADAM Algorithm},
  pdfauthor={A. Barakat and P. Bianchi}
}
\fi

\begin{document}

\maketitle

% REQUIRED
\begin{abstract}
  \adam\ is a popular variant of stochastic gradient descent for
  finding a local minimizer of a function.
  In the constant stepsize regime, assuming that the objective function is differentiable and non-convex,
  we establish the convergence in the long run of the iterates to a stationary point under a stability condition.
  The key ingredient is the introduction of a continuous-time
  version of \adam, under the form of a non-autonomous ordinary
  differential equation.
  This continuous-time system is a relevant approximation of the \adam\
  iterates, in the sense that the interpolated \adam\ process
  converges weakly towards the solution to the ODE.
  The existence and the uniqueness of the
  solution are established. We further show the convergence of the solution
  towards the critical points of the objective function
  and quantify its convergence rate under a \L{}ojasiewicz assumption.
  Then, we introduce a novel decreasing stepsize version of \adam\,.
  Under mild assumptions, it is shown that the iterates are almost surely bounded
  %It is shown that the iterates are almost surely bounded
  and converge almost surely to critical points of the objective function.
  Finally, we analyze the fluctuations of the algorithm by means of a conditional central limit theorem.
\end{abstract}

% REQUIRED
% REQUIRED

\section{Introduction}

Consider the problem of finding a local minimizer of the expectation
$F(x)\eqdef\bE(f(x,\xi))$ w.r.t. $x\in \bR^d$, where $f(\,.\,,\xi)$ is
a possibly non-convex function depending on some random
variable~$\xi$.
The distribution of $\xi$ is assumed unknown, but revealed online by
the observation of iid copies $(\xi_n:n\geq 1)$ of the r.v.~$\xi$.
Stochastic gradient descent (SGD) is the most classical algorithm
to search for such a minimizer.
Variants of SGD which include an inertial term have also become very popular.
In these methods, the update rule depends on a parameter called
the \emph{learning rate}, which is generally assumed constant or vanishing.
These algorithms, although widely used, have at least two limitations.
First, the choice of the learning rate is generally difficult; large learning rates result in
large fluctuations of the estimate, whereas small learning rates
induce slow convergence.  Second, a common learning rate is used for
every coordinate despite the possible discrepancies in the values of
the gradient vector's coordinates.

To alleviate these limitations, the popular
\adam\ algorithm \cite{kingma2014adam} adjusts the learning rate
coordinate-wise, as a function of the past values of the squared
gradient vectors' coordinates. The algorithm thus combines the assets
of inertial methods with an adaptive per-coordinate learning rate
selection. Finally, the algorithm includes a so-called
\emph{bias correction} step. Acting on the current estimate of the gradient vector,
this step is especially useful during the early iterations.

Despite the growing popularity of the algorithm, only few works
investigate its behavior from a theoretical point of
view (see the discussion in Section~\ref{sec:related-works}).
The present paper studies the convergence of \adam\ from a dynamical system viewpoint.\\
%\clearpage

\noindent{\bf Contributions}
\begin{itemize}[leftmargin=*]
\item We introduce a continuous-time version of the \adam\ algorithm under the form
  of a non-autonomous ordinary differential equation (ODE).
  Building on the existence of an explicit Lyapunov function for the ODE,
  we show the existence of a unique global solution to the ODE. This first result
  turns out to be non-trivial due to the irregularity of the vector field.
  We then establish the convergence of the continuous-time \adam\
  trajectory to the set of critical points of the objective function $F$.
The proposed continuous-time version of \adam\ provides useful
  insights on the effect of the bias correction step.  It is shown
  that, close to the origin, the objective function~$F$ is
  non-increasing along the \adam\ trajectory,
  suggesting that early iterations of \adam\ can only improve the initial guess.

\item Under a \L{}ojasiewicz-type condition,
  we prove that the solution to the ODE converges to a single critical point of the objective
  function $F$. We provide convergence rates in this case.

\item
  In discrete time, we first analyze the \adam\ iterates in the constant stepsize
  regime as originally introduced in \cite{kingma2014adam}.
  In this work, it is shown that the discrete-time \adam\ iterates shadow the
  behavior of the non-autonomous ODE in the asymptotic regime where
  the stepsize parameter $\gamma$ of \adam\ is small.  More
  precisely, we consider the interpolated process $\sz^\gamma(t)$
  which consists of a piecewise linear interpolation of the \adam\ iterates.
  The random process $\sz^\gamma$ is indexed by the parameter~$\gamma$, which is
  assumed constant during the whole run of the algorithm.
  In the space of continuous functions on $[0,+\infty)$ equipped
  with the topology of uniform convergence on compact sets,
  we establish that  $\sz^\gamma$ converges in probability
  to the solution to the non-autonomous ODE when $\gamma$ tends to zero.

\item Under a stability condition, we prove the asymptotic ergodic convergence
  of the probability of the discrete-time \adam\ iterates to approach the set of critical
  points of the objective function in the doubly asymptotic regime
  where $n\to\infty$ then $\gamma\to 0$.
\item
Beyond the original constant stepsize \adam\,, we propose
a decreasing stepsize version of the algorithm.
We provide sufficient conditions ensuring the stability and the almost sure convergence of the iterates towards
the critical points of the objective function.
\item  We establish a convergence rate of the stochastic iterates of the decreasing stepsize algorithm under the
form of a conditional central limit theorem.

\end{itemize}
We claim that our analysis can be easily extended to
other adaptive algorithms such as e.g. \textsc{RmsProp} or
\textsc{AdaGrad} \cite{tieleman2012lecture,duchi2011adaptive}
and \textsc{AmsGrad} (see Section~\ref{sec:related-works}).\\%\medskip

The paper is organized as follows.
In Section~\ref{sec:adam}, we present
the \adam\ algorithm and the main assumptions.
Our main results are stated in Sections~\ref{sec:continuous_time} to \ref{sec:discrete_decreasing}.
%in~\Cref{sec:continuous_time,sec:discrete,sec:discrete_decreasing}.
We provide a review of related works in Section~\ref{sec:related-works}.
The rest of the paper addresses the proofs of our results
(Sections~\ref{sec:proofs_cont_time} to \ref{sec:proofs_sec_discrete_decreasing}).
\hfill\\

\noindent{\bf Notations}. If $x$, $y$ are two vectors on $\bR^d$ for some $d\geq 1$, we denote by $x \odot y$, $x^{\odot 2}$, $x/y$, $|x|$, $\sqrt{|x|}$ the vectors
on $\bR^d$ whose $i$-th coordinates are respectively given by $x_iy_i$, $x_i^2$, $x_i/y_i$, $|x_i|$, $\sqrt{|x_i|}$.
Inequalities of the form $x\leq y$ are read componentwise.
Denote by $\|\cdot\|$ the standard Euclidean norm.
For any vector $v\in (0,+\infty)^d$, write $\|x\|^2_v = \sum_i v_i x_i^2$.
Notation $A^T$ represents the transpose of a matrix $A$.
If  $z\in \bR^d$ and $A$ is a non-empty subset of $\bR^d$,
we use the notation $\mathsf d(z,A) \eqdef \inf\{ \|z-z'\| :z'\in A\}$.
If $A$ is a set, we denote  by $\1_A$ the function equal to one on that set and to zero elsewhere.
 We denote by $C([0,+\infty),\bR^d)$ the space of continuous functions from $[0,+\infty)$ to $\bR^d$ endowed with the topology of
uniform convergence on compact intervals.

\section{The \adam\ Algorithm}
\label{sec:adam}

\subsection{Algorithm and Assumptions}

Let $(\Omega,\cF,\bP)$ be a probability space, and let
$(\Xi,\mathfrak{S})$ denote an other measurable space.  Consider a
measurable map $f:\bR^d\times \Xi\to \bR$, where $d$ is an integer.
For a fixed value of $\xi$, the mapping $x\mapsto f(x,\xi)$ is
supposed to be differentiable, and its gradient w.r.t. $x$ is denoted
by $\nabla f(x,\xi)$.  Define $\cZ\eqdef \bR^d\times\bR^d\times
\bR^d$, $\cZ_{+}\eqdef \bR^d\times\bR^d\times [0,+\infty)^d$ and
$\cZ_{+}^*\eqdef \bR^d\times\bR^d\times (0,+\infty)^d$. \adam\
generates a sequence $z_n\eqdef (x_n,m_n,v_n)$ on
$\cZ_+$ given by Algorithm~\ref{alg:adam}.
\begin{algorithm}[tb]
   \caption{\bf \adam$(\gamma,\alpha,\beta,\varepsilon)$.}
   \label{alg:adam}
\begin{algorithmic}
%   \STATE {\bfseries Input:} data $x_i$, number of iterations $n_{iter}$.
%   \STATE {\bfseries Parameters:} $\gamma>0,\varepsilon>0$, $(\alpha,\beta) \in [0,1)^2$.
   \STATE {\bfseries Initialization:} $x_0\in \bR^d, m_0=0$, $v_0=0$.
   \FOR{$n=1$ {\bfseries to} $n_{\text{iter}}$}
   \STATE $m_n = \alpha m_{n-1} + (1-\alpha)  \nabla f(x_{n-1},\xi_n)$
   \STATE $v_n = \beta v_{n-1} + (1-\beta) \nabla f(x_{n-1},\xi_n)^{\odot 2}$
   \STATE $\hat m_{n} = m_{n}/(1-\alpha^{n})$
   \COMMENT{bias correction step}
   \STATE $\hat v_{n} = v_{n}/(1-\beta^{n})$
   \COMMENT{bias correction step}
   \STATE $x_{n} = x_{n-1} - \gamma \hat m_{n} / (\varepsilon+\sqrt{\hat v_{n}}) \,.$
   \ENDFOR
\end{algorithmic}
\end{algorithm}
%Consider the iterated sequence $z_n\eqdef  (x_n,m_n,v_n)$ for all $n$. The latter
It satisfies:
$
z_n=T_{\gamma,\alpha,\beta}(n,z_{n-1},\xi_n)\,,
$
for every $n\geq 1$, where for every $z=(x,m,v)$ in $\cZ_+$, $\xi\in \Xi$,
%$T_{\gamma,\alpha,\beta}(n,z,\xi) \eqdef$
\begin{equation}
  T_{\gamma,\alpha,\beta}(n,z,\xi) \eqdef
  \begin{pmatrix}
 x -\frac{\gamma (1-\alpha^{n})^{-1}(\alpha m+(1-\alpha)\nabla f(x,\xi))}{ \varepsilon+{(1-\beta^{n})^{-1/2} (\beta v+(1-\beta)\nabla f(x,\xi)^{\odot 2})^{1/2}}} \\
\alpha m + (1-\alpha) \nabla f(x,\xi) \\
\beta v + (1-\beta) \nabla f(x,\xi)^{\odot 2}
\end{pmatrix}\,.\label{eq:T}
\end{equation}

\begin{remark}
\label{rem:debiasing}
\quad\quad The iterates $z_n$ form a non-homogeneous Markov chain, because $T_{\gamma,\alpha,\beta}(n,z,\xi)$ depends on $n$.
This is due to the so-called debiasing step, which consists of replacing $m_n,v_n$ in Algorithm~\ref{alg:adam}
by their ``debiased'' versions $\hat m_n,\hat v_n$. The motivation becomes clear when expanding the expression:
\begin{equation*}
\hat m_n = \frac{m_n}{1-\alpha^n} = \frac {1-\alpha}{1-\alpha^n}\sum_{k=0}^{n-1}\alpha^k \nabla f(x_k,\xi_{k+1})\,. \label{eq:debiasing}
\end{equation*}
From this equation, it is observed that,  $\hat m_n$ forms a convex combination of the past gradients.
This is unlike $m_n$, which may be small during the first iterations.

\end{remark}

\begin{assumption} \label{hyp:model}
The mapping $f:\bR^d\times\Xi\to \bR$ satisfies the following.
  \begin{enumerate}[{\sl i)}]
  \item For every $x\in \bR^d$, $f(x,\,.\,)$ is $\mathfrak{S}$-measurable.
  \item For almost every $\xi$, the map $f(\,.\,,\xi)$ is continuously differentiable.  \label{hyp:dif}
  \item There exists $x_*\in \bR^d$ s.t. $\bE(|f(x_*,\xi)|)<\infty$ and  $\bE(\|\nabla f(x_*,\xi)\|^2)<\infty$.
  \item For every compact subset $K\subset \bR^d$, there exists $L_K>~0$ such that
   for every $(x,y)\in K^2$, $\bE(\|\nabla f(x,\xi)-\nabla f(y,\xi)\|^2)\leq L_K^2\|x-y\|^2$.
  \end{enumerate}
\end{assumption}

Under Assumption~\ref{hyp:model}, it is an easy exercise to show that the mappings
$F:\bR^d\to\bR$ and $S:\bR^d\to\bR^d$, given by:
\begin{equation}
  F(x) \eqdef \bE(f(x,\xi)) \quad \text{and} \quad S(x) \eqdef \bE(\nabla f(x,\xi)^{\odot 2})\label{eq:F_and_S}
\end{equation}
% \begin{align}
%   F(x)&\eqdef \bE(f(x,\xi)) \label{eq:F}\\
%   S(x) &\eqdef \bE(\nabla f(x,\xi)^{\odot 2}) \label{eq:S}
% \end{align}
are well defined; $F$ is continuously differentiable and by
Lebesgue's dominated convergence theorem, $\nabla F(x) = \bE(\nabla
f(x,\xi))$ for all $x$. Moreover, $\nabla F$ and $S$ are locally Lipschitz continuous.
%Regarding Assumption~\ref{hyp:model}-\ref{hyp:dif}), note that the case of non-differentiable $f(\,.\,,\xi)$ (for almost every $\xi$)
%is as well interesting in practice, but the analysis is harder and left for future works.
\begin{assumption}
  \label{hyp:coercive}
$F$ is coercive.
\end{assumption}
\begin{assumption}
  \label{hyp:S>0}
For every $x\in \bR^d$, $S(x)>0$.
\end{assumption}
It follows from our assumptions that the set of critical points of
$F$, denoted by $$\cS \eqdef \nabla F^{-1}(\{0\}),$$ is non-empty.
Assumption~\ref{hyp:S>0} means that there is \emph{no} point $x\in
\bR^d$ satisfying $\nabla f(x,\xi) = 0$ with probability one (w.p.1). This is
a mild hypothesis in practice.

\subsection{Asymptotic Regime}

% In stochastic approximation theory \cite{kus-yin-(livre)03}, there are typically two frameworks
% to characterize the asymptotic behavior of the algorithm : vanishing step size
% and constant step size.
% The former assumes that the step size parameter $\gamma$
% depends on $n$, and vanishes as $n$ tends to infinity. This
% contradicts the spirit of \adam, because in practice, $\gamma$
% is fixed when the algorithm starts.
%In this paper, we focus on the constant step size regime, where
We address the constant stepsize regime, where
$\gamma$ is fixed along the iterations (the default value recommended
in \cite{kingma2014adam} is $\gamma = 0.001$). As opposed to the decreasing stepsize context,
%(\emph{i.e.}, when $\gamma$ vanishes along the iteration index $n$),
the sequence $z_n^\gamma\eqdef z_n$ \emph{cannot} in general converge as $n$ tends to infinity, in an almost sure sense.
Instead, we investigate the asymptotic behavior of the {family} of processes $(n\mapsto z_n^\gamma)_{\gamma>0}$
indexed by $\gamma$, in the regime where $\gamma\to0$. We use the so-called ODE method (see e.g. \cite{ben-(cours)99}). %\cite{ljung1977analysis,kus-yin-(livre)03}. %,ben-sch-00}.
The interpolated process $\sz^{\gamma}$ is the
piecewise linear function defined on $[0,+\infty)\to \cZ_+$
for all $t \in [n\gamma ,(n+1)\gamma)$  by:
\begin{equation}
\sz^{ \gamma}(t) \eqdef z_{n}^{ \gamma} + (z_{n+1}^{ \gamma}-z_{n}^{ \gamma})\left(\frac {t-n\gamma}{\gamma}\right)\,.
\label{eq:interpolated-process}
\end{equation}
We establish the convergence in probability of the
family of random processes $(\sz^{ \gamma})_{\gamma>0}$ as $\gamma$ tends to zero, towards a
deterministic continuous-time system defined by an ODE.
The latter ODE, which we provide below at Eq.~(\ref{eq:ode}), will be referred to
as the continuous-time version of \adam.

Before describing the ODE, we need to be more specific about our
asymptotic regime. As opposed to SGD, \adam\ depends on
two parameters $\alpha$, $\beta$, in addition to the stepsize $\gamma$.
The  paper \cite{kingma2014adam} recommends choosing the constants $\alpha$ and $\beta$ close to one
(the default values $\alpha=0.9$ and $\beta=0.999$ are suggested).
It is thus legitimate to assume that
$\alpha$ and $\beta$ tend to one, as $\gamma$ tends to zero.
We set  $\alpha \eqdef \bar \alpha(\gamma)$ and $\beta\eqdef\bar\beta(\gamma)$, where
%$\bar \alpha$ and $\bar \beta$ are some mappings on $\bR_+\to [0,1)$ s.t.
$\bar \alpha(\gamma)$ and $\bar\beta(\gamma)$ converge to one as $\gamma\to 0$.
%  in the regime where the step size $\gamma$ is fixed amongst the iterations, but small.
% Note that the case of decreasing step size is as well interesting (and in fact somewhat easier to deal with in terms of mathematical analysis),
% but in practice Adam is generally used with a constant learning rate $\gamma$, hence motivating our analysis.
% We shall also assume that the constants $\alpha$, $\beta$ are fixed, but close to one.
%On the otherhand, the constant $\varepsilon>0$ is supposed to be fixed.
%Going one step further, we make the following assumption.
\begin{assumption}
\label{hyp:alpha-beta}
The functions $\bar \alpha:\bR_+\to [0,1)$ and $\bar \beta:\bR_+\to [0,1)$ are
s.t. the following limits exist:
  \begin{equation}
    a\eqdef \lim_{\gamma_\downarrow 0} \frac{1-\bar\alpha(\gamma)}\gamma ,\quad b\eqdef  \lim_{\gamma_\downarrow 0}\frac{1-\bar\beta(\gamma)}\gamma\,.\label{eq:regime}
  \end{equation}
Moreover, $a>0$, $b>0$, and the following condition holds: $b\leq 4a\,.$
\end{assumption}
Note that the condition $b\leq 4a$ is compatible with the default settings recommended
by \cite{kingma2014adam}.
In our model, we shall now replace the map $T_{\gamma,\alpha,\beta}$ by $T_{\gamma,\bar \alpha(\gamma),\bar \beta(\gamma)}$.
Let $x_0\in \bR^d$ be fixed. For any fixed $\gamma>0$, we define the sequence $(z_n^\gamma)$ generated by
\adam\ with a fixed stepsize~$\gamma>0$:
\begin{equation}
z_n^\gamma \eqdef T_{\gamma,\bar \alpha(\gamma),\bar \beta(\gamma)}(n,z^\gamma_{n-1},\xi_n)\,,\label{eq:znT}
\end{equation}
the initialization being chosen as $z_0^\gamma=(x_0,0,0)$.

\section{Continuous-Time System}
\label{sec:continuous_time}

\subsection{Ordinary Differential Equation}

In order to gain insight into the behavior of the sequence $(z_n^\gamma)$
defined by (\ref{eq:znT}),
it is convenient to rewrite the \adam\ iterations under the following equivalent form, for every $n\geq 1$:
\begin{equation}
z_n^{\gamma} = z^{\gamma}_{n-1} + \gamma h_{\gamma}(n,z^{\gamma}_{n-1}) + \gamma \Delta^\gamma_{n}\,,\label{eq:RM}
\end{equation}
where we define for every $\gamma>0$, $z\in \cZ_+$,
\begin{equation}
h_\gamma(n,z) \eqdef \gamma^{-1}\bE(T_{\gamma,\bar\alpha(\gamma),\bar\beta(\gamma)}(n,z,\xi)-z)\,,\label{eq:hgamma}
\end{equation}
and where $\Delta^\gamma_{n}\eqdef \gamma^{-1}(z_n^{\gamma} - z^{\gamma}_{n-1}) - h_{\gamma}(n,z^{\gamma}_{n-1})$.
Note that $(\Delta^\gamma_{n})$ is a martingale increment noise sequence in the sense that
$\bE(\Delta^\gamma_{n}|\cF_{n-1}) = 0$ for all $n\geq 1$, where $\cF_n$ stands for the $\sigma$-algebra generated
by the r.v. $\xi_1,\dots,\xi_n$.
Define the map $h:(0,+\infty)\times \cZ_+\to\cZ$ for all $t>0$, all $z=(x,m,v)$ in $\cZ_+$ by:
\begin{equation}
  \label{eq:h}
  h(t,z)= \begin{pmatrix}
  -\frac{(1-e^{-at})^{-1}m}{\varepsilon+\sqrt{(1-e^{-bt})^{-1} v}} \\
a (\nabla F(x)-m) \\
b (S(x)-v)
\end{pmatrix} \,,
\end{equation}
where $a,b$ are the constants defined in Assumption~\ref{hyp:alpha-beta}.
We prove that, for any fixed $(t,z)$, the quantity $h(t,z)$ coincides with the limit
of $h_\gamma(\lfloor t/\gamma\rfloor,z)$ as $\gamma\downarrow 0$. This remark along with Eq.~(\ref{eq:RM})
suggests that, as $\gamma\downarrow 0$, the interpolated process $\sz^\gamma$ shadows the non-autonomous differential equation
\begin{equation}
  \begin{array}[h]{l}
\dot z(t) = h(t, z(t))\,.
\end{array}
\tag{ODE}
\label{eq:ode}
\end{equation}

\subsection{Existence, Uniqueness, Convergence}
\label{subsec:odeanalysis}

Since $h(\,.\,,z)$ is non-continuous
at point zero for a fixed $z\in \cZ_+$, and since $h(t,\,.\,)$
is not locally Lipschitz continuous for a fixed~$t~>~0$,
the existence and uniqueness of the solution to (\ref{eq:ode}) do not stem directly from off-the-shelf theorems.
%cannot be directly solved using off-the-shelf theorems.
%As a consequence, there is no off-the-shelf theorem garanteeing the existence and the uniqueness of a global solution.

%In this paragraph, we study the non autonomous equation~(\ref{eq:ode}).
Let {\color{black}$x_0$ be fixed}.
A continuous map $z:[0,+\infty)\to\cZ_+$ is said to be a global solution to (\ref{eq:ode}) with initial condition {\color{black}$(x_0,0,0)$}
if $z$ is continuously differentiable on $(0,+\infty)$, if Eq.~(\ref{eq:ode})
holds for all $t>0$, and if {\color{black}$z(0)=(x_0,0,0)$}.

\begin{theorem}[Existence and uniqueness]
  \label{th:exist-unique}
  %Let \cref{hyp:model,hyp:coercive,hyp:S>0,hyp:alpha-beta} hold true.
  Let Assumptions~\ref{hyp:model} to \ref{hyp:alpha-beta} hold true.
%Let  {\revision $z_0 = (x_0,m_0,v_0) \in \cZ_+$}.
  There exists a unique global solution $z:[0,+\infty)\to\cZ_+$
  to~(\ref{eq:ode}) with
  initial condition $(x_0,0,0)$.
% Similarly, there exists
%    a unique global solution $z:[0,+\infty)\to\cZ_+$
%   to~(\ref{eq:ode-a}) with initial condition $(x_0,m_0,v_0)$.
  Moreover, $z([0,+\infty))$ is a bounded subset of $\cZ_+$.
\end{theorem}
On the other hand, we note that a solution may not exist for
an initial point$(x_0,m_0,v_0)$ with arbitrary (non-zero) values of $m_0, v_0$.

\begin{theorem}[Convergence]
\label{th:cv-adam}
%Let \cref{hyp:model,hyp:coercive,hyp:S>0,hyp:alpha-beta} hold true.
Let Assumptions~\ref{hyp:model} to \ref{hyp:alpha-beta} hold true.
Assume that $F(\cS)$ has an empty interior.
Let $z:t\mapsto (x(t),m(t),v(t))$ be
the global solution to~(\ref{eq:ode}) with the initial condition
$(x_0,0,0)$.\,
Then, the set $\cS$ is non-empty and $\lim_{t\to\infty} \sd(x(t),\cS) =0$,
%where $\mathsf d(x(t),\cS) \eqdef \inf\{\mathsf \|x(t)-x\|: x\in \cS\}$,
 $\lim_{t\to\infty}m(t)=0$, $\lim_{t\to\infty}S(x(t))-v(t)=0$.
 %{\color{red} $\sd$ n'est plus défini il me semble.}
 %{\color{blue} $\sd$ est définie dans la section notations au début de l'article.}
\end{theorem}
% \noindent Denote by $z(t) = (x(t),m(t),v(t))$
% the global solution to~(\ref{eq:ode}) issued from $(x_0,0,0)$.\\

\noindent{\bf Lyapunov function.} The proof of Th.~\ref{th:exist-unique} %(see Section~\ref{sec:proofs_cont_time})
  relies on the existence of a Lyapunov function for the non-autonomous
  equation~(\ref{eq:ode}).  Define  $V:(0,+\infty)\times \cZ_+\to \bR$ by
% By Lyapunov function, we mean a continuous
%   function $V:(0,+\infty)\times \cZ_+\to \bR$ s.t.  $t\mapsto V(t,z(t))$
%   is decreasing on $(0,+\infty)$. Such a function $V$ is given by:
  \begin{equation}
    \label{eq:V}
    V(t,z)\eqdef F(x)+\frac 12 \left\|m\right\|^2_{U(t,v)^{-1}}\,,
  \end{equation}
  for every $t>0$ and every $z=(x,m,v)$ in $\cZ_+$, where
  $U:(0,+\infty)\times [0,+\infty)^d\to \bR^d$ is the map given by:
  \begin{equation}
    \label{eq:U}
    U(t,v) \eqdef a(1-e^{-at})\left(\varepsilon+\sqrt{\frac{v}{1-e^{-bt}}}\right)\,.
  \end{equation}
Then, $t\mapsto V(t,z(t))$ is decreasing if $z(\cdot)$ is the global solution to~(\ref{eq:ode}).

\noindent{\bf Cost decrease at the origin.} As $F$ itself is not a Lyapunov function for~(\ref{eq:ode}),
there is no guarantee that $F(x(t))$ is decreasing w.r.t. $t$.
Nevertheless, the statement holds at the origin. Indeed, it can be shown that
$\lim_{t\downarrow 0}V(t,z(t))=F(x_0)$ (see Prop.~\ref{prop:adam-bounded}). As a consequence,
\begin{equation}
  \forall t\geq 0,\ F(x(t))\leq F(x_0)\,.
\label{eq:costdecrease}
\end{equation}

In other words, the (continuous-time) \adam\ procedure
\emph{can only improve} the initial guess $x_0$.
This is the consequence of the so-called bias correction steps in \adam\ (see Algorithm~\ref{alg:adam})\,.
%\emph{i.e.},
%the fact that $m_n$ and $v_n$ are respectively divided by $(1-\alpha^n)$ and $(1-\beta^n)$
%before being injected in the update of the iterate $x_n$.
If these debiasing steps were deleted in the \adam\ iterations,
the early stages of the algorithm could degrade the initial estimate $x_0$.
%{\color{black} and \Cref{eq:costdecrease} does not hold true for a trajectory inherited from~(\ref{eq:ode-a}).}

\noindent{\bf Derivatives at the origin.}
The proof of Th.~\ref{th:exist-unique} reveals that the initial derivative is given
by $\dot x(0) = -\nabla F(x_0)/(\varepsilon+\sqrt{S(x_0)})$ (see Lemma~\ref{lem:m-v-derivables-en-zero}).
%Again, this is a remarkable feature of \adam\,.
In the absence of debiasing steps, the initial derivative $\dot x(0)$ would be a function of the initial
parameters $m_0$, $v_0$, and the user would be required to tune these hyperparameters.
No such tuning is required thanks to the debiasing step. %The initial derivative is naturally fixed.
When $\varepsilon$ is small and when the variance of $\nabla f(x_0,\xi)$ is small (\emph{i.e.}, $S(x_0)\simeq \nabla F(x_0)^{\odot 2}$),
the initial derivative $\dot x(0)$ is approximately equal to $-\nabla F(x_0)/|\nabla F(x_0)|$.
This suggests that in the early stages of the algorithm, the \adam\ iterations
are comparable to the \emph{sign} variant of the gradient descent, the properties of which were
discussed in previous works, see \cite{balles2018dissecting}.%,pmlr-v80-bernstein18a}.

\subsection{Convergence rates}

In this paragraph, we establish the convergence to a single critical
point of $F$ and quantify the convergence rate, using the following
assumption \cite{lojasiewicz1963}.
\begin{assumption}[\L{}ojasiewicz property]
  \label{hyp:lojasiewicz_prop}
For any  $x^* \in {\mathcal S}$, there exist % verifiying $\nabla F(a)= 0$, there exist
$c >0\,, \sigma >0$ and $\theta \in (0,\frac 12]$ s.t.
\begin{equation}
  \label{eq:lojasiewicz}
\forall x \in \bR^d \,\,\text{s.t}\,\,\|x-x^*\|\leq \sigma\,,\quad  \|\nabla F(x)\| \geq c |F(x) - F(x^*)|^{1-\theta}\,.
\end{equation}
\end{assumption}
Assumption~\ref{hyp:lojasiewicz_prop} holds for real-analytic functions
and semialgebraic functions.
We refer to %\cite{harauxjendoubi2015,attouch2009convergence,attouch2010proximal,bolte2014proximal}
\cite{harauxjendoubi2015,attouch2009convergence,bolte2014proximal}
for a discussion and a review of applications.
We will call any $\theta$ satisfying (\ref{eq:lojasiewicz}) for some $c,\sigma>0$,
as a \L{}ojasiewicz exponent of $f$ at $x^*$. The next result establishes the convergence
of the function $x(t)$ generated by the ODE to a single critical point of $f$,
and provides the convergence rate as a function of the \L{}ojasiewicz exponent of $f$
at this critical point. The proof is provided in subsection~\ref{sec:cont_asymptotic_rates}.
\begin{theorem} %[Asymptotic convergence rates]
  \label{thm:asymptotic_rates}
  %Let \cref{hyp:model,hyp:coercive,hyp:S>0,hyp:alpha-beta} and \cref{hyp:lojasiewicz_prop} hold true.
  Let Assumptions~\ref{hyp:model} to \ref{hyp:alpha-beta} and \ref{hyp:lojasiewicz_prop} hold true.
  Assume that $F(\cS)$ has an empty interior. Let $x_0 \in \bR^d$ and let $z:t\mapsto (x(t),m(t),v(t))$ be
  the global solution to~(\ref{eq:ode}) with initial condition $(x_0,0,0)$.
  Then, there exists $x^* \in \cS$ such that $x(t)$ converges to $x^*$ as $t \to +\infty$.

  Moreover, if $\theta\in (0,\frac 12]$ is a \L{}ojasiewicz exponent of $f$ at $x^*$,
 there exists a constant $C >0$ s.t. for all $t\geq 0$,
  \begin{align*}
    \|x(t)-x^*\| &\leq C t^{-\frac{\theta}{1-2\theta}}\,,\quad \text{if}\,\,\, 0<\theta<\frac{1}{2}\,,\\
    \|x(t)-x^*\| &\leq C e^{- \delta t} \,,\quad \text{for some}\,\, \delta >0\,\,\text{if}\,\, \theta = \frac 12 \,.
  \end{align*}
\end{theorem}

\section{Discrete-Time System: Convergence of \adam}
\label{sec:discrete}

\begin{assumption}
  \label{hyp:iid}
The sequence $(\xi_n:n\geq 1)$ is iid, with the same distribution as~$\xi$.
\end{assumption}

\begin{assumption}
  \label{hyp:moment-f}
Let $p>0$. Assume either one of the following conditions.
\begin{enumerate}[i)]
\item \label{momentegal} For every compact set $K\subset \bR^d$, $\sup_{x\in K} \bE(\|\nabla f(x,\xi)\|^{p})<\infty\,.$
\item \label{momentreinforce} For every compact set $K\subset \bR^d$,  $\exists\, p_K>p$, $\sup_{x\in K} \bE(\|\nabla f(x,\xi)\|^{p_K})<\infty.$
\end{enumerate}
\end{assumption}
The value of $p$ will be specified in the sequel, in the statement of
the results. Clearly, Assumption~\ref{hyp:moment-f}~\ref{momentreinforce} is stronger than Assumption~\ref{hyp:moment-f}~\ref{momentegal}.
We shall use either the latter or the former in our statements.

\begin{theorem}
  \label{th:weak-cv}
Let Assumptions~\ref{hyp:model} to \ref{hyp:alpha-beta} and \ref{hyp:iid} hold true. Let Assumption~\ref{hyp:moment-f}~\ref{momentreinforce}
hold with $p=2$.
Consider $x_0\in \bR^d$. For every $\gamma>0$, let $(z_n^\gamma:n\in \bN)$ be the random sequence defined by the \adam\ iterations~(\ref{eq:znT})
and $z_0^\gamma = (x_0,0,0)$. Let $\sz^\gamma$ be the corresponding interpolated process defined by Eq.~(\ref{eq:interpolated-process}).
Finally, let $z$ denote the unique global solution to (\ref{eq:ode}) issued from $(x_0,0,0)$.
Then,
$$
\forall\, T>0,\ \forall\, \delta>0,\ \ \lim_{\gamma\downarrow 0}\bP\left(\sup_{t\in[0,T]}\|\sz^\gamma(t)-z(t)\|>\delta\right)=0\,.
$$
\end{theorem}
Recall that a family of r.v. $(X_\alpha)_{\alpha\in I}$ is called \emph{bounded in probability}, or \emph{tight}, if for every
$\delta>0$, there exists a compact set $K$ s.t. $\bP(X_\alpha\in K)\geq 1-\delta$ for every $\alpha\in I$.
\begin{assumption}
\label{hyp:tight}
  There exists $\bar \gamma_0>0$ s.t. the family of r.v. $(z_n^\gamma:n\in \bN,0<\gamma<\bar \gamma_0)$ is bounded in probability.
\end{assumption}

\begin{theorem}
Consider $x_0\in \bR^d$. For every $\gamma>0$, let $(z_n^\gamma:n\in \bN)$ be the random sequence defined by the \adam\ iterations~(\ref{eq:znT})
and $z_0^\gamma = (x_0,0,0)$.
Let Assumptions~\ref{hyp:model} to \ref{hyp:alpha-beta}, \ref{hyp:iid} and \ref{hyp:tight} hold.
Let Assumption~\ref{hyp:moment-f}~\ref{momentreinforce} hold with $p=2$. Then,
for every $\delta > 0$,
  \begin{equation}
    \label{eq:long-run}
    \lim_{\gamma\downarrow 0}\limsup_{n\to\infty} \frac 1{n} \sum_{k=1}^n \bP(\sd(x_k^\gamma, \cS)>\delta) = 0\,.
  \end{equation}
\label{th:longrun}
\end{theorem}

\noindent{\bf Convergence in the long run.}
When the stepsize $\gamma$ is constant, the sequence
$(x_n^\gamma)$ cannot converge in the almost sure sense as
$n\to\infty$.
Convergence may only hold in the doubly asymptotic regime where
$n\to\infty$ then $\gamma\to 0$. %This doubly asymptotic regime is
%referred to as the convergence in the long run following the terminology of~\cite{rot-san-siam13}.
% Theorem \ref{th:longrun} establishes the convergence in the long run of the iterates
% of \adam, in an ergodic sense.

\noindent{\bf Randomization.} For every $n$, consider a r.v. $N_n$  uniformly
distributed on  $\{1,\dots,n\}$. Define $\tilde x_n^\gamma = x_{N_n}^\gamma$.
We obtain from Th.~\ref{th:longrun} that for every $\delta>0$,
$$
\limsup_{n\to\infty} \ \bP(\sd(\tilde x_n^\gamma, \cS)>\delta) \xrightarrow[\gamma\downarrow 0]{} 0\,.
$$

\noindent{\bf Relationship between discrete and continuous time \adam.}
Th.~\ref{th:weak-cv} means that the family of random processes $(\sz^\gamma:~\gamma>0)$ converges
in probability as $\gamma\downarrow 0$ towards the unique solution to~(\ref{eq:ode}) issued from $(x_0,0,0)$.
%Convergence in probability is understood here in the space~$C([0,+\infty),\cZ_+)$ of continuous functions
%on $[0,+\infty)$ endowed with the topology of uniform convergence on compact sets.
This motivates the fact that the non-autonomous system~(\ref{eq:ode})
is a relevant approximation to the behavior of the iterates $(z_n^\gamma:n\in \bN)$ for a small
value of the stepsize~$\gamma$.

\noindent{\bf Stability.} Assumption~\ref{hyp:tight} ensures %is a stability condition
%ensuring
that the iterates $z_n^\gamma$ do not explode in the long run.
A sufficient condition is for instance that
$\sup_{n,\gamma} \bE \|z_n^\gamma\|<\infty\,.$ % where $\psi$ can be any coercive function on $\cZ\to\bR$.
%Checking this assumption is not easy.
{In theory, this assumption can be difficult to verify.}
Nevertheless, in practice, a projection step on a compact set
%Note that in practice, a projection step on a compact set
can be introduced to ensure the boundedness of the estimates.
% , in which case
% Assumption~\ref{hyp:tight} is automatically satisfied.

\section{A Decreasing Stepsize \adam\, Algorithm}
\label{sec:discrete_decreasing}

\subsection{Algorithm}

\adam\, inherently uses constant stepsizes. Consequently, the
iterates~(\ref{eq:znT}) do not converge in the almost sure sense.
In order to achieve convergence, we introduce in this section a decreasing
stepsize version of \adam. The iterations are given in Algorithm~\ref{alg:adam-decreasing}.
\begin{algorithm}[tb]
   \caption{\bf \adam - decreasing stepsize $(((\gamma_n,\alpha_n,\beta_n):n\in \bN^*), \varepsilon)$.}
   \label{alg:adam-decreasing}
\begin{algorithmic}
  \STATE {\bfseries Initialization:} $x_0\in \bR^d, m_0=0$, $v_0=0$, $r_0=\bar r_0=0$.
   \FOR{$n=1$ {\bfseries to} $n_{\text{iter}}$}
   \STATE $m_n = \alpha_n m_{n-1} + (1-\alpha_n)  \nabla f(x_{n-1},\xi_n)$
   \STATE $v_n = \beta_n v_{n-1} + (1-\beta_n) \nabla f(x_{n-1},\xi_n)^{\odot 2}$
   \STATE $r_n = \alpha_n r_{n-1} + (1-\alpha_n)$
   \STATE $\bar r_n = \beta_n \bar r_{n-1} + (1-\beta_n)$
  \STATE $\hat m_{n} = m_{n}/r_n$
   \COMMENT{bias correction step}
   \STATE $\hat v_{n} = v_{n}/\bar r_n$
   \COMMENT{bias correction step}
   \STATE $x_{n} = x_{n-1} - \gamma_n \hat m_{n} / (\varepsilon+\sqrt{\hat v_{n}}) \,.$
   \ENDFOR
\end{algorithmic}
\end{algorithm}
The algorithm generates a sequence $z_n=(x_n,m_n,v_n)$ with initial point $z_0=(x_0,0,0)$,
where $x_0\in\bR^d$. Apart from the fact that the hyperparameters $(\gamma_n,\alpha_n,\beta_n)$
now depend on $n$, the main difference w.r.t Algorithm~\ref{alg:adam} lies in the expression of the
debiasing step. As noted in Remark~\ref{rem:debiasing}, the aim is to rescale $m_n$ (resp. $v_n$)
in such a way that the rescaled version $\hat m_n$ (resp. $\hat v_n$) is a convex combination of
past stochastic gradients (resp. squared gradients). While in the constant step case the rescaling
coefficient is $(1-\alpha^n)^{-1}$ (resp. $(1-\beta^n)^{-1}$), the decreasing step case requires dividing
$m_n$ by the coefficient $r_n=1-\prod_{i=1}^n\alpha_i$ (resp. $v_n$ by $\bar r_n=1-\prod_{i=1}^n\beta_i$),
which keeps track of the previous weights:
$$
\hat m_n = \frac{m_n}{r_n} = \frac{\sum_{k=1}^{n} \rho_{n,k} \nabla f(x_{k-1},\xi_{k})}{\sum_{k=1}^{n} \rho_{n,k}}\,,
$$
where for every $n,k$, $\rho_{n,k} = \alpha_n\cdots\alpha_{k+1}(1-\alpha_k)$. A similar equation holds for $\hat v_n$.

\subsection{Almost sure convergence}% towards critical points}
  \begin{assumption}[Stepsizes]
    \label{hyp:stepsizes}
  %The sequences $(\gamma_n), (\alpha_n)$ and $(\beta_n)$ satisfy the following conditions.
  The following holds.
  \begin{enumerate}[{\sl i)}]
  \item For all $n \in \bN$, $\gamma_n > 0$ and $\gamma_{n+1}/\gamma_n\to 1$,
  \item $\sum_n \gamma_n = +\infty$ and $\sum_n \gamma_n^2 <+\infty$,
  \item For all $n \in \bN$, $0 \leq \alpha_n \leq 1$ and $0 \leq \beta_n \leq 1$,
  \item There exist $a,b$ s.t. $0<b<4a$,\,
  $\gamma_n^{-1}(1-\alpha_n) \to a$ and $\gamma_n^{-1}(1-\beta_n) \to b$\,. %as $n \to \infty$\,.

   %$$
   %\lim_{n\to\infty} \gamma_n^{-1}(1-\alpha_n) = a,\quad \lim_{n\to\infty} \gamma_n^{-1}(1-\beta_n) = b\,.
   %$$
  %%\item $\forall n \in \bN,\, \gamma_n^{-1}(1-\alpha_n) = a$ and $\gamma_n^{-1}(1-\beta_n) = b$.
  \end{enumerate}
  \end{assumption}
%{\color{red} Attention j'ai supprimé l'hypothese moment-growth. Il n'y a plus de $p(\cdots)$ borné sur les compacts.}
 % \begin{assumption}%{Noise control - moment growth}
 % \label{hyp:moment-growth}
 % There exists a measurable function $p: \bR^d \to [0,+\infty)$ which is bounded on bounded sets
 % such that for all $x \in \bR^d$:
 % \[
 % \bE[\|\nabla f(x,\xi)\|^4] \leq p(x)\,.
 % \]
 % \end{assumption}
 \begin{theorem}
 \label{thm:as_conv_under_stab}
%Let \cref{hyp:model,hyp:coercive,hyp:S>0,hyp:iid,hyp:stepsizes,hyp:moment-growth} hold.
Let Assumptions~\ref{hyp:model} to \ref{hyp:S>0}, \ref{hyp:iid} and~\ref{hyp:stepsizes} hold.
Let Assumption \ref{hyp:moment-f}~\ref{momentegal} hold with $p=4$.
Assume that $F(\cS)$ has an empty interior
and that the random sequence $((x_n,m_n,v_n):n\in \bN)$
given by Algorithm~\ref{alg:adam-decreasing}
is bounded, with probability one.
Then, w.p.1, $\lim_{n\to\infty}
\sd (x_n,\cS)=0$, $lim_{n\to\infty}m_n = 0$
and $lim_{n\to\infty} (S(x_n)-v_n)=0$.
If moreover $\cS$ is finite or countable, then w.p.1, there exists $x^*\in \cS$
s.t. $\lim_{n\to\infty} (x_n,m_n,v_n) = (x^*,0,S(x^*))$.
 % the sequence $(z_n)$ defined in~(\ref{eq:algo}) converges a.s. towards the set
 % $\cE = h_\infty^{-1}(\{0\})$ of all equilibrium points of (\ref{eq:ode-a}), namely:
 % \begin{equation}
 %   \label{def:eq_points}
 % \cE = \{(x,m,v)\in \cZ_+:\nabla F(x)=0,m=0,v=S(x)\}\,.
 % \end{equation}
 %  If moreover the set of critical points of $F$ is finite or countable, then $(z_n)$ converges a.s. to a
 % point of $\cE$.
 \end{theorem}
Th.~\ref{thm:as_conv_under_stab} establishes the almost sure convergence of
$x_n$ to the set of critical points of~$F$, under the assumption that the sequence
$((x_n,m_n,v_n))$ is a.s. bounded. The next result provides a sufficient condition
under which almost sure boundedness holds. %We first need to reinforce our hypotheses.
\begin{assumption}
\label{hyp:stab}
  The following holds.
  \begin{enumerate}[{\sl i)}]
\item \label{lipschitz} $\nabla F$ is Lipschitz continuous.
\item \label{momentgrowth} There exists $C>0$ s.t. for all $x \in \bR^d$, $\bE[\|\nabla f(x,\xi)\|^2] \leq C (1+F(x))$.
\item We assume the condition:
$%$$
\lim\sup_{n\to\infty}
\left(\frac 1{\gamma_n}-\left(\frac{1-\alpha_{n+2}}{1-\alpha_{n+1}}\right)\frac 1{\gamma_{n+1}}\right)
< 2(a-\frac b4)\,,
$\\%$$
which is satisfied for instance if $b<4a$ and $1-\alpha_{n+1} = a\gamma_n$.
  \end{enumerate}
\end{assumption}
 \begin{theorem} %[Stability]
 \label{thm:stab}
 %Let \Cref{hyp:model,hyp:iid,hyp:coercive,hyp:stepsizes,hyp:moment-growth,hyp:stab} hold.
Let Assumptions~\ref{hyp:model}, \ref{hyp:coercive}, \ref{hyp:iid}, \ref{hyp:stepsizes} and \ref{hyp:stab} hold.
Let Assumption~\ref{hyp:moment-f}~\ref{momentegal} hold with $p=4$.
Then, the sequence
$((x_n,m_n,v_n):n\in \bN)$ given by Algorithm~\ref{alg:adam-decreasing} is bounded with probability one.
 \end{theorem}
% Notice that the condition $\bE[\|\nabla f(x,\xi)\|^2] \leq C (1+F(x))$ is, for instance, satisfied in the simple case
% of a linear regression model $f(x,\xi) = (\ps{\xi_1,x}-\xi_2)^2$ for a r.v. $\xi=(\xi_1,\xi_2)$ in $\bR^d\times \bR$
% having bounded fourth order moments.

\subsection{Central Limit Theorem}

 \begin{assumption} %[local regularity of the mean field $h_\infty$]
   \label{hyp:mean_field_tcl}
Let $x^*\in \cS$. There exists a neighborhood $\cV$ of $x^*$ s.t.
\begin{enumerate}[{\sl i)}]
\item $F$ is twice continuously differentiable on $\cV$,
  and the Hessian $\nabla^2 F(x^*)$ of $F$ at $x^*$ is positive
  definite.
\item $S$ is continuously differentiable on $\cV$.
\end{enumerate}
 \end{assumption}
Define $
D \eqdef \textrm{diag}\left((\varepsilon + \sqrt{S_1(x^*)})^{-1},\cdots,(\varepsilon + \sqrt{S_d(x^*)})^{-1}\right)\,.
$
Let $P$ be an orthogonal matrix s.t. the following spectral decomposition holds:
$$
D^{1/2}\nabla^2F(x^*)D^{1/2} = P\textrm{diag}(\lambda_1,\cdots,\lambda_d)P^{-1}\,,
$$
where $\lambda_1, \cdots,\lambda_d$ are the (positive) eigenvalues of
$D^{1/2} \nabla^2F(x^*)D^{1/2}$.
Define
\begin{equation}
H \eqdef
\begin{pmatrix}
  0 & -D & 0\\ a\nabla^2F(x^*) & -a I_d & 0 \\ b \nabla S(x^*) & 0 & -b I_d
\end{pmatrix}\label{eq:H}
\end{equation}
where $I_d$ represents the $d\times d$ identity matrix and $\nabla S(x^*)$ is the
Jacobian matrix of~$S$ at $x^*$.
The largest real part of the eigenvalues of $H$ coincides with $-L$, where
\begin{equation}
L\eqdef b\wedge \frac a2\left( 1-\sqrt{\left(1-\frac{4\lambda_1}a\right)\vee 0}\right) >0\,.\label{eq:L}
\end{equation}
Finally, define the $3d\times 3d$ matrix% covariance matrix of the
% vector $(a\nabla f(x^*,\xi),b\nabla f(x^*,\xi)^{\odot 2})$ \emph{i.e.},
\begin{equation}
  \resizebox{.9\hsize}{!}{$
Q \eqdef
\left( \begin{array}[h]{cc}
   0 &
   \begin{array}[h]{cc}
     0 & \ \ \ \ \ \ \ \ \ \ \ \ \ \ \ \ \ \ \ \ \ \ \ \ \ \ \ \ \ \ 0
   \end{array}\\
   \begin{array}[h]{c}
     0 \\ 0
   \end{array}
 & \bE\left[
  \begin{pmatrix}
    a\nabla f(x^*,\xi) \\ b(\nabla f(x^*,\xi)^{\odot 2}-S(x^*))
  \end{pmatrix}\begin{pmatrix}
    a\nabla f(x^*,\xi) \\ b(\nabla f(x^*,\xi)^{\odot 2}-S(x^*))
  \end{pmatrix}^T\right]
 \end{array}\right)\,.\label{eq:Q}
 $}
\end{equation}
\begin{assumption}
\label{hyp:step-tcl}
The following holds.
\begin{enumerate}[{\sl i)}]
\item \label{step-tcl-i} There exist $\kappa \in (0,1]$, $\gamma_0>0$, s.t. the sequence $(\gamma_n)$
satisfies
$\gamma_n = {\gamma_0}/{(n+1)^\kappa}$ for all $n$.
If $\kappa = 1$, we assume moreover that $\gamma_0 > \frac{1}{2L}$.
\item The sequences $\left(\frac{1}{\gamma_n}(\frac{1-\alpha_n}{\gamma_n}-a)\right)$ and $\left(\frac{1}{\gamma_n}(\frac{1-\beta_n}{\gamma_n}-b)\right)$ are bounded.
%This is satisfied for instance if $1-\alpha_n = a\gamma_n$ and $1-\beta_n = b\gamma_n$.
\end{enumerate}
\end{assumption}

\noindent For an arbitrary sequence $(X_n)$ of random variables on some Euclidean space, a probability measure
$\mu$ on that space and an event $\Gamma$ s.t. $\bP(\Gamma)>0$, we say that $X_n$ converges in distribution
to $\mu$ \emph{given $\Gamma$} if the measures $\bP (X_n\in \cdot\,|\Gamma)$ converge weakly to $\mu$.
\begin{theorem}
 \label{thm:clt}
Let Assumptions~\ref{hyp:model}, \ref{hyp:S>0}, \ref{hyp:iid}, \ref{hyp:mean_field_tcl} and \ref{hyp:step-tcl} hold true.
Let Assumption~\ref{hyp:moment-f}~\ref{momentreinforce} hold with $p=4$.
Consider the iterates $z_n=(x_n,m_n,v_n)$ given by Algorithm~\ref{alg:adam-decreasing}. Set $z^*=(x^*,0,S(x^*))$.
Set $\zeta \eqdef 0$ if $0<\kappa<1$ and $\zeta \eqdef \frac{1}{2 \gamma_0}$ if $\kappa =1$.
%Assume $\bP(x_n \to x^*)>0$. Then, given the event $\{x_n\to x^*\}$,
Assume $\bP(z_n \to z^*)>0$. Then, given the event $\{z_n\to z^*\}$,
the rescaled vector $\sqrt{\gamma_n}^{-1}(z_n-z^*)$
converges in distribution to a zero mean Gaussian distribution on $\bR^{3d}$ with a covariance matrix $\Sigma$
which is solution to the Lyapunov equation: $  \left(H + \zeta I_{3d} \right) \Sigma + \Sigma \left( H^T +  \zeta I_{3d} \right) = - Q$.
In particular, given $\{z_n\to z^*\}$, the vector $\sqrt{\gamma_n}^{-1}(x_n-x^*)$
converges in distribution to a zero mean Gaussian distribution with a covariance matrix $\Sigma_1$ given by:
\begin{equation}
\Sigma_1 = D^{1/2} P
\left(
\frac{C_{k,\ell}}{ (1 - \frac{2\zeta}{a})(\lambda_k+\lambda_\ell-2\zeta + \frac 2a \zeta^2) +\frac 1{2(a-2\zeta)}(\lambda_k-\lambda_\ell)^2}
 %\frac{C_{k,\ell}}{ (1-\frac{\zeta}{a})(\lambda_k+\lambda_\ell - 2\zeta) +\frac 1{2(a-\zeta)}(\lambda_k-\lambda_\ell)^2}
\right)_{k,\ell=1\dots d}
P^{-1}D^{1/2}\label{eq:cov}
\end{equation}
%\begin{equation}
%\Sigma_1 = D^{1/2} P
%\left(
%  \frac{C_{k,\ell}}{\lambda_k+\lambda_\ell+\frac 1{2a}(\lambda_k-\lambda_\ell)^2}
%\right)_{k,\ell=1\dots d}
%P^{-1}D^{1/2}\label{eq:cov}
%\end{equation}
where $C\eqdef P^{-1}D^{1/2}\bE\left(\nabla f(x^*,\xi)\nabla f(x^*,\xi)^T\right)D^{1/2}P$.
\end{theorem}

The following remarks are useful.
\begin{itemize}[leftmargin=*]
\item The variable $v_n$ has an impact on the limiting covariance $\Sigma_1$ through its limit $S(x^*)$ (used to define $D$),
but the fluctuations of $v_n$ and the parameter $b$ have no effect on $\Sigma_1$.
As a matter of fact, $\Sigma_1$ coincides with the limiting covariance matrix that would have been obtained by considering
iterates of the form
\begin{equation*}
     \begin{cases}
       x_{n+1} &= x_n - \gamma_{n+1} p_{n+1} \\
       p_{n+1} &= p_n + a\gamma_{n+1}(D\nabla f(x_n,\xi_{n+1})-p_n) \,,
     \end{cases}
\end{equation*}

which can be interpreted as a preconditioned version of the stochastic heavy ball algorithm~\cite{gadat2018stochastic}.
Of course, the above iterates are not implementable because the preconditioning matrix $D$ is unknown.
\item  %Eq.~(\ref{eq:cov}) suggests that %Cor.~\ref{cor:x}
When $a$ is large, $\Sigma_1$ is close to the matrix $\Sigma_1^{(0)}$ obtained
when letting $a \to +\infty$ in Eq.~(\ref{eq:cov}).
%when deleting the squared term in the denominator of Eq.~(\ref{eq:cov}).
The matrix $\Sigma_1^{(0)}$ is the solution to the  Lyapunov equation
$$
(D \nabla^2F(x^*) - \zeta I_d)  \Sigma_1^{(0)} + \Sigma_1^{(0)} (\nabla^2F(x^*) D - \zeta I_d) = D \bE\left(\nabla f(x^*,\xi)\nabla f(x^*,\xi)^T\right) D\,.
$$
%$$
%\nabla^2F(x^*) D \Sigma_1^{(0)} + \Sigma_1^{(0)} D \nabla^2F(x^*) = - \bE\left(\nabla f(x^*,\xi)\nabla f(x^*,\xi)^T\right)\,.
%$$
The matrix $\Sigma_1^{(0)}$ can be interpreted as the asymptotic covariance matrix of the $x$-variable
in the absence of the inertial term (that is, when one considers \textsc{RmsProp} instead of \adam).
The matrix $\Sigma_1^{(0)}$ approximates $\Sigma_1$ in the sense that $\Sigma_1 = \Sigma_1^{(0)}+\frac 1a\Delta + O(\frac 1{a^2})$
for some symmetric matrix $\Delta$ which can be explicited. The matrix $\Delta$ is neither positive nor negative definite
in general.
%which suggests that the question of the potential benefit of the presence of an inertial term
This suggests that the question of the potential benefit of the presence of an inertial term
is in general problem dependent. %{\color{red} Verifier ce que je raconte.}
\item In the statement of Th.~\ref{thm:clt},
the conditioning event $\{z_n\to z^*\}$ can be replaced by the event $\{x_n\to x^*\}$
under the additional assumption that $\sum_n \gamma_n^2 < +\infty$.
\end{itemize}

 \section{Related Works}
 \label{sec:related-works}

 Although the idea of adapting the
 (per-coordinate) learning rates as a function of past gradient values
 is not new (see \emph{e.g.} variable metric methods such as the BFGS
 algorithms), % \cite{fletcher1970new})
 \textsc{AdaGrad} \cite{duchi2011adaptive} led the way to a new class of algorithms
 that are sometimes referred to as adaptive gradient methods. \textsc{AdaGrad}
 consists of dividing the learning rate by the square root of the sum
 of previous gradients squared
 componentwise.  %Theoretical regret bounds are provided in \cite{duchi2011adaptive}.
 The idea was to give larger learning rates to highly informative but
 infrequent features instead of using a fixed predetermined schedule.
 However, in practice, the division by the cumulative sum of squared
 gradients may generate small learning rates, thus freezing the
 iterates too early.  Several works proposed
 heuristical ways to set the learning rates using a less aggressive
 policy. %, see \emph{e.g.} \cite{schaul2013no}.
 The work \cite{tieleman2012lecture} introduced an unpublished, yet popular, algorithm
 referred to as \textsc{RmsProp} where the cumulative sum
 used in \textsc{AdaGrad} is replaced by a moving average of squared
 gradients.
 %Variants \textsc{SC-AdaGrad} and \textsc{SC-RmsProp} were proposed
 %for strongly convex objectives with logarithmic regret bounds
 %\cite{pmlr-v70-mukkamala17a}.
 \adam\ combines the advantages of both \textsc{AdaGrad},
 \textsc{RmsProp} and inertial methods.% \cite{polyak1964some}.

 As opposed to \textsc{AdaGrad}, for which theoretical convergence guarantees exist
 \cite{duchi2011adaptive,chen2018convergence,zhou2018convergence,ward2018adagrad},
 \adam\ is comparatively less studied.
 The initial paper \cite{kingma2014adam} suggests a $\mathcal{O}(\frac{1}{\sqrt{T}})$ average regret bound in the convex setting,
 but \cite{j.2018on} exhibits a counterexample in contradiction with this statement.
 The latter counterexample implies that the average regret bound of \adam\ does
 not converge to zero. A first way to overcome the problem is to modify the \adam\ iterations
 themselves in order to obtain a vanishing average regret. This led \cite{j.2018on}
 to propose a variant called \textsc{AmsGrad} with the aim to recover, at least in the convex case, the sought guarantees.
 % \cite{j.2018on} also explain that \textsc{AmsGrad} outperforms \adam\ in
 % their experiments, although more empirical exploration still seems to be needed to endorse the claim \cite{korzeniowski-17}.
 % More precisely, they showed that \textsc{RMSprop} and \adam\ fail to converge to the optimal solution for a simple convex one-dimensional
 % optimization setting involving a time periodic function sequence. They proposed a new algorithm called \textsc{AMSGrad} relying on a "long-term memory"
 % of past gradients to fix convergence issues due to the exponential moving average.
 %However, the convergence analysis of \textsc{AMSGrad} was limited to the convex setting whereas \adam\ is mostly used in highly nonconvex applications
 %in deep learning.
 % On the otherhand, a few recent works aim at providing answers to the question
 % of the theoretical guarantees of \adam.
 The work \cite{balles2018dissecting} interprets \adam\ as a variance-adapted sign descent combining an update direction given by the sign and
 a magnitude controlled by a variance adaptation principle. A ``noiseless'' version %(the function $f$ is non-random)
 of \adam\ is considered in \cite{basu2018convergence}. Under quite specific values of the \adam-hyperparameters, it is shown that for every $\delta>0$,
 there exists some time instant %(non-explicit, but with an explicit upper bound)
 for which the norm of the gradient of the objective
 at the current iterate is no larger than~$\delta$.
 %It is also worth noting that the choice of the algorithm parameters in the theorem is far from the practical setting of \adam\.
 %(achieving a small norm gradient needs a first exponential decay rate that is close to zero whereas typical values for the latter are close to 1).
 The recent paper \cite{chen2018convergence} provides a similar result
 for \textsc{AmsGrad} and \textsc{AdaGrad}, but the generalization to \adam\ is subject
 to conditions which are not easily verifiable.
 The paper \cite{zaheer2018adaptive} provides a convergence result for \textsc{RmsProp}
 using the objective function $F$ as a Lyapunov function.
 %To that end, the objective $F$ is used as a Lyapunov function,
 However, our work suggests that unlike \textsc{RmsProp},
 \adam\ does not admit $F$ as a Lyapunov function.
 This makes the approach of
 % ,which makes the approach of
 \cite{zaheer2018adaptive} hardly generalizable to \adam.
 Moreover, \cite{zaheer2018adaptive} considers biased gradient estimates instead of the debiased
 estimates used in \adam.

 In the present work, we study the behavior of an ODE, interpreted as the %weak
 limit in probability of the (interpolated) \adam\ iterates as the stepsize tends to zero.
 Closely related continuous-time dynamical systems are also studied in \cite{attouch2000heavy,cabot2009long}.
 We leverage the idea of approximating a discrete time stochastic system by a deterministic continuous one,
 often referred to as the ODE method. %, traces back to the works of \cite{ljung1977analysis}. %(see also \cite{kus-yin-(livre)03}).
 A recent work \cite{gadat2018stochastic} fruitfully exploits this method to study
 a stochastic version of the celebrated heavy ball algorithm. %called stochastic heavy ball.
 We refer to \cite{davis2018stochastic} for the reader interested in the non-differentiable setting
 with an analysis of the stochastic subgradient algorithm for non-smooth non-convex objective functions.
 Concomitant to the present paper, Da Silva and Gazeau \cite{da2018general}
 (posted only four weeks after the first version of the present work) % \cite{barakat-bianchi2018convergence})
 %and
 study the asymptotic behavior of a similar dynamical system as the one introduced here.
 They establish several results in continuous time, such as avoidance of traps
 as well as convergence rates in the convex case; such aspects are out of the scope of this paper.
 However, the question of the convergence of the (discrete-time) iterates is left open.
 In the current paper, we also exhibit a Lyapunov function which allows, amongst others, to draw useful conclusions on the effect of the
 debiasing step of \adam\,. Finally, \cite{da2018general} studies a slightly modified version of \adam\, allowing to recover an
 ODE with a locally Lipschitz continuous vector field, whereas the original \adam\ algorithm \cite{kingma2014adam} leads %, on the other hand,
 to an ODE with an irregular vector field. This technical issue is tackled in the present paper.

%%%%%%%%%%%%%%%%%%%%%%%%%%%%%%%%%%%%%%%%%
\section{Proofs of Section~\ref{sec:continuous_time}}
\label{sec:proofs_cont_time}

\subsection{Preliminaries}
\label{subsec:setting}
The results in this section are not specific to the case where $F$ and $S$ are defined as in
% Eqs.~(\ref{eq:F})--(\ref{eq:S}): they are stated for
Eq.~(\ref{eq:F_and_S}): they are stated for
\emph{any} mappings $F$, $S$ satisfying the following hypotheses.
\begin{assumption}
  \label{hyp:F}
The function $F:\bR^d\to\bR$ is s.t.: $F$ is continuously differentiable and
$\nabla F$ is locally Lipschitz continuous.
\end{assumption}
\begin{assumption}
\label{hyp:S}
The map $S:\bR^d\to [0,+\infty)^d$ is locally Lipschitz continuous.
\end{assumption}
%Note that Assumption~\ref{hyp:S}-\ref{hyp:S} is stronger than Assumption~\ref{hyp:S}-\ref{hyp:S}.
%We shall use the former to prove the existence of a global solution to~(\ref{eq:ode}), whereas the latter will be needed to show uniqueness.
%  In the special case
% When $F$ and $S$ are given by Eq.~(\ref{eq:F})--(\ref{eq:S}),% it is quite easy to prove that
% Assumptions~\ref{hyp:F} and~\ref{hyp:S} follow from Assumption~\ref{hyp:model}.
In the sequel, we consider the following generalization of Eq. (\ref{eq:ode}) for any $\eta >0$:
\begin{equation}
  \begin{array}[h]{l}
\dot z(t) = h(t+\eta, z(t))\,.
\end{array}
\tag{ODE\mbox{$_\eta$}}
\label{eq:odeeta}\end{equation}
%where $\eta\geq 0$ is a given parameter.
When $\eta=0$, Eq. (\ref{eq:odeeta}) boils down to the equation of interest (\ref{eq:ode}).
The choice $\eta\in (0,+\infty)$ will be revealed useful to prove Th.~\ref{th:exist-unique}.
Indeed, for $\eta>0$, a solution to Eq. (\ref{eq:odeeta}) can be shown to exist (on some interval) due to the continuity of
the map $h(\,.+\eta,\,.\,)$. Considering a family of such solutions indexed by $\eta\in (0,1]$,
the idea is to prove the existence of a solution to (\ref{eq:ode}) as a cluster point of the latter family when $\eta\downarrow 0$.
Indeed, as the family is shown to be equicontinuous, such a cluster point does exist thanks to the Arzelà-Ascoli theorem.
When $\eta=+\infty$, %we obtain Eq.~(\ref{eq:ode-a}).
%% will be revealed useful to prove Th.~\ref{th:cv-adam}.
%% Setting $h_\infty(z)\eqdef \lim_{t\to \infty} h(t,z)$ for every $z\in \cZ$,
Eq. (\ref{eq:odeeta}) rewrites
\begin{equation}
  \label{eq:ode-a}
  \begin{array}[h]{l}
\dot z(t) = h_\infty(z(t))\,,
\end{array}
\tag{ODE\mbox{$_\infty$}}
\end{equation}
where $h_\infty(z)\eqdef \lim_{t\to \infty} h(t,z)$.
It is useful to note that for $(x,m,v)\in \cZ_+$,
\begin{equation}
  \label{eq:h_infty}
h_{\infty}((x,m,v)) = \left(-m / (\varepsilon+\sqrt{v})\,,\, a (\nabla F(x)-m) \,,\,b (S(x)-v) \right)\,.
\end{equation}
%$$h_{\infty}(x,m,v) = \left(-m / (\varepsilon+\sqrt{v})\,,\, a (\nabla F(x)-m) \,,\,b (S(x)-v) \right)\,.$$
%% \begin{equation*}
%%   h_{\infty}(z)  = \begin{pmatrix}
%%   -m / (\varepsilon+\sqrt{|v|}) \\
%%   a (\nabla F(x)-m) \\
%%   b (S(x)-v)
%% \end{pmatrix}\,.
%% \end{equation*}
Contrary to Eq. (\ref{eq:ode}), Eq.~(\ref{eq:ode-a}) defines an  autonomous ODE.
The latter admits a unique global solution for any initial condition in $\cZ_+$,
and defines a dynamical system $\cD$. We shall exhibit a strict Lyapunov function
for this dynamical system $\cD$, and deduce that any solution to (\ref{eq:ode-a}) converges
to the set of equilibria of $\cD$ as $t\to\infty$.
On the otherhand, we will prove that the solution to (\ref{eq:ode}) with a proper initial condition is a so-called asymptotic pseudotrajectory (APT) of $\cD$. Due to the
existence of a strict Lyapunov function, the APT shall inherit the convergence behavior of the autonomous system as $t\to\infty$,
which will prove Th.~\ref{th:cv-adam}.

It is convenient to extend the map $h:(0,+\infty)\times \cZ_+\to\cZ$ on  $(0,+\infty)\times \cZ\to\cZ$ by
setting $h(t,(x,m,v))\eqdef h(t,(x,m,|v|))$ for every $t>0$, $(x,m,v)\in \cZ$.
Similarly, we extend $h_\infty$ as $h_\infty((x,m,v)) \eqdef h_\infty((x,m,|v|))$.
% In this paragraph, we will only consider initial points in $\cZ$ of the form
% $(x_0,0,0)$, for some $x_0\in \bR^d$.
For any $T\in (0,+\infty]$ and any $\eta\in [0,+\infty]$, we say that a map $z:[0,T)\to \cZ$ is a solution
to (\ref{eq:odeeta}) on $[0,T)$ with initial condition $z_0\in \cZ_+$,
if $z$ is continuous on $[0,T)$, continuously differentiable
on $(0,T)$, and if (\ref{eq:odeeta}) holds for all $t\in (0,T)$.
When $T=+\infty$, we say that the solution is global.
We denote by $Z^\eta_T(z_0)$ the subset of $C([0,T),\cZ)$ formed by the solutions to (\ref{eq:odeeta})
on $[0,T)$ with initial condition $z_0$.
For any $K\subset\cZ_+$, we define $Z^\eta_T(K)\eqdef \bigcup_{z\in K}Z^\eta_T(z)$.

% When $\eta>0$, it is clear from the continuity of the map $h(\,.+\eta,\,.\,)$ that
% any $z\in Z^\eta_T(z_0)$ is continuously differentiable on $[0,T)$, that is, at point zero as well.
% The following lemma provides an extension to the case $\eta=0$.

% We first start by providing some useful properties
% of the solutions to (\ref{eq:odeeta}).
% The issue of the existence of such solution is postponed to
% Section~\ref{sec:existence-adam}.
\begin{lemma}
  \label{lem:m-v-derivables-en-zero}
  Let Assumptions~\ref{hyp:F} and \ref{hyp:S} hold.  Consider $x_0\in \bR^d$,
  $T\in (0,+\infty]$ and let $z\in Z_T^0((x_0,0,0))$,
  which we write $z(t) = (x(t),m(t),v(t))$.  Then, $z$
  is continuously differentiable on $[0,T)$, %and it holds that
  $\dot m(0)=a\nabla F(x_0)$, $\dot v(0)=bS(x_0)$ and
$
\dot x(0) =  \frac{-\nabla F(x_0)}{\varepsilon + \sqrt{S(x_0)}}.
$
\end{lemma}
\begin{proof}
% When $\eta>0$, the result is immediate by using the continuity of $(t,z)\mapsto h(t+\eta,z)$ on $[0,+\infty)\times\cZ$
% along with the fundamental theorem of calculus. We consider the case  $\eta=0$.
By definition of $z(\,.\,)$, $m(t)=\int_0^ta(\nabla F(x(s))-m(s))ds$ for all $t\in [0,T)$
(and a similar relation holds for $v(t)$).
The integrand being continuous, it holds %follows from the fundamental theorem of calculus
that $m$ and $v$ are differentiable at  zero and $\dot m(0)=a\nabla F(x_0)$, $\dot v(0)=bS(x_0)$.
Similarly, $x(t) = x_0+\int_0^t h_x(s,z(s))ds$, where %for all $s\in(0,T)$,
$h_x(s,z(s)) \eqdef -(1-e^{-as})^{-1}m(s)/(\varepsilon+\sqrt{(1-e^{-bs})^{-1}v(s)})\,.$
% \begin{align*}
%   h_x(s,z(s)) &\eqdef \frac{-(1-e^{-as})^{-1}m(s)}{\varepsilon+\sqrt{(1-e^{-bs})^{-1}v(s)}}\,.
% \end{align*}
Note that $m(s)/s \to \dot m(0) = a\nabla F(x_0) $ as $s\downarrow 0$.
%On the otherhand, $(1-e^{-as})/s \to a$. Therefore,
Thus, $(1-e^{-as})^{-1}m(s)\to \nabla F(x_0)$ as $s\to 0$. Similarly,
$(1-e^{-bs})^{-1}v(s)\to S(x_0)$. It follows that
$h_x(s,z(s))\to -(\varepsilon+\sqrt{S(x_0)})^{-1}\nabla F(x_0)$.
Thus, $s\mapsto h_x(s,z(s))$ %defined on $(0,T)\to\bR^d$
can be extended
to a continuous map on $[0,T)\to\bR^d$ and the differentiability of $x$ at zero
follows. % again by the fundamental theorem of calculus.
\end{proof}

\begin{lemma}
\label{lem:v-positif}
%Let \cref{hyp:S>0,hyp:F,hyp:S} hold.
Let Assumptions~\ref{hyp:S>0}, \ref{hyp:F} and \ref{hyp:S} hold.
For every $\eta\in [0,+\infty]$, $T\in (0,+\infty]$,  $z_0\in \cZ_+$,  $z\in Z_T^\eta(z_0)$,
it holds that $z((0,T))\subset \cZ_+^*$.
% Let Assumptions~\ref{hyp:F}-\ref{hyp:Fderivable} and \ref{hyp:S}-\ref{hyp:S} hold.
% Consider $z_0\in \cZ_+$, $T\in (0,+\infty]$ and $z\in \Phi_{\text{a},T}(z_0)$. Then,
% $z([0,T))\subset \cZ_+$. Moreover, under Assumption~\ref{hyp:S>0},
% $z((0,T))\subset \cZ_+^*$.
\end{lemma}
\begin{proof}
Set $z(t) = (x(t),m(t),v(t))$ for all $t$. Consider $i\in \{1,\dots,d\}$.
Assume by contradiction that there exists $t_0\in (0,T)$ s.t.
$v_i(t_0)<0$. Set $\tau\eqdef\sup\{t\in [0,t_0]:v_i(t)\geq 0\}$.
Clearly, $\tau<t_0$ and $v_i(\tau)=0$ by the continuity of $v_i$.
Since $v_i(t)\leq 0$ for all $t\in (\tau,t_0]$, it holds that $\dot v_i(t) = b(S_i(x(t))-v_i(t))$
is nonnegative for all  $t\in (\tau,t_0]$. This contradicts the fact that $v_i(\tau)>v_i(t_0)$.
Thus, $v_i(t)\geq 0$ for all $t\in [0,T)$.
Now assume by contradiction that there exists $t\in (0,T)$  s.t.
$v_i(t)=0$. Then, $\dot v_i(t)=bS_i(x(t))>0$.
Thus,
$
\lim_{\delta\downarrow 0} \frac{v_i(t-\delta)}{-\delta} = bS_i(x(t))\,.
$
In particular, there exists $\delta>0$ s.t.
$v_i(t-\delta)  \leq  -\frac{\delta b}2S_i(x(t))\,.$ This contradicts the first point.
\end{proof}

 Recall the definitions of $V$ and $U$ from Eqs.~(\ref{eq:V}) and (\ref{eq:U}).
%Clearly, $\lim_{t\to \infty} U(t,v)$ and $\lim_{t\to \infty} V(t,z)$ exist for every $z\in \cZ_+$
 Clearly, $U_\infty(v)\eqdef\lim_{t\to \infty} U(t,v)=a(\varepsilon+\sqrt{v})$ is well defined for every $v\in [0,+\infty)^d$.
 Hence, we can also define $V_\infty(z)\eqdef \lim_{t\to \infty} V(t,z)$ for every $z\in \cZ_+$.
 %Hence, define $V_\infty(z)\eqdef \lim_{t\to \infty} V(t,z)$ for every $z\in \cZ_+$ and
 %$U_\infty(v)\eqdef \lim_{t\to \infty} U(t,v)=a(\varepsilon+\sqrt{v})$ for every $v\in [0,+\infty)^d$.

\begin{lemma}
\label{lem:V}
  %Let \cref{hyp:F,hyp:S} hold.
  Let Assumptions~\ref{hyp:F} and \ref{hyp:S} hold.
   Assume that $0< b\leq 4a$.
Consider $(t,z)\in (0,+\infty)\times \cZ_+^*$ and set $z=(x,m,v)$.
Then, $V$ and $V_\infty$ are  differentiable at points  $(t,z)$ and $z$ respectively. Moreover,
$\ps{\nabla V_\infty(z),h_{\infty}(z)}\leq -\varepsilon \left\|\frac {am}{U_\infty(v)}\right\|^2\,$ and
\begin{equation*}
\ps{\nabla V(t,z), (1,h(t,z))} \leq -\frac{\varepsilon }2\left\|\frac{a\,m}{U(t,v)}\right\|^2\,.
\end{equation*}
\end{lemma}
\begin{proof}
We only prove the second point, the proof of the first point follows the same line.
%and can be found in \cite[Lemma 5.3]{barakat-bianchi2018convergence}.
 Consider $(t,z)\in (0,+\infty)\times \cZ_+^*$.
We decompose %$\nabla V(t,z)=(\partial_tV(t,z),\nabla_z V(t,z))$ where $\partial_t$ stands for the derivative w.r.t. $t$ and $\nabla_z$,
$\ps{\nabla V(t,z), (1,h(t,z))} = \partial_tV(t,z)
+\ps{\nabla_z V(t,z),h(t,z)}$. After tedious but straightforward derivations, we get: %obtain:
\begin{equation}
  \label{eq:partial-t}
  \resizebox{0.99\hsize}{!}{$
  \partial_t V(t,z) =- \sum_{i=1}^d \frac{a^2m_i^2}{U(t,v_i)^2}\left(\frac{e^{-at}\varepsilon}2+
\left(\frac{e^{-at}}2-\frac{be^{-bt}(1-e^{-at})}{4a(1-e^{-bt})}\right)\sqrt{\frac{v_i}{1-e^{-bt}}}\right)\,,
$}
\end{equation}
where $U(t,v_i)=a(1-e^{-at})\left(\varepsilon+\sqrt{\frac{v_i}{1-e^{-bt}}}\right)$ and $\ps{\nabla_z V(t,z),h(t,z)}$ is equal to:
\begin{equation*}
    \sum_{i=1}^d \frac{-a^2m_i^2(1-e^{-at})}{U(t,v_i)^2}
  \left(\varepsilon
+ (1-\frac b{4a})\sqrt{\frac{v_i}{1-e^{-bt}}}
+\frac{bS_i(x)}{4a\sqrt{v_i(1-e^{-bt})}}
\right)\,.
\end{equation*}
Using that $S_i(x)\geq 0$, we obtain:
\begin{equation}
\ps{\nabla V(t,z), (1,h(t,z))} \leq  -\sum_{i=1}^d \frac{a^2m_i^2}{U(t,v_i)^2}\left(
(1-\frac{e^{-at}}2)\varepsilon+c_{a,b}(t)\sqrt{\frac{v_i}{1-e^{-bt}}}
\right)\,,\label{eq:ineg-V}
\end{equation}
where $c_{a,b}(t)\eqdef 1-\frac{e^{-at}}2-\frac b{4a}\frac{1-e^{-at}}{1-e^{-bt}}\,.$
Using inequality $1-{e^{-at}}/2\geq 1/2$ in (\ref{eq:ineg-V}), the inequality~(\ref{eq:ineg-V})
proves the Lemma, provided that one is able to show that $c_{a,b}(t)\geq 0$, for all $t>0$
and all $a,b$ satisfying $0< b\leq 4a$. We prove this last statement.
It can be shown that the function $b\mapsto c_{a,b}(t)$ is decreasing on $[0,+\infty)$.
Hence, $c_{a,b}(t)\geq c_{a,4a}(t)$. Now, $c_{a,4a}(t) = q(e^{-at})$ where $q:[0,1)\to\bR$ is the
function defined for all $y\in [0,1)$ by $q(y) = y \left(y^4-2y^3+1\right)/(2(1-y^4))\,$.
Hence $q \geq 0$. Thus, $c_{a,b}(t)\geq q(e^{-at})\geq 0$.
% This concludes the proof of the second inequality.
\end{proof}

\subsection{Proof of Th.~\ref{th:exist-unique}}

\subsubsection{Boundedness}
%We study the boundedness of any solution to (\ref{eq:odeeta}) (provided it exists). The stated results in the cases $\eta<+\infty$
%and $\eta=+\infty$ are slightly different. They are stated separately in
%Prop.~\ref{prop:adam-bounded} and Prop.~\ref{prop:bounded} respectively.
Define $\cZ_0 \eqdef \{(x,0,0):x\in \bR^d\}$.
% Let $\bar e:(0,+\infty)\times \cZ_+\to\cZ_+$ be defined for every $t>0$ and every $z=(x,m,v)$ in $\cZ_+$ by:
% \begin{equation}
% \bar e(t,z)\eqdef (x,m/(1-e^{-at}),v/(1-e^{-bt}))\,.\label{eq:ebar}
% \end{equation}
Let $\bar e:(0,+\infty)\times \cZ_+\to\cZ_+$ be defined by
$\bar e(t,z)\eqdef (x,m/(1-e^{-at}),v/(1-e^{-bt}))\,$
for every $t>0$ and every $z=(x,m,v)$ in $\cZ_+$.

\begin{proposition}
\label{prop:adam-bounded}
%Let \cref{hyp:coercive,hyp:S>0,hyp:F,hyp:S} hold.
Let Assumptions~\ref{hyp:coercive}, \ref{hyp:S>0}, \ref{hyp:F} and \ref{hyp:S} hold.
Assume that $0< b\leq 4a$.
For every $z_0\in \cZ_0$, there exists a compact set $K\subset \cZ_+$ s.t.
for all $\eta\in [0,+\infty)$, all $T\in (0,+\infty]$ and all $z\in Z_T^\eta(z_0)$,
$\left\{\bar e(t+\eta,z(t)) :t\in (0,T)\right\} \subset K\,.$
Moreover, choosing $z_0$ of the form $z_0=(x_0,0,0)$ and $z(t) = (x(t), m(t),v(t))$, it holds that $F(x(t))\leq F(x_0)$ for all $t\in [0,T)$.
\end{proposition}

\begin{proof}
%Consider $\eta\in [0,+\infty)$.
Let $\eta\in [0,+\infty)$.
Consider a solution $z_\eta(t) = (x_\eta(t),m_\eta(t),v_\eta(t))$  as in the statement, defined on some interval $[0,T)$.
Define
$\hat m_\eta(t) = m_\eta(t)/(1-e^{-a(t+\eta)})$,
$\hat v_\eta(t) = v_\eta(t)/(1-e^{-b(t+\eta)})$.
By Lemma~\ref{lem:v-positif}, $t\mapsto V(t+\eta,z(t))$ is continuous on $[0,T)$, and
continuously differentiable on $(0,T)$.
%By Lemma~\ref{lem:V}, $\dot V(t+\eta,z_\eta(t)) = \ps{\nabla V(t+\eta,z_\eta(t)),(1,h(t+\eta,z_\eta(t)))}\leq 0$ for all $t>0$.
By Lemma~\ref{lem:V}, $\dot V(t+\eta,z_\eta(t)) \leq 0$ for all $t>0$.
As a consequence, $t\mapsto V(t+\eta,z_\eta(t))$ is non-increasing on $[0,T)$.
Thus, for all $t\geq 0$, $F(x_{\eta}(t))\leq \lim_{t'\downarrow 0}V(t'+\eta,z_\eta(t'))$. Note that
$ %\begin{equation}
  V(t+\eta,z_\eta(t)) \leq F(x_\eta(t))+\frac 12 \sum_{i=1}^d
  \frac{m_{\eta,i}(t)^2}{a(1-e^{-a(t+\eta)})\varepsilon}\,. \label{eq:majV}
$ %\end{equation}
% $$
% \leq F(x(t))+\frac 12  \sum_{k=1}^d \frac{m_{\eta,k}(t)^2}{a(1-e^{-at})\sqrt{v_{\eta,k}(t)}}\,.
% $$
If $\eta>0$, every term in the sum in the righthand side %of (\ref{eq:majV})
tends to zero, upon noting that
%$m_{\eta,i}(t)\to 0$ as $t\to 0$, for every $i\in \{1,\dots,d\}$.
$m_{\eta}(t)\to 0$ as $t\to 0$.
The statement still holds if $\eta=0$. Indeed, by Lemma~\ref{lem:m-v-derivables-en-zero},
for a given $i\in \{1,\dots,d\}$, there exists $\delta>0$ s.t. for all $0<t<\delta$,
$m_{\eta,i}(t)^2\leq 2 a^2 (\partial_i F(x_0))^2 t^2$ %, $v_{\eta,k}(t)\geq \frac b2 S_k(x_0)t$
and $1-e^{-at}\geq (at)/2$. As a consequence, each term of the sum %in the righthand side of~(\ref{eq:majV})
is no larger than $4 (\partial_i F(x_0))^2 t/\varepsilon$, which tends to zero as $t\to 0$.
% Using Assumption~\ref{hyp:S>0},
% \begin{align*}
%   \lim_{t\downarrow 0}V(t+\eta,z_\eta(t)) &\leq F(x_0)+\frac 12
%   \lim_{t\downarrow 0} \sum_{k=1}^d \frac{4 a^2 (\partial_k F(x_0))^2
%     t^2}{a^2t\sqrt{\frac b2 S_k(x_0)t}}\ = F(x_0)\,.
% \end{align*}
We conclude that for all $t\geq 0$, $F(x_\eta(t))\leq F(x_0)$.
In particular, $\{x_\eta(t):t\in [0,T)\}\subset \{F\leq F(x_0)\}$, the latter set being bounded
by Assumption~\ref{hyp:coercive}.

We prove that $v_{i,\eta}(t)$ is (upper)bounded.
Define $R_{i} \eqdef \sup  S_i(\{F\leq F(x_0)\})$, which is finite by continuity of $S$.
Assume by contradiction that the set $\{t\in [0,T):v_{\eta,i}(t)\geq R_{i}+1\}$ is non-empty, and denote
its infimum by $\tau$. By continuity of $v_{\eta,i}$, one has
$v_{\eta,i}(\tau) =  R_{i}+1$. This by the way implies that $\tau>0$. Hence,
$\dot v_{\eta,i}(\tau) = b(S_i(x_\eta(\tau))-v_{\eta,i}(\tau)) \leq -b$.
This means that there exists $\tau'<\tau$ s.t. $v_{\eta,i}(\tau')>v_{\eta,i}(\tau)$, which contradicts the definition of $\tau$.
We have shown that $v_{\eta,i}(t)\leq R_i+1$ for all $t\in (0,T)$.
In particular, when $t\geq 1$, $\hat v_{\eta,i}(t) = v_{\eta,i}(t)/(1-e^{-bt}) \leq (R_i+1)/(1-e^{-b})\,.$
Consider $t\in (0,1\wedge T)$.
By the mean value theorem, there exists $\tilde t_\eta\in [0,t]$ s.t. $v_{\eta,i}(t) = \dot v_{\eta,i}(\tilde t_\eta)t$.
Thus, $v_{\eta,i}(t) \leq b S_i(x(\tilde t_\eta)) t\leq b R_i t$. Using that the map $y\mapsto y/(1-e^{-y})$ is increasing on $(0,+\infty)$,
it holds that for all $t\in (0,1\wedge T)$,
$\hat v_{\eta,i}(t) %= \frac{v_{\eta,k}(t)}{1-e^{-bt}}
\leq bR_i /(1-e^{-b})\,.$
We have shown that, for all $t\in (0,T)$ and all $i\in \{1,\dots,d\}$, $0\leq \hat v_{\eta,i}(t)\leq M$, where
$M\eqdef (1-e^{-b})^{-1}(1+ b)(1+\max\{R_\ell:\ell\in \{1,\dots,d\})$.

As $V(t+\eta,z_\eta(t))\leq F(x_0)$, we obtain: $F(x_0) \geq F(x_\eta(t))+\frac 12
\left\|m_\eta(t)\right\|^2_{U(t+\eta,v_\eta(t))^{-1}}$.
Thus, $F(x_0) \geq \inf F+\frac 1{2a(\varepsilon+\sqrt{M})} \|m_{\eta}(t)\|^2\,$.
Therefore, $m_\eta(\,.\,)$ is bounded on $[0,T)$, uniformly in $\eta$.
The same holds for $\hat m_\eta$ by using the mean value theorem
in the same way as for $\hat v_\eta$. The proof is complete.
\end{proof}

\begin{proposition}
\label{prop:bounded}
%Let \cref{hyp:coercive,hyp:S>0,hyp:F,hyp:S} hold.
Let Assumptions~\ref{hyp:coercive}, \ref{hyp:S>0}, \ref{hyp:F} and \ref{hyp:S} hold.
Assume that $0< b\leq 4a$.
Let $K$ be a compact subset of $\cZ_+$.
Then, there exists an other compact set $K'\subset \cZ_+$ s.t.
for every $T\in (0,+\infty]$ and every $z\in Z_{T}^\infty(K)$,
$z([0,T))\subset K'$.
\end{proposition}
\begin{proof}
The proof follows the same line as Prop.~\ref{prop:adam-bounded} and is omitted.
\end{proof}
For any $K\subset \cZ_+$, define $v_{\min}(K)\eqdef\inf\{v_k: (x,m,v)\in K,i\in \{1,\dots,d\}\}$.
\begin{lemma}
\label{lem:v-lowerbound}
%Under \cref{hyp:coercive,hyp:S>0,hyp:F,hyp:S},
Under Assumptions~\ref{hyp:coercive}, \ref{hyp:S>0}, \ref{hyp:F} and \ref{hyp:S},
the following holds true.%the following statements hold.
\begin{enumerate}[{\it i)},leftmargin=*]
\item For every compact set $K\subset \cZ_+$, there exists $c>0$, s.t. for every $z\in Z^\infty_{\infty}(K)$, of the form
$z(t)= (x(t),m(t),v(t))$, $v_i(t)\geq c \min\left(1 ,\frac{v_{\min}(K)}{2c}+ t\right)\qquad(\forall t\geq 0, \forall i\in\{1,\dots,d\})\,.$
\item For every $z_0\in \cZ_0$, there exists $c>0$ s.t. for every $\eta\in [0,+\infty)$ and every $z\in Z_\infty^\eta(z_0)$,
$v_i(t)\geq c\min(1,t)\qquad(\forall t\geq 0, \forall i\in\{1,\dots,d\})\,.$
\end{enumerate}
\end{lemma}
\begin{proof}
We prove the first point. Consider a compact set $K\subset \cZ_+$.
By Prop.~\ref{prop:bounded}, one can find a compact set $K'\subset \cZ_+$ s.t.
for every $z\in Z^\infty_{\infty}(K)$, it holds that $\{z(t):t\geq 0\}\subset K'$.
Denote by $L_S$ the Lipschitz constant of $S$ on the compact set $\{x:(x,m,v)\in K'\}$.
Introduce the constants $M_1\eqdef \sup\{\|m/(\varepsilon + \sqrt v)\|_\infty:(x,m,v)\in K'\}$,
$M_2\eqdef \sup\{\|S(x)\|_\infty:(x,m,v)\in K'\}$.
The constants $L_S, M_1, M_2$ are finite.
Now consider a global solution $z(t)=(x(t),m(t),v(t))$ in $Z^\infty_{\infty}(K)$.
Choose $i\in \{1,\dots,d\}$ and consider $t\geq 0$. By the mean value theorem,
there exists $t'\in [0,t]$ s.t. $v_i(t) = v_i(0)+\dot v_i(t')t$. Thus,
$ v_i(t) = v_i(0) + \dot v_i(0) t + b(S_i(x(t')) - v_i(t') - S_i(x(0))+ v_i(0)) t$,
which in turn implies
$ v_i(t)\geq v_i(0) + \dot v_i(0) t - b L_S\|x(t')-x(0)\|t - b  |v_i(t') - v_i(0)| t$.
Using again the mean value theorem, for every $\ell\in \{1,\dots,d\}$, there exists $t''\in [0,t']$ s.t.
$
|x_\ell(t')-x_\ell(0)| = t' |\dot x_\ell(t'')| \leq t M_1\,.
$
Therefore, $\|x(t')-x(0)\|\leq \sqrt d M_1 t$. Similarly, there exists $\tilde t$ s.t.:
$|v_i(t') - v_i(0)|=  t'|\dot v_i(\tilde t)|\leq t'b S_i(x(\tilde t)) \leq t bM_2\,.$
Putting together the above inequalities, $v_i(t) \geq  v_i(0) (1-bt) + bS_i(x(0)) t - bC t^2 \,$,
where $C\eqdef (M_2+L_S\sqrt d M_1)$.
For every $t\leq 1/(2b)$, $v_i(t) \geq \frac{v_{\min}}{2} + tbC\left(\frac{S_{\min}}C  - t\right) \,,$
where we defined $S_{\min}\eqdef\inf\{S_i(x):i\in \{1,\dots,d\}, (x,m,v)\in K\}$.
Setting $\tau \eqdef 0.5\min(1/b,S_{\min}/C)$,
\begin{equation}
\forall t\in [0,\tau],\  v_i(t) \geq \frac{v_{\min}}{2} + \frac{bS_{\min}t}{2}\,.\label{eq:vlin}
\end{equation}
Set $\kappa_1\eqdef 0.5(v_{\min} + bS_{\min}\tau)$. Note that $v_i(\tau)\geq \kappa_1$.
Define $S_{\min}'\eqdef\inf\{S_i(x):i\in \{1,\dots,d\}, (x,m,v)\in K'\}\,.$
Note that $S_{\min}'>0$ by Assumptions~\ref{hyp:S} and \ref{hyp:S>0}.
Finally, define $\kappa = 0.5\min(\kappa_1,S_{\min}')$.
By contradiction, assume that the set $\{t\geq \tau : v_i(t)<\kappa\}$ is non-empty, and
denote by $\tau'$ its infimum. It is clear that $\tau'>\tau$ and
$v_i(\tau')=\kappa$. Thus, $b^{-1}\dot v_i(\tau') =S_i(x(\tau'))-\kappa$.
We obtain that $b^{-1}\dot v_i(\tau') \geq 0.5S_{\min}'>0$.
As a consequence, there exists $t\in (\tau,\tau')$ s.t. $v_i(t)<v_i(\tau')$. This contradicts
the definition of $\tau'$. We have shown that for all $t\geq \tau$, $v_i(t)\geq \kappa$. Putting this together
with Eq.~(\ref{eq:vlin}) and using that $\kappa\leq v_{\min} + bS_{\min}\tau$,
we conclude that:
$\forall t\geq 0,\ v_i(t)\geq \min\left(\kappa\,, \frac{v_{\min}}{2} + \frac{bS_{\min}t}{2}\right)\,.$
Setting $c\eqdef \min(\kappa,bS_{\min}/2)$, the result follows.

We prove the second point.  By Prop.~\ref{prop:adam-bounded}, there exists a compact set $K\subset \cZ_+$ s.t.
for every $\eta\geq 0$, every $z\in Z_\infty^\eta(x_0)$ of the form $z(t) = (x(t),m(t),v(t))$
satisfies
$\{(x(t),\hat m(t),\hat v(t)):t\geq 0\}\subset K$,
where $\hat m(t) = m(t)/(1-e^{-a(t+h)})$ and $\hat v(t) = v(t)/(1-e^{-b(t+h)})$.
Denote by $L_S$ the Lipschitz constant of $S$ on the set $\{x:(x,m,v)\in K\}$.
Introduce the constants $M_1\eqdef \sup\{\|m/(\varepsilon + \sqrt v)\|_\infty:(x,m,v)\in K\}$,
$M_2\eqdef \sup\{\|S(x)\|_\infty:(x,m,v)\in K'\}$. These constants being introduced,
the rest of the proof follows the same line as the proof of the first point.
\end{proof}

\subsubsection{Existence}
\label{sec:existence-adam}

\begin{corollary}
\label{coro:existence}
%Let \cref{hyp:coercive,hyp:S>0,hyp:F,hyp:S} hold.
Let Assumptions~\ref{hyp:coercive}, \ref{hyp:S>0}, \ref{hyp:F} and \ref{hyp:S} hold.
Assume that $0< b\leq 4a$.
For every $z_0\in \cZ_+$,$Z_{\infty}^\infty(z_0)\neq\emptyset$.
For every $(z_0,\eta)\in \cZ_0\times (0,+\infty)$,$Z_{\infty}^\eta(z_0)\neq\emptyset$.
\end{corollary}
\begin{proof}
We prove the first point (the proof of the second point follows the same line).
Under Assumptions~\ref{hyp:F} and \ref{hyp:S}, $h_\infty$ is continuous.
Therefore, Cauchy-Peano's theorem guarantees the existence of a solution to the (\ref{eq:ode}) issued from $z_0$,
which we can extend over a maximal interval of existence $[0,T_{\max})$.% \cite[Th. 2.1, Th. 3.1]{hartman1982ordinary}.
We conclude that the solution is global ($T_{\max} = +\infty$) using the boundedness of the solution given by Prop.~\ref{prop:bounded}.
%and \cite[Cor.~3.2]{hartman1982ordinary}.
\end{proof}
%To complete Cor.~\ref{coro:existence}, we need to show the existence of a global solution to~(\ref{eq:odeeta}) in the case $\eta=0$,
%with initial condition $z_0\in \cZ_0$. To this end, we need the following lemma.
\begin{lemma}
\label{lem:equicont-eta}
  %Let \cref{hyp:coercive,hyp:S>0,hyp:F,hyp:S} hold.
  Let Assumptions~\ref{hyp:coercive}, \ref{hyp:S>0}, \ref{hyp:F} and \ref{hyp:S} hold.
  Assume that $0< b\leq 4a$.
Consider $z_0\in \cZ_0$. Denote by $(z_\eta:\eta\in (0,+\infty))$ a family of functions on $[0,+\infty)\to \cZ_+$
s.t. for every $\eta>0$, $z_\eta\in Z_\infty^\eta(z_0)$.
Then,  $(z_\eta)_{\eta>0}$ is equicontinuous.
\end{lemma}
\begin{proof}
For every such solution $z_\eta$, we set $z_\eta(t)=(x_\eta(t),m_\eta(t),v_\eta(t))$ for all $t\geq 0$,
and define $\hat m_\eta$ and $\hat v_\eta$ as in Prop.~\ref{prop:adam-bounded}.
By Prop.~\ref{prop:adam-bounded}, there exists a constant $M_1$ s.t. for all $\eta>0$ and all $t\geq 0$,
$\max(\|x_\eta(t)\|,\|\hat m_\eta(t)\|_\infty,\|\hat v_\eta(t)\|)\leq M_1$.
Using the continuity of $\nabla F$ and $S$, there exists an other finite constant $M_2$ s.t.
$M_2\geq \sup\{\|\nabla F(x)\|_\infty:x\in \bR^d, \|x\|\leq M_1\}$ and
$M_2\geq \sup\{\|S(x)\|_\infty:x\in \bR^d, \|x\|\leq M_1\}$.
For every $(s,t)\in [0,+\infty)^2$,  we have for all $i\in \{1,\dots,d\}$,
% \begin{align*}
% &  |x_{\eta,i}(t)-x_{\eta,i}(s)|\leq \int_s^t\left|\frac{\hat m_{\eta,i}(u)}{\varepsilon + \sqrt{\hat v_{\eta,i}(u)}}\right|du\,\leq \frac{M_1}\varepsilon |t-s|\,,\\
% &  |m_{\eta,i}(t)-m_{\eta,i}(s)|\leq \int_s^ta\left|\partial_i F(x_{\eta}(u))-m_{\eta,i}(u)\right|du\,\leq a(M_1+M_2)|t-s|\,,\\
% &  |v_{\eta,i}(t)-v_{\eta,i}(s)|\leq \int_s^tb\left|S_i (x_{\eta}(u))-v_{\eta,i}(u)\right|du\,\leq b(M_1+M_2)|t-s|\,.
% \end{align*}
$|x_{\eta,i}(t)-x_{\eta,i}(s)|\leq \int_s^t\left|\frac{\hat m_{\eta,i}(u)}{\varepsilon + \sqrt{\hat v_{\eta,i}(u)}}\right|du\,\leq \frac{M_1}\varepsilon |t-s|$,
and similarly $|m_{\eta,i}(t)-m_{\eta,i}(s)| \leq a(M_1+M_2)|t-s|$,
$|v_{\eta,i}(t)-v_{\eta,i}(s)|\leq b(M_1+M_2)|t-s|$\,.
Therefore, there exists a constant $M_3$, independent from $\eta$, s.t. for all $\eta>0$ and all $(s,t)\in [0,+\infty)^2$,
$\|z_\eta(t)-z_\eta(s)\|\leq M_3 |t-s|$.%, which concludes the proof. %As a consequence,  $(z_\eta)_{\eta>0}$ is equicontinuous.
\end{proof}

\begin{proposition}
  \label{prop:existence-adam}
  %Let \cref{hyp:coercive,hyp:S>0,hyp:F,hyp:S} hold.
  Let Assumptions~\ref{hyp:coercive}, \ref{hyp:S>0}, \ref{hyp:F} and \ref{hyp:S} hold.
  Assume that $0< b\leq 4a$.
For every  $z_0\in \cZ_0$, $Z_\infty^0(z_0)\neq \emptyset$ \emph{i.e.},
(\ref{eq:ode}) admits a global solution issued from $z_0$.
\end{proposition}
\begin{proof}
  By Cor.~\ref{coro:existence}, there exists  a family $(z_\eta)_{\eta>0}$ of functions on $[0,+\infty)\to \cZ$
s.t. for every $\eta>0$, $z_\eta\in Z^\eta_\infty(z_0)$.
We set as usual $z_\eta(t)=(x_\eta(t),m_\eta(t),v_\eta(t))$. By Lemma~\ref{lem:equicont-eta},
and the Arzelà-Ascoli theorem, there exists a map $z:[0,+\infty)\to \cZ$ and a sequence $\eta_n\downarrow 0$ s.t.
$z_{\eta_n}$ converges to $z$ uniformly on compact sets, as $n\to\infty$. Considering some fixed scalars $t>s> 0$,
$z(t) = z(s) + \lim_{n\to\infty}\int_s^t h(u+\eta_n, z_{\eta_n}(u))du\,.$
By Prop.~\ref{prop:adam-bounded}, there exists a compact set $K\subset \cZ_+$ s.t.
$\{z_{\eta_n}(t):t\geq 0\}\subset K$ for all $n$.
Moreover, by Lemma~\ref{lem:v-lowerbound}, there exists a constant $c>0$ s.t.
for all $n$ and all $u\geq 0$, $v_{\eta_n,k}(u)\geq c \min(1,u)$.
Denote by $\bar K\eqdef K\cap (\bR^d\times\bR^d\times [c\min(1,s),+\infty)^d)$.
It is clear that $\bar K$ is a compact subset of $\cZ_+^*$.
Since $h$ is continuously differentiable on the set
$[s,t]\times \bar K$, it is Lipschitz continuous on that set. Denote by $L_h$ the corresponding
Lipschitz constant. We obtain:
$$
\int_s^t\|h(u+\eta_n, z_{\eta_n}(u)) - h(u, z(u))\|du \leq L_h\left(\eta_n + \sup_{u\in [s,t]}\|z_{\eta_n}(u)-z(u)\|\right)(t-s)\,,
$$
and the righthand side converges to zero. As a consequence, for all $t>s$,
$
z(t) = z(s) + \int_s^t h(u, z(u))du\,.
$ Moreover, $z(0)=z_0$. This proves that $z\in Z^0_\infty(z_0)$.
\end{proof}

% In Prop.~\ref{prop:semiflow}, we prove the uniqueness of the global solution to (\ref{eq:ode-a}) (case $\eta=+\infty$)
% and show that (\ref{eq:ode-a}) defines a semiflow.
% In Prop.~\ref{prop:unique-adam}, we prove the uniqueness of the global solution to (\ref{eq:ode}) (case $\eta=0$)
% issued from some $z_0\in \cZ_0$.
\subsubsection{Uniqueness}

\begin{proposition}
  \label{prop:unique-adam}
%Let \cref{hyp:coercive,hyp:S>0,hyp:F,hyp:S} hold.
Let Assumptions~\ref{hyp:coercive}, \ref{hyp:S>0}, \ref{hyp:F} and \ref{hyp:S} hold.
Assume  $b\leq 4a$.
For every $z_0\in\cZ_0$, $Z_\infty^0(z_0)$ is a singleton
\emph{i.e.}, there exists a unique global solution to (\ref{eq:ode})
with initial condition $z_0$.
% \begin{enumerate}[{\it i)},leftmargin=*]
% \item For every $z_0\in\cZ_0$, $Z_\infty^0(z_0)$ is a singleton \emph{i.e.},
% there exists a unique global solution to (\ref{eq:ode})
% with initial condition $z_0$.
% \item For every compact subset $K$ of $\cZ_+$, there exist nonnegative constants $c_1,c_2$ s.t.
% for every $(z,z')\in Z_{\infty}^{\infty}(K)^2$,
% $$
% \forall t\geq 0,\ \|z(t)-z'(t)\|^2\leq \|z(0)-z'(0)\|^2 \exp(c_1+c_2t)\,.
% $$
% \end{enumerate}
\end{proposition}
\begin{proof}
  %i)
  Consider solutions $z$ and $z'$ in $Z^0_\infty(z_0)$.  We denote by $(x(t),m(t),v(t))$ the blocks of $z(t)$,
  and we define $(x'(t),m'(t),v'(t))$ similarly.  For all $t>0$, we
  define $\hat m(t)\eqdef m(t)/(1-e^{-at})$, $\hat v(t)\eqdef
  v(t)/(1-e^{-bt})$, and we define $\hat m'(t)$ and $\hat v'(t)$
  similarly.  By Prop.~\ref{prop:adam-bounded}, there exists a compact
  set $K\subset \cZ_+$ s.t.  $ (x(t),\hat m(t),\hat v(t))$ and
  $(x'(t),\hat m'(t),\hat v'(t))$ are both in $K$ for all $t> 0$.  We
  denote by $L_S$ and $L_{\nabla F}$ the Lipschitz constants of $S$
  and $\nabla F$ on the compact set $\{x:(x,m,v)\in K\}$. These
  constants are finite by Assumptions~\ref{hyp:F}
  and \ref{hyp:S}.
We define $M\eqdef \sup\{\|m\|_\infty:(x,m,v)\in K\}$.
Define $u_x(t) \eqdef \|x(t)-x'(t)\|^2$,
$u_m(t)\eqdef \|\hat m(t)-\hat m'(t)\|^2$ and $u_v(t)\eqdef \|\hat v(t)-\hat v'(t)\|^2$.
Let $\delta>0$. Define: $u^{(\delta)}(t) \eqdef u_x(t)+\delta u_m(t)+\delta u_v(t)\,.$
By the chain rule and the Cauchy-Schwarz inequality,
$\dot u_x(t)\leq 2\|x(t)-x'(t)\|\|\frac{\hat m(t)}{\varepsilon +\sqrt{\hat v(t)}}-\frac{\hat m'(t)}{\varepsilon +\sqrt{\hat v'(t)}}\|$. Thus,
%% \begin{equation*}
%%   \dot u_x(t) \leq 2\|x(t)-x'(t)\|\left(\varepsilon^{-1}\left\|\hat m(t)-\hat m'(t)\right\|
%% +M\varepsilon^{-2}\left\|\sqrt{\hat v(t)}-\sqrt{\hat v'(t)}\right\|\right)\,.
%% \end{equation*}
%\begin{equation*}
%  \dot u_x(t) \leq 2\|x(t)-x'(t)\|\left(\varepsilon^{-1}\left\|\hat m(t)-\hat m'(t)\right\|
%+M\varepsilon^{-2}\left\|\frac{\hat v(t)-{\hat v'(t)}}{\sqrt{\hat v(t)}+\sqrt{\hat v'(t)}}\right\|\right)\,.
%\end{equation*}
%%For every $i\in \{1,\dots,d\}$,
%%$\left|\sqrt{\hat v_i(t)}-\sqrt{\hat v'_i(t)}\right| = \frac{|\hat v_i(t)-{\hat v'_i(t)}|}{|\sqrt{\hat v_i(t)}+\sqrt{\hat v'_i(t)}|}$.
using Lemma~\ref{lem:v-lowerbound}, there exists $c>0$ s.t.
%for all $t\geq 0$, for every $i\in \{1,\dots,d\}$, $\hat v_i(t)\wedge \hat v'_i(t)\geq c\min(1,t)$. Thus,
\begin{equation*}
  \dot u_x(t)\leq 2\|x(t)-x'(t)\|\left(\varepsilon^{-1}\left\|\hat m(t)-\hat m'(t)\right\|
+\frac M{2\varepsilon^2\sqrt{c\min(1,t)}}\left\|{\hat v(t)}-{\hat v'(t)}\right\|\right)\,.
\end{equation*}
For any $\delta>0$,
$2\|x(t)-x'(t)\|\,\|\hat m(t)-\hat m'(t)\|\leq \delta^{-1/2}(u_x(t)+\delta u_m(t))\leq \delta^{-1/2}u^{(\delta)}(t)$.
Similarly, $2\|x(t)-x'(t)\|\,\|\hat v(t)-\hat v'(t)\|\leq \delta^{-1/2}u^{(\delta)}(t)$.
Thus, for any $\delta>0$,
\begin{align}
\label{eq:ux}
  \dot u_x(t)&\leq \left(\frac 1{\varepsilon\sqrt \delta}+\frac M{2\varepsilon^2\sqrt{\delta c\min(1,t)}}\right) u^{(\delta)}(t)\,.
\end{align}
We now study $u_m(t)$. For all $t>0$, we obtain after some algebra:
$\frac {d}{dt}\hat m(t) = a(\nabla F(x(t)) - \hat m(t))/(1-e^{-at})\,.$
Therefore,
$
\dot u_m(t) \leq \frac{2aL_{\nabla F}}{1-e^{-at}}\|\hat m(t)-\hat m'(t)\|\,\|x(t) -x'(t)\|\,.
$
% \begin{align*}
%   \dot u_m(t) %&= \frac{2a}{1-e^{-at}}\ps{\hat m(t)-\hat m'(t), \nabla F(x(t)) - \hat m(t)-\nabla F(x'(t)) + \hat m'(t)} \\
% &\leq  \frac{2aL_{\nabla F}}{1-e^{-at}}\|\hat m(t)-\hat m'(t)\|\,\|x(t) -x'(t)\|\,.
% \end{align*}
For any $\theta>0$, it holds that $2\|\hat m(t)-\hat m'(t)\|\,\|x(t) -x'(t)\|\leq \theta u_x(t)+ \theta^{-1}u_m(t)$.
In particular, letting $\theta\eqdef 2L_{\nabla F}$, we obtain that for all $\delta>0$,
\begin{equation}
\resizebox{\hsize}{!}{$
\delta \dot u_m(t)\leq  \frac{a }{2(1-e^{-at})}\left(4\delta L_{\nabla F}^2 u_x(t)+ \delta u_m(t)\right)
                    \leq  \left(\frac a2+\frac 1{2t}\right)\left(4\delta L_{\nabla F}^2 u_x(t)+ \delta u_m(t)\right)\,,
$}
\label{eq:um}
\end{equation}
% \begin{align}
%  \delta \dot u_m(t)&\leq  \frac{a }{2(1-e^{-at})}\left(4\delta L_{\nabla F}^2 u_x(t)+ \delta u_m(t)\right) %\nonumber\\
%                    %&\leq  \left(\frac a2+\frac 1{2t}\right)\left(4\delta L_{\nabla F}^2 u_x(t)+ \delta u_m(t)\right)\,,
%                    \leq  \left(\frac a2+\frac 1{2t}\right)\left(4\delta L_{\nabla F}^2 u_x(t)+ \delta u_m(t)\right)\,,
% \label{eq:um}
% \end{align}
where the last inequality is due to the fact that $y/(1-e^{-y})\leq 1+y$ for all $y>0$.
Using the exact same arguments, we also obtain that
\begin{align}
 \delta \dot u_v(t)&\leq  \left(\frac b2+\frac 1{2t}\right)\left(4\delta L_{S}^2 u_x(t)+ \delta u_m(t)\right)\,.
\label{eq:uv}
\end{align}
We now choose any $\delta$ s.t. $4\delta \leq  1/\max(L_S^2,L_{\nabla F}^2)$.
Then, Eq.~(\ref{eq:um}) and~(\ref{eq:uv}) respectively imply that
$\delta \dot u_m(t)\leq 0.5(a+t^{-1})u^{(\delta)}(t)$ and
$\delta \dot u_v(t)\leq 0.5(b+t^{-1})u^{(\delta)}(t)$.
Summing these inequalities along with Eq.~(\ref{eq:ux}), we obtain that for every $t>0$,
$\dot u^{(\delta)}(t) \leq \psi(t) u^{(\delta)}(t)\,$,
where: $\psi(t) \eqdef \frac{a+b}2+\frac 1{\varepsilon\sqrt \delta}+\frac M{2\varepsilon^2\sqrt{\delta c\min(1,t)}}
+ \frac 1t\,.$
From Gr\"onwall's inequality, it holds that for every $t>s>0$,
$u^{(\delta)}(t)\leq u^{(\delta)}(s)\exp\left(\int_s^t \psi(s')ds'\right)\,$.
We first consider the case where $t\leq 1$. We set $c_1\eqdef  (a+b)/2+(\varepsilon\sqrt \delta)^{-1}$
and $c_2\eqdef M/(\varepsilon^2\sqrt{\delta c})$. With these notations,
$\int_s^t \psi(s')ds' \leq c_1t+c_2\sqrt t + \ln \frac ts\,.$
Therefore, $u^{(\delta)}(t)\leq \frac{u^{(\delta)}(s)}{s}
\exp\left(c_1t+c_2\sqrt t + \ln t\right)\,$.
By Lemma~\ref{lem:m-v-derivables-en-zero}, recall that $\dot x(0)$ and
$\dot x'(0)$ are both well defined (and coincide). Thus,
$$
u_x(s) = \|x(s)-x'(s)\|^2
\leq  2\|x(s)-x(0)-\dot x(0)s\|^2+2\|x'(s)-x'(0)-\dot x'(0)s\|^2\,.
$$
It follows that $u_x(s)/s^2$ converges to zero as $s\downarrow 0$.
We now show the same kind of result for $u_m(s)$ and $u_v(s)$.
Consider $i\in \{1,\dots,d\}$. By the mean value theorem, there exists $\tilde s$ (resp. $\tilde s'$) in
%the interval $[0,t]$ s.t. $m_i(s)=\dot m_i(\tilde s)s$ (resp. $m_i'(s)=\dot m_i'(\tilde s')s$).
$[0,t]$ s.t. $m_i(s)=\dot m_i(\tilde s)s$ (resp. $m_i'(s)=\dot m_i'(\tilde s')s$).
Thus, $\hat m_i(s) = \frac{as}{1-e^{-as}} \left(\partial_i F(x(\tilde s))-m_i(\tilde s)\right)$,
and a similar equality holds for $\hat m_i'(s)$. Then,
given that $\|x(\tilde s)-x'(\tilde s')\| \vee \|m(\tilde s)-m'(\tilde s')\| \leq \|z(\tilde s)-z'(\tilde s')\|$
%and using that $\tilde s\leq s$ and $\tilde s'\leq s$, it follows that:
, $\tilde s\leq s$ and $\tilde s'\leq s$, %it follows that:
% %As a consequence,
% \begin{equation*}
%   \resizebox{\hsize}{!}{$
%   |\hat m_i(s) -\hat m_i'(s) | %&\leq \frac{as}{1-e^{-as}}
% %\left(|\partial_i F(x(\tilde s))-\partial_i F(x'(\tilde s'))|+|m_i(\tilde s)-m_i'(\tilde s')|\right)\\
% \leq \frac{as}{1-e^{-as}} \left(L_{\nabla F}\|x(\tilde s)-x'(\tilde s')\|+|m_i(\tilde s)-m_i'(\tilde s')|\right)
% \leq \frac{2a(L_{\nabla F}\vee 1)s}{1-e^{-as}} \|z(\tilde s)-z'(\tilde s')\|\,.
% $}
% \end{equation*}
%%\begin{align*}
%%  |\hat m_i(s) -\hat m_i'(s) | %&\leq \frac{as}{1-e^{-as}}
%%%\left(|\partial_i F(x(\tilde s))-\partial_i F(x'(\tilde s'))|+|m_i(\tilde s)-m_i'(\tilde s')|\right)\\
%%&\leq \frac{as}{1-e^{-as}} \left(L_{\nabla F}\|x(\tilde s)-x'(\tilde s')\|+|m_i(\tilde s)-m_i'(\tilde s')|\right)\\
%%&\leq \frac{2a(L_{\nabla F}\vee 1)s}{1-e^{-as}} \|z(\tilde s)-z'(\tilde s')\|\,,
%%\end{align*}
%where we used $\|x(\tilde s)-x'(\tilde s')\|\leq \|z(\tilde s)-z'(\tilde s')\|$
%and $|m_i(\tilde s)-m_i'(\tilde s')|\leq \|z(\tilde s)-z'(\tilde s')\|$ to obtain the last inequality.
%Using that $\tilde s\leq s$ and $\tilde s'\leq s$, it follows that:
$$
\frac{|\hat m_i(s) -\hat m_i'(s) |}s
\leq \frac{2a(L_{\nabla F}\vee 1)s}{1-e^{-as}} \left(\frac{\|z(\tilde s)-z(0)\|}{\tilde s}+\frac{\|z'(\tilde s')-z'(0)\|}{\tilde s'}\right)\,.
$$
By Lemma~\ref{lem:m-v-derivables-en-zero}, $z$ and $z'$ are differentiable at point zero.
Then, the above inequality gives $\limsup_{s\downarrow 0}\frac{|\hat m_i(s) -\hat m_i'(s) |}s \leq 4(L_{\nabla F}\vee 1)\|\dot z(0)\|$ %\,.$
%Thus, the righthand side of the above inequality has a limit as $s\downarrow 0$:
%$\limsup_{s\downarrow 0}\frac{|\hat m_i(s) -\hat m_i'(s) |}s \leq 4(L_{\nabla F}\vee 1)\|\dot z(0)\|\,.$
%Thus,
and
$%$$
\limsup_{s\downarrow 0}\frac{u_m(s)}{s^2} \leq 16d(L_{\nabla F}^2\vee 1)\|\dot z(0)\|^2\,.
$%$$
Therefore, $u_m(s)/s$ converges to zero as $s\downarrow 0$.
By similar arguments, it can be shown that
$\limsup_{s\downarrow 0}{u_v(s)}/{s^2} \leq 16d(L_{S}^2\vee 1)\|\dot z(0)\|^2$,
thus $\lim u_v(s)/s=0$.
Finally, we obtain that
${u^{(\delta)}(s)}/{s}$ converges to zero as $s\downarrow 0$.
Letting $s$ tend to zero, we obtain that for every
$t\leq 1$, $u^{(\delta)}(t)=0$. Setting $s=1$ and $t>1$,
and noting that $\psi$ is integrable on $[1,t]$, it follows that $u^{(\delta)}(t)=0$ for all $t>1$.
This proves that $z=z'$.
\end{proof}

We recall that a semiflow $\Phi$ on the space $(E,\sd)$ is a continuous map
$\Phi$ from $[0,+\infty)\times E$ to $E$ defined by $(t,x) \mapsto \Phi(t,x) = \Phi_t(x)$
%\begin{align*}
%  \Phi:&[0,+\infty)\times E \to E \\
%&(t,x) \mapsto \Phi(t,x) = \Phi_t(x)
%\end{align*}
such that $\Phi_0$ is the identity and $\Phi_{t+s} = \Phi_t\circ\Phi_s$ for all $(t,s)\in [0,+\infty)^2$.

\begin{proposition}
  \label{prop:semiflow}
%Let  \cref{hyp:coercive,hyp:S>0,hyp:F,hyp:S} hold.
Let Assumptions~\ref{hyp:coercive}, \ref{hyp:S>0}, \ref{hyp:F} and \ref{hyp:S} hold.
Assume that $0< b\leq 4a$.
The map $Z_{\infty}^{\infty}$ is single-valued on $\cZ_+\to C([0,+\infty),\cZ_+)$
\emph{i.e.}, there exists a unique global solution to~(\ref{eq:ode-a})
starting from any given point in $\cZ_+$.
Moreover, the following map is a semiflow:
\begin{equation}
  \label{eq:flot}
  \begin{array}[h]{rcl}
    \Phi:[0,+\infty)\times \cZ_+&\to& \cZ_+ \\
(t,z) &\mapsto& Z_{\infty}^{\infty}(z)(t)
  \end{array}
\end{equation}
%is a semiflow.
\end{proposition}
% In the sequel, the semiflow $(t,z) \mapsto Z_{\infty}^{\infty}(z)(t)$ is still denoted by $\flot$ with a slight notation abuse \emph{i.e.},
% we write $Z_{\infty}^{\infty}(t,z)=Z_{\infty}^{\infty}(z)(t)$.
\begin{proof}
The result is a direct consequence of Lemma~\ref{prop:unique-adam}.
\end{proof}

%\section{Convergence of the Trajectories}
\subsection{Proof of Th.~\ref{th:cv-adam}}
\label{sec:convergence}
\subsubsection{Convergence of the semiflow}

%In this paragraph,
We first recall some useful definitions and results.
Let $\Psi$ represent any semiflow on an arbitrary metric space $(E,\sd)$.
A point $z\in E$ is called an \emph{equilibrium point} of the semiflow $\Psi$ if $\Psi_t(z)=z$ for all $t\geq 0$.
We denote by $\Lambda_\Psi$ the set of equilibrium points of~$\Psi$.
A continuous function $\sV:E\to\bR$ is called a \emph{Lyapunov function} for the semiflow $\Psi$
if $\sV(\Psi_t(z))\leq \sV(z)$ for all $z\in E$ and all $t\geq 0$.
It is called a \emph{strict Lyapunov function} if, moreover,
$
\{ z\in E\,:\, \forall t\geq 0,\,\sV(\Psi_t(z))=\sV(z) \}= \Lambda_\Psi
$.
If $\sV$ is a strict Lyapunov function for $\Psi$ and if $z\in E$ is a point s.t. $\{\Psi_t(z):t\geq 0\}$ is relatively compact,
then it holds that $\Lambda_\Psi\neq \emptyset$ and $\sd(\Psi_t(z),\Lambda_\Psi)\to 0$, see \cite[Th.~2.1.7]{haraux1991systemes}.
A continuous function $z:[0,+\infty)\to E$ is said to be an asymptotic pseudotrajectory (APT) %(APT, \cite{ben-hir-96})
for the semiflow $\Psi$ if for every $T\in (0,+\infty)$,
$
\lim_{t\to+\infty} \sup_{s\in [0,T]} \sd(z(t+s),\Psi_s(z(t))) = 0\,.
$
%{\color{red} Pourquoi utiliser la distance $\sd$ ici ? Ce n'est pas simplement la norme ?
%Par ailleurs, on ne définit plus la distance de la cv compacte ? Plus besoin ?}
%{\color{blue} $\sd$ car les défs de ce paragraphe sont données dans le cas général d'un espace métrique $(E,\sd)$.
%Il me semble que l'on a jamais défini la dist de la cv compacte.}
The following result follows from \cite[Th.~5.7]{ben-(cours)99} and \cite[Prop.~6.4]{ben-(cours)99}.
\begin{proposition}[\cite{ben-(cours)99}]\hfill\\
\label{prop:benaim}
  Consider a semiflow $\Psi$ on $(E,d)$ and a map $z:[0,+\infty)\to E$. Assume the following:
  \begin{enumerate}[{\it i)}]
  \item $\Psi$ admits a strict Lyapunov function $\sV$.
  \item The set $\Lambda_\Psi$ of equilibrium points of $\Psi$ is compact.
  \item  $\sV(\Lambda_\Psi)$ has an empty interior.
  \item $z$ is an APT of $\Psi$.
  \item $z([0,\infty))$ is relatively compact.
  \end{enumerate}
Then, $
\bigcap_{t\geq 0}\overline{z([t,\infty))}$ is a compact connected subset of $\Lambda_\Psi$\,.
\end{proposition}

For every $\delta>0$ and every $z = (x,m,v)\in \cZ_+$, define:
\begin{equation}
W_\delta(x,m,v) \eqdef V_{\infty}(x,m,v) - \delta \ps{\nabla F(x),m} + \delta \|S(x)-v\|^2\,,\label{eq:Wdelta}
\end{equation}
where we recall that $V_\infty(z)\eqdef \lim_{t\to \infty} V(t,z)$ for every $z\in \cZ_+$ and $V$ is defined by Eq.(\ref{eq:V}).
Consider the set $\cE\eqdef h_\infty^{-1}(\{0\})$ of all equilibrium points of (\ref{eq:ode-a}), namely:
$\cE = \{(x,m,v)\in \cZ_+:\nabla F(x)=0,m=0,v=S(x)\}\,$.
The set $\cE$ is non-empty by Assumption~\ref{hyp:coercive}.

\begin{proposition}
  \label{prop:Wstrict}
%Let  \cref{hyp:coercive,hyp:S>0,hyp:F,hyp:S} hold.
Let Assumptions~\ref{hyp:coercive}, \ref{hyp:S>0}, \ref{hyp:F} and \ref{hyp:S} hold.
Assume that $0< b\leq 4a$.
Let $K\subset \cZ_+$ be a compact set. Define $K'\eqdef \overline{\{\flot(t,z):t\geq 0, z\in K\}}$.
Let $\bflot:[0,+\infty)\times K'\to K'$ be the restriction of the semiflow $\flot$ to $K'$ \emph{i.e.},
$\bflot(t,z) = \flot(t,z)$ for all  $t\geq 0, z\in K'$. Then,
\begin{enumerate}[{\it i)}]
\item $K'$ is compact.
\item $\bflot$ is well defined and is a semiflow on $K'$.
\item The set of equilibrium points of $\bflot$ is equal to $\cE\cap K'$.
\item There exists $\delta>0$ s.t. $W_\delta$ is a strict Lyapunov function for the semiflow $\bflot$.
% {\color{blue}
% Moreover, if $z(t) = (x(t),m(t),v(t)) \eqdef \bflot_{t}(z)$, there exists a constant $c>0$ such that:
% \begin{equation}
% \label{eq:derivative_lyap}
% \forall t>0,\ \  \cL_{W_\delta}(t) \leq -c\left( \|m(t)\|^2 + \|\nabla  F(x(t))\|^2 + \|S(x(t))-v(t)\|^2\right)\,,
% \end{equation}
% where for all $t>0$, $ \cL_{\mathsf W}(t) \eqdef \limsup_{s\to 0} s^{-1}({\mathsf W}(\bflot_{t+s}(z)) -   {\mathsf W}(\bflot_{t}(z)))\,.$
% }
\end{enumerate}
\end{proposition}

\begin{proof}
The first point is a consequence of Prop.~\ref{prop:bounded}.
%The second point is a consequence of Prop.~\ref{prop:semiflow}.
The second point stems from Prop.~\ref{prop:semiflow}.
The third point is immediate from the definition of $\cE$ and the fact that $\bflot$ is valued in $K'$.
We now prove the last point.
Consider $z\in K'$ and write $\bflot_t(z)$ under the form $\bflot_t(z) = (x(t),m(t),v(t))$. For \emph{any} map ${\mathsf W}:\cZ_+\to\bR$, define
for all $t>0$, $ \cL_{\mathsf W}(t) \eqdef \limsup_{s\to 0} s^{-1}({\mathsf W}(\bflot_{t+s}(z)) -   {\mathsf W}(\bflot_{t}(z)))\,.$
Introduce $G(z)\eqdef -\ps{\nabla F(x),m}$ and $H(z)\eqdef \|S(x)-v\|^2$ for every $z=(x,m,v)$.
Consider $\delta>0$ (to be specified later on).
We study $\cL_{W_\delta} = \cL_V + \delta \cL_{G} + \delta \cL_{H}$.
Note that $\bflot_t(z)\in K'\cap \cZ_+^*$ for all $t>0$ by Lemma~\ref{lem:v-positif}. Thus, $t\mapsto V_\infty(\bflot_t(z))$
is differentiable at any point $t>0$ and the derivative coincides with $\cL_V(t) = \dot V_\infty(\bflot_t(z))$.
% By Lemma~\ref{lem:V},
% \begin{equation*}
% %  \dot V_\infty(\bflot_t(z))
% \cL_V(t) = \ps{\nabla V_\infty(\bflot_t(z)),h_{\infty}(\bflot_t(z))} \leq -\frac{\varepsilon}{(\varepsilon+\sqrt{\|v(t)\|_\infty})^2} \|m(t)\|^2 \,. %\ & \leq -c \left\|m(t)\right\|^2\,,
% \end{equation*}
% Define $C_1\eqdef \sup\{\|v\|_\infty:(x,m,v)\in K'\}$. Then, $\cL_V(t)\leq -\varepsilon(\varepsilon+\sqrt{C_1})^{-2} \left\|m(t)\right\|^2$.
Define $C_1\eqdef \sup\{\|v\|_\infty:(x,m,v)\in K'\}$.
Then, by Lemma~\ref{lem:V}, $\cL_V(t)\leq -\varepsilon(\varepsilon+\sqrt{C_1})^{-2} \left\|m(t)\right\|^2$.
Let $L_{\nabla F}$ be the Lipschitz constant of $\nabla F$
on $\{x:(x,m,v)\in K'\}$.
%We now study $\cL_G$.
For every $t>0$,
\begin{align*}
  \cL_G(t)%&=  \limsup_{s\to 0} s^{-1}(-\ps{\nabla F(x(t+s)),m(t+s)}+\ps{\nabla F(x(t)),m(t)}) \\
&\leq  \limsup_{s\to 0} s^{-1}\|\nabla F(x(t))- \nabla F(x(t+s))\| \|m(t+s)\| - \ps{\nabla F(x(t)),\dot m(t)}\\
%&\leq  L_{\nabla F}\|\dot x(t)\| \|m(t)\| - \ps{\nabla F(x(t)),\dot m(t)}\\
&\leq  L_{\nabla F}\varepsilon^{-1}\|m(t)\|^2 - a\|\nabla F(x(t))\|^2  + a\ps{\nabla F(x(t)),m(t)}\\
&\leq - \frac a2\|\nabla F(x(t))\|^2  +  \left(\frac a2+\frac{L_{\nabla F}}{\varepsilon}\right)\|m(t)\|^2\,.
\end{align*}
Denote by $L_S$ the Lipschitz constant of $S$ on $\{x:(x,m,v)\in K'\}$.
%We now study $\cL_H$.
For every $t>0$,
\begin{eqnarray*}
  \cL_H(t)%&=&  \limsup_{s\to 0} s^{-1}(\|S(x(t+s))-v(t+s)\|^2-\|S(x(t))-v(t)\|^2) \\
&=&  \limsup_{s\to 0} s^{-1}(\|S(x(t+s))-S(x(t)) + S(x(t))-v(t+s)\|^2-\|S(x(t))-v(t)\|^2)\\
% &=& \limsup_{s\to 0} s^{-1}(\|S(x(t+s))-S(x(t))\|^2  + 2\ps{S(x(t+s))-S(x(t)),S(x(t))-v(t+s)})\\
% & &  + \lim_{s\to 0} s^{-1}(\|S(x(t))-v(t+s)\|^2-\|S(x(t))-v(t)\|^2) \\
&=& - 2\ps{S(x(t))-v(t),\dot v(t)}%\\
%&&+\limsup_{s\to 0} 2s^{-1}\ps{S(x(t+s))-S(x(t)),S(x(t))-v(t+s)} \\
+\limsup_{s\to 0} 2s^{-1}\ps{S(x(t+s))-S(x(t)),S(x(t))-v(t+s)} \\
%&\leq  - 2b\|S(x(t))-v(t)\|^2+2L_S \|\dot x(t)\| \|S(x(t))-v(t)\| \\
&\leq&   - 2b\|S(x(t))-v(t)\|^2+2L_S\varepsilon^{-1} \|m(t)\| \|S(x(t))-v(t)\|\,.
\end{eqnarray*}
% Expanding the square norm, we obtain:
% \begin{multline*}
%   \cL_H(t)= \limsup_{s\to 0} s^{-1}(\|S(x(t+s))-S(x(t))\|^2  + 2\ps{S(x(t+s))-S(x(t)),S(x(t))-v(t+s)})\\
%    + \lim_{s\to 0} s^{-1}(\|S(x(t))-v(t+s)\|^2-\|S(x(t))-v(t)\|^2) \,.
% \end{multline*}
%The second term in the righthand side coincides with $-2\ps{S(x(t))-v(t),\dot v(t)}$.
% Note that $s^{-1}(\|S(x(t+s))-S(x(t))\|^2)\leq L_S^2 s\| s^{-1}(x(t+s)-x(t))\|^2$ which converges
% to zero as $s\to 0$. Thus,
% \begin{align*}
%   \cL_H(t)&= - 2\ps{S(x(t))-v(t),\dot v(t)}\\
% &+\limsup_{s\to 0} 2s^{-1}\ps{S(x(t+s))-S(x(t)),S(x(t))-v(t+s)} \\
% &\leq  - 2b\|S(x(t))-v(t)\|^2+2L_S \|\dot x(t)\| \|S(x(t))-v(t)\| \\
% &\leq   - 2b\|S(x(t))-v(t)\|^2+2L_S\varepsilon^{-1} \|m(t)\| \|S(x(t))-v(t)\|\,.
% \end{align*}
Using that $2 \|m(t)\| \|S(x(t))-v(t)\|\leq \frac {L_S}{b\varepsilon}\|m(t)\|^2 + \frac {b\varepsilon}{L_S}\|S(x(t))-v(t)\|^2$, we obtain
$
\cL_H(t) \leq - b\|S(x(t))-v(t)\|^2+\frac {L_S^2}{b\varepsilon^2}\|m(t)\|^2\,.
$
Hence,
%Recalling that $\cL_{W_\delta} = \cL_V + \delta \cL_{G} + \delta \cL_{H}$, we have shown that
for every $t>0$,
$$
  \cL_{W_\delta}(t) \leq -M(\delta)
  \|m(t)\|^2  - \frac {a\delta}2\|\nabla
  F(x(t))\|^2 - \delta b \|S(x(t))-v(t)\|^2\,.
$$
where $M(\delta)\eqdef \varepsilon(\varepsilon+\sqrt{C_1})^{-2} -\frac {\delta
      L_S^2}{b\varepsilon^2} - \delta \left(\frac a2+\frac{L_{\nabla
        F}}{\varepsilon}\right)\,.$
Choosing $\delta$ s.t. $M(\delta)>0$,
\begin{equation}
\forall t>0,\ \  \cL_{W_\delta}(t) \leq -c\left( \|m(t)\|^2 + \|\nabla  F(x(t))\|^2 + \|S(x(t))-v(t)\|^2\right)\,,\label{eq:ae-derivee}
\end{equation}
where $c \eqdef \min\{ M(\delta), \frac {a\delta}2, \delta b\}$.
It can easily be seen that for every $z\in K'$,  $t\mapsto W_\delta(\bflot_t(z))$ is Lipschitz continuous, hence absolutely continuous.
Its derivative almost everywhere coincides with $\cL_{W_\delta}$, which is non-positive.
Thus, $W_\delta$ is a Lyapunov function for $\bflot$.
We prove that the Lyapunov function is strict.
Consider $z\in K'$ s.t. $W_\delta(\bflot_t(z))=W_\delta(z)$ for all $t>0$.
The derivative almost everywhere of  $t\mapsto W_\delta(\bflot_t(z))$ is identically zero,
and by Eq. (\ref{eq:ae-derivee}), this implies that
 $-c\left(\|m_t\|^2 + \|\nabla F(x_t)\|^2 +\|S(x_t)-v_t\|^2\right)$ is equal to zero
for every $t$ a.e. (hence, for every~$t$, by continuity of $\bflot$).
In particular for $t=0$, $m=\nabla F(x)=0$ and $S(x)-v=0$.
Hence, $z\in h_\infty^{-1}(\{0\})$.
\end{proof}

\begin{corollary}
  \label{coro:cv}
%Let  \cref{hyp:coercive,hyp:S>0,hyp:F,hyp:S} hold.
Let Assumptions~\ref{hyp:coercive}, \ref{hyp:S>0}, \ref{hyp:F} and \ref{hyp:S} hold.
Assume that $0< b\leq 4a$.
For every $z\in \cZ_+$, $\lim_{t\to \infty} \sd(\Phi(z,t),\cE)=0\,.$
\end{corollary}
\begin{proof}
Use Prop.~\ref{prop:Wstrict} with $K\eqdef\{z\}$.
and \cite[Th.~2.1.7]{haraux1991systemes}.
\end{proof}

\subsubsection{Asymptotic Behavior of the Solution to (\ref{eq:ode})}
\label{sec:cv-non-autonomous}

% The proof of the following lemma is left to the reader:
% \begin{lemma}
% \label{lem:ha-h}
%   Let  \cref{hyp:F} hold true.
% For every compact set $K\subset \cZ_+^*$,
% $$
% \lim_{t\to\infty} \sup_{s\geq 0}\sup_{z\in K} \|h_{\infty}(z)-h(t+s,z)\|=0\,.
% $$
% \end{lemma}

\begin{proposition}[APT]
  %Let  \cref{hyp:coercive,hyp:S>0,hyp:F,hyp:S} hold true.
  Let Assumptions~\ref{hyp:coercive}, \ref{hyp:S>0}, \ref{hyp:F} and \ref{hyp:S} hold true.
  Assume that $0< b\leq 4a$.
Then, for every $z_0\in \cZ_0$, $Z^0_\infty(z_0)$ is an asymptotic pseudotrajectory
of the semiflow  $\Phi$ given by~(\ref{eq:flot}).
\label{prop:apt}
\end{proposition}
\begin{proof}
  Consider  $z_0\in \cZ_0$, $T\in (0,+\infty)$ and define $z\eqdef Z_\infty^0(z_0)$.
Consider $t\geq 1$. For every $s\geq 0$, define $\Delta_t(s)\eqdef  \|z(t+s)-\flot(z(t))(s)\|$.
The aim is to prove that $\sup_{s\in [0,T]}\Delta_t(s)$ tends to zero as~$t\to\infty$.
Putting together Prop.~\ref{prop:adam-bounded} and Lemma~\ref{lem:v-lowerbound},
the set $K\eqdef \overline{\{z(t):t\geq 1\}}$
is a compact subset of $\cZ_+^*$.
% For every $s\in [0,T]$,
% \begin{multline*}
% \Delta_t(s)\leq \int_0^s
%   \|h(t+s',z(t+s'))-h_\infty(z(t+s'))\|ds'  \\+ \int_0^s
%   \|h_\infty(z(t+s'))-h_{\infty}(\flot(z(t))(s'))\|ds'\,.
% \end{multline*}
Define $C(t)\eqdef \sup_{s\geq 0}\sup_{z'\in K}\|h(t+s,z')-h_\infty(z')\|$.
%By Lemma \ref{lem:ha-h},
It can be shown that $\lim_{t\to\infty} C(t)=0$.
We obtain that for every $s\in [0,T]$, $
\Delta_t(s)\leq T C(t)  + \int_0^s
  \|h_\infty(z(t+s'))-h_{\infty}(\flot(z(t))(s'))\|ds'\,.$
By Lemma~\ref{lem:v-lowerbound}, $K'\eqdef \overline{\bigcup_{z'\in\flot(K)}z'([0,\infty))}$ is a compact subset of $\cZ_+^*$.
It is immediately seen from the definition that $h_{\infty}$ is Lipschitz continuous on every compact subset of $\cZ_+^*$, hence on $K\cup K'$. Therefore, there exists a constant $L$, independent from $t,s$, s.t.
$
\Delta_t(s)\leq T C(t)  + \int_0^s L \Delta_t(s')ds'\qquad (\forall t\geq 1, \forall s\in[0,T])\,.
$
Using Gr\"onwall's lemma, it holds that for all $s\in [0,T]$,
$
\Delta_t(s)\leq TC(t) e^{Ls}\,.
$
As a consequence,
$\sup_{s\in [0,T]}\Delta_t(s)\leq TC(t) e^{LT}$ and the righthand side converges to zero
as $t\to\infty$.
\end{proof}

\subsubsection*{End of the Proof of Th.~\ref{th:cv-adam}}

%Let $K\subset \cZ_+$ be a compact set.
% Let $\bflot:[0,+\infty)\times K'\to K'$ be the restriction of the semiflow $\flot$ to $K'$ \emph{i.e.},
% $\bflot(t,z) = \flot(t,z)$ for all  $t\geq 0, z\in K'$. Then,
% \begin{enumerate}[{\it i)}]
% \item $K'$ is compact.
% \item $\bflot$ is well defined and is a semiflow on $K'$.
% \item The set of equilibrium points of $\bflot$ is equal to $\cE\cap K'$.
% \item There exists $\delta>0$ s.t. $W_\delta$ is a strict Lyapunov function for the semiflow $\bflot$.
% \end{enumerate}
By Prop.~\ref{prop:adam-bounded}, the set $K\eqdef \overline{Z^0_\infty(z_0)([0,\infty))}$ is a compact subset of $\cZ_+$.
Define $K'\eqdef \overline{\{\flot(t,z):t\geq 0, z\in K\}}\,,$ and let
 $\bflot:[0,+\infty)\times K'\to K'$ be the restriction $\flot$ to $K'$.
By Prop.~\ref{prop:Wstrict}, there exists $\delta>0$ s.t.
$W_\delta$ is  a strict Lyapunov function for the semiflow $\bflot$.
Moreover, the set of equilibrium points coincides with $\cE\cap K'$.
In particular, the equilibrium points of $\bflot$ form a compact set.
By Prop.~\ref{prop:apt}, $Z^0_\infty(z_0)$ is an APT of $\bflot$.
Note that every $z\in\cE$ can be written under the form
$z=(x,0,S(x))$ for some $x\in \cS$.
From the definition of $W_\delta$ in (\ref{eq:Wdelta}),
$W_\delta(z)=W_\delta(x,0,S(x)) =V_{\infty}(x,0,S(x)) = F(x)$.
Since $F(\cS)$ is assumed to have an empty interior, the same holds
for $W_\delta(\cE\cap K')$. By Prop.~\ref{prop:benaim},
$
\bigcap_{t\geq 0}\overline{Z^0_\infty(z_0)([t,\infty))} \subset \cE\cap K'\,.
$
The set in the righthand side coincides with the set of limits of convergent
sequences of the form $Z^0_\infty(z_0)(t_n)$ for $t_n\to\infty$.
As $Z^0_\infty(z_0)([0,\infty))$ is a bounded set, $\sd(Z^0_\infty(z_0)(t),\cE)$ tends to zero.

\subsection{Proof of Th.~\ref{thm:asymptotic_rates}}
\label{sec:cont_asymptotic_rates}

The proof follows the path of \cite[Th.~10.1.6, Th.~10.2.3]{harauxjendoubi2015},
but requires specific adaptations to deal with the dynamical system at hand.
Define for all $\delta>0$, $t>0$, and $z=(x,m,v)$,
\begin{equation}
\tilde W_\delta(t,(x,m,v)) \eqdef V(t,(x,m,v)) - \delta \ps{\nabla F(x),m} + \delta \|S(x)-v\|^2\,.\label{eq:Wdelta-t}
\end{equation}
The function $\tilde W_\delta$ is the non-autonomous version of the function (\ref{eq:Wdelta}).
%previously used to establish the convergence of the (autonomous) semiflow to its equilibrium points.
Consider a fixed $x_0\in \bR^d$, and define $w_{\delta}(t)\eqdef \tilde W_\delta(t,z(t))$
where $z(t)=(x(t),m(t),v(t))$ is the solution to~(\ref{eq:ode}) with initial condition $(x_0,0,0)$.
The proof uses the following steps.

\begin{enumerate}[{\sl i)},leftmargin=11pt]
\item \textit{Upper-bound on  $w_\delta(t)$.}
%Recalling the definition of $V$ in Eq.~(\ref{eq:V}),
From Eq.~(\ref{eq:V}),
we obtain that for every $t\geq 1$,
$%$$
V(t,z(t))\leq |F(x(t))|+\frac {\|m(t)\|^2}{2a\varepsilon (1-e^{-a})}\,.
$%$$
%Consider now the last two remaining terms in the righthand side of Eq.~(\ref{eq:Wdelta-t}).
Using $\ps{\nabla F(x),m}\leq (\|\nabla F(x)\|^2+\|m\|^2)/2$, we obtain that there exists a constant
$c_1$ (depending on $\delta$) s.t. for every $t\geq 1$,
\begin{equation}
  \label{eq:w-up}
  w_\delta(t) \leq c_1\left(|F(x(t))|+\|m(t)\|^2+\|\nabla F(x(t))\|^2+ \|S(x(t))-v(t)\|^2\right)\,.
\end{equation}

\item \textit{Upper-bound on  $\frac d{dt} w_\delta(t)$.}
The function $w_{\delta}$ is absolutely continuous on $[1,+\infty)$.
    %{\color{red} Vérifier qu'elle est lischitzienne, je
    %ne l'ai pas fait. Ca ne doit pas poser de problème, le premier
    %terme est de classe C1, il faut juste regarder les deux autres et
    %le caractère borné de la trajectoire $z(t)$. Je ne sais meme pas
    %s'il est nécessaire de se restreindre à $t\geq 1$.  Par contre le
    %$t\geq 1$ est nécessaire pour l'inégalité suivante.}
Moreover, there exist $\delta>0$, $c_2>0$ (both depending on
$x_0$) s.t. for every $t\geq 1$ a.e.,
\begin{equation}
  \label{eq:dw-up}
  \frac d{dt} w_\delta (t) \leq -c_2\left(\|m(t)\|^2+\|\nabla F(x(t))\|^2+\|S(x(t))-v(t)\|^2\right)\,.
\end{equation}
The proof of Eq.~(\ref{eq:dw-up}) uses arguments that are similar to
the ones used in the proof of Prop.~\ref{prop:Wstrict} (just
use Lemma~\ref{lem:V} to bound the derivative of the first term in
Eq.~(\ref{eq:Wdelta-t})).  For this reason, it is omitted.
\item \textit{Positivity of $w_\delta(t)$.} By Lemma~\ref{lem:V},
the function $t\mapsto V(t,z(t))$ is decreasing. As it is lower bounded,
$\ell\eqdef \lim_{t\to\infty}V(t,z(t))$ exists. By Th.~\ref{th:cv-adam},
$m(t)$ tends to zero, hence this limit coincides with $\lim_{t\to\infty} F(x(t))$.
Replacing $F$ with $F-\ell$, one can assume without loss of generality that $\ell=0$.
By Eq.~(\ref{eq:dw-up}), $w_\delta$ is non-increasing on $[1,+\infty)$, hence converging
to some limit. Using again Th.~\ref{th:cv-adam}, $\ps{\nabla F(x(t)),m(t)}\to 0$
and $S(x(t))-v(t)\to 0$. Thus, $\lim_{t\to\infty} w_\delta(t) = \ell = 0$.
Assume that there exists $t_0\geq 1$ s.t. $w_\delta(t_0)=0$. Then, $w_\delta$
is constant on $[t_0,+\infty)$. By Eq.~(\ref{eq:dw-up}), this implies that $m(t)=0$ on this interval.
Hence, $d x(t)/dt = 0$. This means that $x(t) = x(t_0)$ for all $t\geq t_0$. By Th.~\ref{th:cv-adam},
it follows that $x(t_0)\in \cS$. In that case, the final result is shown. Therefore, one
can assume that $w_\delta(t)>0$ for all $t\geq 1$.

\item \textit{Putting together (\ref{eq:w-up}) and (\ref{eq:dw-up}) using the \L{}ojasiewicz condition.}
By Prop. \ref{prop:benaim} and \ref{prop:apt}, the set
$
L\eqdef \overline{\bigcup_{s\geq 0}\{z(t):t\geq s\}}\,
$
is a compact connected subset of $\cE = \{(x,0,S(x)):\nabla F(x)=0\}$.
%{\color{red} Changer $\Gamma$ en autre chose, car $\Gamma$ désigne un événement plus loin dans le papier.}
The set $\mathcal{U}\eqdef \{x:(x,0,S(x))\in L\}$ is a compact and connected subset of $\cS$.
Using Assumption~\ref{hyp:lojasiewicz_prop} and \cite[Lemma 2.1.6]{harauxjendoubi2015},
there exist constants $\sigma,c>0$ and $\theta\in (0,\frac 12]$, s.t.
%$$\forall x,\ \sd(x,\mathcal{U})<\sigma\ \Rightarrow \|\nabla F(x)\| \geq c|F(x)|^{1-\theta}\,.$$
$\|\nabla F(x)\| \geq c|F(x)|^{1-\theta}\,$ for all $x$ s.t.~$\sd(x,\mathcal{U})<\sigma\,$.
As $\sd(x(t), \mathcal{U})\to 0$, there exists $T\geq 1$ s.t. for all $t\geq T$,
$\|\nabla F(x(t))\|\geq c |F(x(t))|^{1-\theta}$. Thus, we may replace the term $\|\nabla F(x(t))\|^2$
in the righthand side of Eq.~(\ref{eq:dw-up}) using
$\|\nabla F(x(t))\|^2\geq \frac 12\|\nabla F(x(t))\|^2+ \frac 12|F(x(t))|^{2(1-\theta)}$.
Upon noting that $2(1-\theta)\geq 1$,
we thus obtain that there exists a constant $c_3$ and some $T'\geq 1$ s.t. for $t\geq T'$ a.e.,
$$
\frac d{dt} w_\delta (t) \leq -c_3\left(\|m(t)\|^2+\|\nabla F(x(t))\|^2+|F(x(t))|+\|S(x(t))-v(t)\|^2\right)^{2(1-\theta)}\,.
$$
Putting this inequality together with Eq.~(\ref{eq:w-up}), we obtain that for some constant $c_4>0$ and for all
$t\geq T'$ a.e.,
$
\frac d{dt} w_\delta (t) \leq -c_4 w_\delta(t)^{2(1-\theta)}\,.
$
% \begin{equation}
%   \label{eq:ineg-w}
% \frac d{dt} w_\delta (t) \leq -c_4 w_\delta(t)^{2(1-\theta)}\,.
% \end{equation}
\item \textit{End of the proof.} Following the arguments of \cite[Th.~10.1.6]{harauxjendoubi2015},
%by integrating Eq.~(\ref{eq:ineg-w})
by integrating the preceding inequality,
over $[T',t]$, we obtain %{\color{red} je n'ai pas vérifié cette preuve}
$
w_\delta(t) \leq c_5 t^{-\frac 1{1-2\theta}}%\qquad (t\geq T')
$
for $t\geq T'$
in the case where $\theta<\frac 12$, whereas $w_\delta(t)$ decays exponentially if $\theta=\frac 12$.
From now on, we focus on the case $\theta<\frac 12$.
By definition of~(\ref{eq:ode}), $\|\dot{x}(t)\|^2\leq \|m(t)\|^2/((1-e^{-a T'})^2\varepsilon^2)$
for all $t\geq T'$. Since Eq.~(\ref{eq:dw-up}) implies $\|m(t)\|^2\leq -\dot w_\delta(t)/c_2$,
we deduce that there exists $c,c'>0$ s.t. for all $t\geq T'$,
$%\begin{align*}
  %\int_t^{2t} \|\dot{x}(s)\|^2 ds &\leq c w_\delta(t) \leq c' t^{-\frac 1{1-2\theta}}\,.
  \int_t^{2t} \|\dot{x}(s)\|^2 ds \leq c w_\delta(t) \leq c' t^{-\frac 1{1-2\theta}}\,.
$%\end{align*}
Applying \cite[Lemma 2.1.5]{harauxjendoubi2015}, it follows that
$\int_t^{\infty} \|\dot{x}(s)\|^2 ds  \leq c  t^{-\frac{\theta}{1-2\theta}}$
for some other constant $c$.
Therefore $x^* \eqdef \lim_{t \to +\infty} x(t)$ exists by Cauchy's criterion and
for all $t\geq T'$, $\|x(t)-x^*\| \leq c  t^{-\frac{\theta}{1-2\theta}}$\,.
%\[
%\|x(t)-x^*\| \leq c  t^{-\frac{\theta}{1-2\theta}}\,.
%\]
%% If $\theta = \frac 12$, the exponential decay stems from the application of
%%\cite[Lemma 2.1.4]{harauxjendoubi2015}. %to $p(t) \eqdef \|\dot{x}(t)\|$.
Finally, since $x(t)\to a$, we remark that, using the same arguments,
the global \L{}ojasiewicz exponent $\theta$ can be replaced by any \L{}ojasiewicz exponent
of $f$ at $x^*$.
When $\theta=\frac 12$, the proof follows the same line.
\end{enumerate}

%\section{Proof of Th.~\ref{th:weak-cv}}
\section{Proofs of Section~\ref{sec:discrete}}
\label{sec:proofs_sec_discrete}
%\subsection{Preliminary Lemmas}
\subsection{Proof of Th.~\ref{th:weak-cv}}

Given an initial point $x_0\in \bR^d$ and a stepsize $\gamma>0$,
we consider the iterates $z_n^{\gamma}$ given by~(\ref{eq:znT})
and $z_0^\gamma\eqdef (x_0,0,0)$.
For every $n\in \bN^*$ and every $z\in \cZ_+$, we define
\begin{equation*}
H_\gamma(n,z,\xi)\eqdef \gamma^{-1} (T_{\gamma,\bar\alpha(\gamma),\bar\beta(\gamma)}(n,z,\xi)-z)\,.
\end{equation*}
%where $T_{\gamma,\alpha,\beta}$ is the mapping defined in (\ref{eq:T}).
Thus, $z_n^\gamma = z_{n-1}^\gamma+\gamma H_\gamma(n,z_{n-1}^\gamma,\xi_n)$ for every $n\in\bN^*$.
For every $n\in\bN^*$ and every $z\in \cZ$ of the form
$z=(x,m,v)$, we define $e_\gamma(n,z)\eqdef (x,(1-\bar\alpha(\gamma)^n)^{-1}m,(1-\bar\beta(\gamma)^n)^{-1}v)$,
and set $e_\gamma(0,z)\eqdef z$.
\begin{lemma}
  \label{lem:moment-H}
%Let \cref{hyp:model,hyp:alpha-beta,hyp:moment-f} hold true.
Let Assumptions~\ref{hyp:model}, \ref{hyp:alpha-beta} and \ref{hyp:moment-f} hold true.
There exists $\bar \gamma_0>0$ s.t. for every $R>0$, there exists $s>0$,
\begin{equation}
\sup\left\{\bE\left(\left\|H_\gamma(n+1,z,\xi)\right\|^{1+s}\right):\gamma\in (0,\bar \gamma_0], n\in \bN, z\in \cZ_+\,\text{s.t.}\,\|e_\gamma(n,z)\|\leq R\right\}<\infty\,.
\label{eq:UI}
\end{equation}
\end{lemma}
\begin{proof}
%By Assumption~\ref{hyp:alpha-beta},
%the functions $\gamma\mapsto (1-\bar \alpha(\gamma))/\gamma$ and $\gamma\mapsto (1-\bar \beta(\gamma))/\gamma$
%converge as $\gamma\downarrow 0$.
 %Thus, there exist $\bar \gamma_0>0$ and a constant $A>0$ s.t. both functions
 %are upper bounded by $A$ on $(0,\bar \gamma_0]$.
%Let $R>0$. By Assumption~\ref{hyp:moment-f}, there exists $s>0$ and a finite constant $C>0$ s.t.
Let $R>0$.
%By Assumption~\ref{hyp:moment-f}, there exist $s>0$ and $C>0$ s.t.
%$\bE(\|\nabla f(x,\xi)\|^{2+2s})\leq C$ for every $x$ s.t. $\|x\|\leq R$.
%We denote the block components of $H_\gamma$ by $(H_{\gamma,\sx},H_{\gamma,\sm},H_{\gamma,\sv})\eqdef H_{\gamma}$.
We denote by $(H_{\gamma,\sx},H_{\gamma,\sm},H_{\gamma,\sv})$ the block components of $H_\gamma$.
There exists a constant $C_s$ depending only on $s$ s.t.
$\|H_\gamma\|^{1+s}\leq C_s(\|H_{\gamma,\sx}\|^{1+s}+\|H_{\gamma,\sm}\|^{1+s}+\|H_{\gamma,\sv}\|^{1+s})$.
%As a consequence, it is sufficient to prove that Eq.~(\ref{eq:UI}) holds
Hence, it is sufficient to prove that Eq.~(\ref{eq:UI}) holds
%respectively when replacing $H_\gamma$ with each of its three components $H_{\gamma,\sx},H_{\gamma,\sm},H_{\gamma,\sv}$.
respectively when replacing $H_\gamma$ with each of $H_{\gamma,\sx},H_{\gamma,\sm},H_{\gamma,\sv}$.
%In the sequel, we write $\alpha\eqdef \bar \alpha(\gamma)$ and $\beta=\bar \beta(\gamma)$.
Consider $z=(x,m,v)$ in $\cZ_+$. We write:
$\|H_{\gamma,\sx}(n+1,z,\xi)\| \leq \varepsilon^{-1}(\|\frac{m}{1-\bar \alpha(\gamma)^n}\|+\|\nabla f(x,\xi)\|)\,.$
Thus, for every $z$ s.t. $\|e_\gamma(n,z)\|\leq R$, there exists a constant $C$ depending only on $\varepsilon$, $R$ and $s$ s.t.
$\|H_{\gamma,\sx}(n+1,z,\xi)\|^{1+s} \leq C(1+ \|\nabla f(x,\xi)\|^{1+s})$. By Assumption~\ref{hyp:moment-f}, (\ref{eq:UI}) holds for
$H_{\gamma,\sx}$ instead of $H_\gamma$. Similar arguments hold for $H_{\gamma,\sm}$ and $H_{\gamma,\sv}$
upon noting that the functions $\gamma\mapsto(1-\bar \alpha(\gamma))/\gamma$ and $\gamma\mapsto (1-\bar \beta(\gamma))/\gamma$
are bounded under Assumption~\ref{hyp:alpha-beta}.
%the functions $\gamma\mapsto (1-\bar \alpha(\gamma))/\gamma$ and $\gamma\mapsto (1-\bar \beta(\gamma))/\gamma$
%converge as $\gamma\downarrow 0$.
%$$
%\sup\left\{\bE\left(\left\|H_{\gamma,\sx}(n+1,z,\xi)\right\|^{1+r}\right):\gamma\in (0,\bar \gamma_0],n\in \bN, z\in \cZ_+\text{ s.t. }\|e_\gamma(n,z)\|\leq R\right\}<\infty\,.
%$$
% Consider $H_{\gamma,\sm}$. For every $\gamma<\bar \gamma_0$, it holds that:
% $\left\|H_{\gamma,\sm}(n+1,z,\xi)\right\| = \frac{1-\alpha}{\gamma}\left\|\nabla f(x,\xi)-m \right\|\,.$\\
% For every $z$ s.t. $\|e_\gamma(n,z)\|\leq R$,$\|H_{\gamma,\sm}(n+1,z,\xi)\|\leq A(\|\nabla f(x,\xi)\|+R)$.
% Just as above, we deduce that $\bE(\|H_{\gamma,\sx}(n+1,z,\xi)\|^{1+s})$ is uniformly bounded
% on the set $\{(\gamma,n,z):\gamma\in (0,\bar \gamma_0],\|e_\gamma(n,z)\|\leq R\}$.
% Finally, $H_{\gamma,\sv}$ satisfies the same kind of inequality for  every $z$ s.t. $\|e_\gamma(n,z)\|\leq R$,
% $
% \bE(\|H_{\gamma,\sv}(n+1,z,\xi)\|^{1+s})\leq C'(1+\bE(\|\nabla f(x,\xi)\|^{2(1+s)}))\,,
% $
% which is again bounded uniformly in $(\gamma,n,z)$ s.t. $\gamma\in (0,\bar \gamma_0]$ and $\|e_\gamma(n,z)\|\leq R$ by Assumption~\ref{hyp:moment-f}.
\end{proof}

For every $R>0$, and every arbitrary sequence $z=(z_n:n\in\bN)$ on $\cZ_+$, we define
$\tau_R(z) \eqdef \inf\{n\in \bN:\|e_\gamma(n,z_n)\|>R\}$ with the convention that $\tau_R(z)=+\infty$ when the set is empty.
We define the map $B_R:\cZ_+^\bN\to\cZ_+^\bN$ given for any arbitrary sequence  $z=(z_n:n\in\bN)$ on $\cZ_+$ by
$B_R(z)(n) = z_n\1_{n<\tau_R(z)}+z_{\tau_R(z)}\1_{n\geq\tau_R(z)}$.
We define the random sequence $z^{\gamma,R}\eqdef B_R(z^\gamma)$. Recall that a family $(X_i:i\in I)$ of random variables on some Euclidean space
is called \emph{uniformly integrable} if $\lim_{A\to+\infty} \sup_{i\in I}\bE(\|X_i\|\1_{\|X_i\|>A})=0$.
\begin{lemma}
  \label{lem:UI}
%Let \cref{hyp:model,hyp:alpha-beta,hyp:moment-f,hyp:iid} hold true.
Let Assumptions~\ref{hyp:model}, \ref{hyp:alpha-beta}, \ref{hyp:moment-f} and \ref{hyp:iid} hold true.
There exists $\bar \gamma_0>0$ s.t. for every $R>0$, the family of r.v.
$(\gamma^{-1}(z_{n+1}^{\gamma,R} - z_n^{\gamma,R}):n\in \bN,\gamma\in (0,\bar \gamma_0])$ is uniformly integrable.
\end{lemma}
\begin{proof}
%We define  $\hat z_n^\gamma\eqdef e_\gamma(n,z_n^\gamma)$ for every $n\in \bN$.
Let $R>0$.
As the event $\{n<\tau_R(z^\gamma)\}$ coincides with
$\bigcap_{k=0}^n\{\|e_\gamma(k,z_k^\gamma)\|\leq R\}$, it holds that
for every $n\in \bN$,
\begin{equation*}
  \frac{z_{n+1}^{\gamma,R} - z_n^{\gamma,R}}\gamma =\frac{z_{n+1}^{\gamma} - z_n^{\gamma}}\gamma \1_{n<\tau_R(z^\gamma)}
 = H_\gamma(n+1,z_n^\gamma,\xi_{n+1}) \prod_{k=0}^n\1_{\|e_\gamma(k, z_k^\gamma)\|\leq R} \,.
\end{equation*}
Choose $\bar \gamma_0>0$ and $s>0$ as in Lemma~\ref{lem:moment-H}.
For every $\gamma\leq \bar \gamma_0$,
\begin{equation*}
  \resizebox{\hsize}{!}{$
  \bE\left(\left\|\gamma^{-1}(z_{n+1}^{\gamma,R} - z_n^{\gamma,R})\right\|^{1+s}\right)
 \leq \sup\left\{ \bE\left(\left\|H_{\gamma'}(\ell+1,z,\xi)\right\|^{1+s} \right):\gamma'\in (0,\bar \gamma_0],\ell\in \bN,z\in \cZ_+, \|e_\gamma(\ell,z)\|\leq R\right\}\,.
 $}
\end{equation*}
%\begin{align*}
%  \bE&\left(\left\|\gamma^{-1}(z_{n+1}^{\gamma,R} - z_n^{\gamma,R})\right\|^{1+s}\right)
%\leq \bE\left(\left\|H_\gamma(n+1,z_n^\gamma,\xi_{n+1})\right\|^{1+s}\1_{\|e_\gamma(n,z_n^\gamma)\|\leq R} \right) \\
%& \leq \sup\left\{ \bE\left(\left\|H_{\gamma'}(\ell+1,z,\xi)\right\|^{1+s} \right):\gamma'\in (0,\bar \gamma_0],\ell\in \bN,z\in \cZ_+, \|e_\gamma(\ell,z)\|\leq R\right\}\,.
%\end{align*}
By Lemma~\ref{lem:moment-H}, the righthand side is finite and does not depend on $(n,\gamma)$.
\end{proof}
% We denote by $z_n^\gamma = (x_n^\gamma, m_n^\gamma, v_n^\gamma)$ the components of $z_n^\gamma$ in the product space $\cZ_+$.
% We denote by $\hat z_n^\gamma = (x_n^\gamma, \hat m_n^\gamma, \hat v_n^\gamma)$ the components of $\hat z_n^\gamma$.
%We endow the space $C([0,+\infty),\cZ)$ of continuous functions on $[0,+\infty)\to \cZ$
%with the topology of uniform convergence on compact sets.
For a fixed $\gamma>0$, we define the interpolation map $\sX_\gamma:\cZ^\bN\to C([0,+\infty),\cZ)$ as follows for every
sequence $z=(z_n:n\in \bN)$ on $\cZ$:
$$
  \sX_\gamma(z)\,:t \mapsto z_{\lfloor \frac t\gamma\rfloor} + (t/\gamma-\lfloor t/\gamma\rfloor)(z_{\lfloor \frac t\gamma\rfloor+1}-z_{\lfloor \frac t\gamma\rfloor})\,.
$$
For every $\gamma,R>0$, we define $\sz^{\gamma,R}\eqdef \sX_\gamma(z^{\gamma,R}) = \sX_\gamma\circ B_R (z^\gamma)$.
Namely, $\sz^{\gamma,R}$ is the interpolated process associated with the sequence $(z^{\gamma,R}_n)$.
It is a random variable on $C([0,+\infty),\cZ)$.
We recall that $\cF_n$ is the $\sigma$-algebra generated by the r.v. $(\xi_k:1\leq k\leq n)$.
For every $\gamma,n,R$, we use the notation: $\Delta_0^{\gamma,R}\eqdef 0$ and
$$
\Delta_{n+1}^{\gamma,R} \eqdef \gamma^{-1}(z_{n+1}^{\gamma,R} - z_n^{\gamma,R}) - \bE(\gamma^{-1}(z_{n+1}^{\gamma,R} - z_n^{\gamma,R})|\cF_n)\,.
$$
%and $\Delta_0^{\gamma,R}\eqdef 0$.
\begin{lemma}
  \label{lem:tightness-in-C}
%Let \cref{hyp:model,hyp:alpha-beta,hyp:moment-f,hyp:iid} hold true.
Let Assumptions~\ref{hyp:model}, \ref{hyp:alpha-beta}, \ref{hyp:moment-f} and \ref{hyp:iid} hold true.
There exists $\bar \gamma_0>0$ s.t. for every $R>0$, the family of r.v. $(\sz^{\gamma,R}:\gamma\in (0,\bar \gamma_0])$
is tight. Moreover, for every $\delta>0$,
$  %\begin{equation}
\bP\left(\max_{0\leq n\leq \lfloor\frac T\gamma\rfloor}\gamma\left\|\sum_{k=0}^n\Delta_{k+1}^{\gamma,R}\right\|>\delta\right)\xrightarrow[]{\gamma\to 0} 0\,.
%\label{eq:asym-rate-change}
$%\end{equation}
\end{lemma}
\begin{proof}
  It is an immediate consequence of Lemma~\ref{lem:UI} and \cite[Lemma 6.2]{bianchi2019constant}
\end{proof}
\noindent The proof of the following lemma is omitted. %and can be found in \cite[Lemma 7.4]{barakat-bianchi2018convergence}.
\begin{lemma}
  \label{lem:cv-h}
%Let \cref{hyp:model,hyp:alpha-beta} hold true.
Let Assumptions~\ref{hyp:model} and \ref{hyp:alpha-beta} hold true.
Consider $t>0$ and $z\in \cZ_+$. Let $(\varphi_n,z_n)$ be a sequence on $\bN^*\times \cZ_+$ s.t.
$\lim_{n\to\infty}\gamma_n\varphi_n= t$ and $\lim_{n\to\infty}z_n= z$. Then, $\lim_{n\to\infty} h_{\gamma_n}(\varphi_n,z_n)= h(t,z)$
and $\lim_{n\to\infty} e_{\gamma_n}(\varphi_n,z_n)= \bar e(t,z)$.
\end{lemma}

\noindent\textbf{End of the Proof of Th.~\ref{th:weak-cv}}
%\label{endofproof-weak-cv}
Consider $x_0\in \bR^d$ and set $z_0=(x_0,0,0)$. Define
$R_0\eqdef \sup\left\{\|\bar e(t,Z_\infty^0(z_0)(t))\|:t>0\right\}\,.$
By Prop.~\ref{prop:adam-bounded}, $R_0<+\infty$. We select an arbitrary $R$ s.t. $R\geq  R_0+1$.
For every $n\geq 0$, $z\in \cZ_+$,
$$
z_{n+1}^{\gamma,R} = z_{n}^{\gamma,R}+\gamma H_\gamma(n+1,z_{n}^{\gamma,R},\xi_{n+1}) \1_{\|e_\gamma(n,z_{n}^{\gamma,R})\|\leq R}  \,.
$$
%Thus,
%$
%\Delta_{n+1}^{\gamma,R} = \gamma^{-1}(z_{n+1}^{\gamma,R} - z_n^{\gamma,R}) - \bE(H_\gamma(n+1,z_{n}^{\gamma,R},\xi_{n+1}) \1_{\|e_\gamma(n,z_{n}^{\gamma,R})\|\leq R}|\cF_n)\,,
%$
Define for every $n\geq 1$, $z\in \cZ_+$,
$h_{\gamma,R}(n,z)\eqdef h_\gamma(n,z)\1_{\|e_\gamma(n-1,z)\|\leq R}$. Then,\\
$
\Delta_{n+1}^{\gamma,R} = \gamma^{-1}(z_{n+1}^{\gamma,R} - z_n^{\gamma,R}) - h_{\gamma,R}(n+1,z_{n}^{\gamma,R})\,.
$
Define also for every $n\geq 0$, \\
$%\begin{equation*}
  M_n^{\gamma,R} \eqdef \sum_{k=1}^n\Delta_k^{\gamma,R} =\gamma^{-1}(z_n^{\gamma,R} - z_0) - \sum_{k=0}^{n-1} h_{\gamma,R}(k+1,z_{k}^{\gamma,R})\,.
$ %\end{equation*}
Consider $t\geq 0$ and set $n\eqdef \lfloor t/\gamma\rfloor$. %It holds that:
% \begin{equation*}
%   \sz^{\gamma,R}(t) = z_0 +  \int_{0}^{t}h_{\gamma,R}(\lfloor s/\gamma\rfloor+1,\sz^{\gamma,R}(\gamma\lfloor s/\gamma\rfloor))ds
% +\gamma M_n^{\gamma,R}+(t-n\gamma)\Delta_{n+1}^{\gamma,R}\,.
% \end{equation*}
%%As a consequence,
%%\begin{equation*}
%%  \left\| \sz^{\gamma,R}(t)-z_0 - \int_{0}^{t}h_{\gamma,R}(\lfloor s/\gamma\rfloor+1,\sz^{\gamma,R}(\gamma\lfloor s/\gamma\rfloor))ds\right\|
%%\leq \|\gamma M_n^{\gamma,R}+(t-n\gamma)\Delta_{n+1}^{\gamma,R}\|\,.
%%\end{equation*}
%Therefore, for any $T>0$,
For any $T>0$, it holds that :
\begin{equation*}
  \sup_{t\in [0,T]} \left\| \sz^{\gamma,R}(t)-z_0 - \int_{0}^{t}h_{\gamma,R}(\lfloor s/\gamma\rfloor+1,\sz^{\gamma,R}(\gamma\lfloor s/\gamma\rfloor))ds\right\|\\
  %\leq \sqrt 2 \max_{0\leq n\leq \lfloor T/\gamma\rfloor+1} \gamma \|M_n^{\gamma,R}\|\,.
  \leq \max_{0\leq n\leq \lfloor T/\gamma\rfloor+1} \gamma \|M_n^{\gamma,R}\|\,.
\end{equation*}
By Lemma~\ref{lem:tightness-in-C},
\begin{equation}
  \label{eq:Mn-cv-proba-zero}
 \bP\left( \sup_{t\in [0,T]} \left\| \sz^{\gamma,R}(t)-z_0 - \int_{0}^{t}h_{\gamma,R}\left(\lfloor s/\gamma\rfloor+1,\sz^{\gamma,R}(\gamma\lfloor s/\gamma\rfloor)\right)ds\right\|
>\delta\right) \xrightarrow[]{\gamma\to 0} 0\,.
\end{equation}
As a second consequence of Lemma~\ref{lem:tightness-in-C}, the family of r.v. $(\sz^{\gamma,R}:0<\gamma\leq \bar \gamma_0)$ is tight, where $\bar \gamma_0$
is chosen as in Lemma~\ref{lem:tightness-in-C} (it does not depend on $R$).
By Prokhorov's theorem, there exists a sequence $(\gamma_k:k\in \bN)$ s.t. $\gamma_k\to 0$ and
s.t. $(\sz^{\gamma_k,R}:k\in \bN)$ converges in distribution to some probability measure $\nu$ on $C([0,+\infty),\cZ_+)$.
By Skorohod's representation theorem, there exists a r.v. $\sz$ on some probability space $(\Omega',\cF',\bP')$, with distribution $\nu$,
and a sequence of r.v. $(\sz_{(k)}:k\in \bN)$ on that same probability space where for each $k \in \bN$, the r.v. $\sz_{(k)}$ has the same distribution
as the r.v. $\sz^{\gamma_k,R}$,
and s.t. for every $\omega\in\Omega'$, $\sz_{(k)}(\omega)$ converges
to $\sz(\omega)$ uniformly on compact sets. Now select a fixed $T>0$. According to Eq.~(\ref{eq:Mn-cv-proba-zero}), the sequence
$$
\sup_{t\in [0,T]} \left\| \sz_{(k)}(t)-z_0 - \int_{0}^{t}h_{\gamma_{k},R}\left(\lfloor s/\gamma_{k}\rfloor+1,\sz_{(k)}(\gamma_{k}\lfloor s/\gamma_{k}\rfloor)\right)ds\right\|\,,
$$
indexed by $k\in \bN$, converges in probability to zero as $k\to\infty$. One can therefore extract a further subsequence $\sz_{(\varphi_k)}$,
s.t. the above sequence converges to zero almost surely. In particular, since $\sz_{(k)}(t)\to \sz(t)$ for every $t$, we obtain that
\begin{equation}
\sz(t) = z_0 + \lim_{k\to\infty} \int_{0}^{t}h_{\gamma_{\varphi_k},R}\left(\lfloor s/\gamma_{\varphi_k}\rfloor+1,\sz_{(\varphi_k)}(\gamma_{\varphi_k}\lfloor s/\gamma_{\varphi_k}\rfloor)\right)ds\quad (\forall t\in [0,T])\,.\label{eq:cv_z}
\end{equation}
Consider $\omega\in \Omega'$ s.t. the r.v. $\sz$ satisfies (\ref{eq:cv_z}) at point $\omega$.
From now on, we consider that $\omega$ is fixed, and we handle $\sz$ as an element of $C([0,+\infty),\cZ_+)$,
and no longer as a random variable.
Define $\tau\eqdef \inf\{t\in [0,T]: \|\bar e(t,\sz(t))\|>R_0+\frac 12\}$ if the latter set is non-empty,
and $\tau\eqdef T$ otherwise.
Since $\sz(0)=z_0$ and $\|z_0\|<R_0$, it holds that $\tau>0$ using the continuity of $\sz$.
Choose any $(s,t)$ s.t. $0<s<t<\tau$.
Note that $\sz_{(k)}(\gamma_{k}\lfloor s/\gamma_{k}\rfloor)\to \sz(s)$
and $\gamma_k(\lfloor s/\gamma_{k}\rfloor+1)\to s$. Thus, by Lemma~\ref{lem:cv-h},
$h_{\gamma_{k}}\left(\lfloor s/\gamma_{k}\rfloor+1,\sz_{(k)}(\gamma_{k}\lfloor s/\gamma_{k}\rfloor)\right)$
converges to $h(s,\sz(s))$ and %$h_{\gamma,R}(n,z)\eqdef h_\gamma(n,z)\1_{\|e_\gamma(n-1,z)\|\leq R}$.
$e_{\gamma_k}(\lfloor s/\gamma_{k}\rfloor,\sz_{(k)}(\gamma_{k}\lfloor s/\gamma_{k}\rfloor))$
converges to $\bar e(s,\sz(z))\,.$
Since $s<\tau$,  $\bar e(s,\sz(z))\leq R_0+\frac 12$. As $R\geq R_0+1$, there exists a certain
$K(s)$ s.t. for every $k\geq K(s)$,
$
\1_{\|e_{\gamma_k}(\lfloor s/\gamma_{k}\rfloor,\sz_{(k)}(\gamma_{k}\lfloor s/\gamma_{k}\rfloor))\|\leq R} = 1\,.
$
As a consequence, $h_{\gamma_{k},R}(\lfloor s/\gamma_{k}\rfloor+1,\sz_{(k)}(\gamma_{k}\lfloor s/\gamma_{k}\rfloor))$
converges to $h(s,\sz(s))$ as $k\to\infty$.
% $
% \lim_{k\to\infty} h_{\gamma_{k},R}\left(\lfloor s/\gamma_{k}\rfloor+1,\sz_{(k)}(\gamma_{k}\lfloor s/\gamma_{k}\rfloor)\right)
%  = h(s,\sz(s)).
% $
Using Lebesgue's dominated convergence theorem,  we obtain, for all $t\in [0,\tau]$:
$
\sz(t) = z_0 +  \int_{0}^{t}h\left(s,\sz(s))\right)ds\,.%(\forall t\in [0,\tau])\,.
$
Therefore $\sz(t) = Z_\infty^0(x_0)(t)$ for every $t\in [0,\tau]$. In particular,
$\|\sz(\tau)\|\leq R_0$
%Recalling the definition of $\tau$, this means that
and this means that
$\tau = T$. Thus, $\sz(t)=Z_\infty^0(x_0)(t)$ for every $t\in [0,T]$
(and consequently for every $t\geq 0$). We have shown that for every $R\geq R_0+1$, the sequence of r.v.
$(\sz^{\gamma,R}:\gamma \in (0,\bar \gamma_0])$ is tight and converges in probability to $Z_\infty^0(z_0)$ as $\gamma\to 0$.
Therefore, for every $T>0$,
\begin{equation}
\forall \delta>0,\ \lim_{\gamma\to 0} \bP\left(\sup_{t\in [0,T]}\left\| \sz^{\gamma,R}(t) - Z_\infty^0(x_0)(t)\right\|>\delta \right) =0\,.
\label{eq:cv-proba-R}
\end{equation}
In order to complete the proof, we show that %it is now sufficient to establish that:
$
\bP\left(\sup_{t\in [0,T]}\left\| \sz^{\gamma,R}(t) - \sz^{\gamma}(t)\right\|>\delta \right) \to 0
$
as $\gamma \to 0$, for all $\delta >0$.
% \begin{equation}
%   \label{eq:no-trunc}
%   \forall \delta>0,\ \lim_{\gamma\to 0} \bP\left(\sup_{t\in [0,T]}\left\| \sz^{\gamma,R}(t) - \sz^{\gamma}(t)\right\|>\delta \right) =0\,,
% \end{equation}
where we recall that $\sz^\gamma=\sX_\gamma(z^\gamma)$.
% Note that for every $T,\delta>0$,
% \begin{equation*}
%   \bP\left(\sup_{t\in [0,T]}\left\| \sz^{\gamma,R}(t) - \sz^{\gamma}(t)\right\|>\delta \right)
% \leq  \bP\left(\sup_{t\in [0,T]}\left\| \sz^{\gamma,R}(t)\right\|\geq R \right)\,.
% \end{equation*}
%By the triangular inequality, $\left\| \sz^{\gamma,R}(t)\right\|\leq
%\left\| \sz^{\gamma,R}(t)-Z_\infty^0(z_0)(t)\right\| + R_0$. Therefore,
Note that $\left\| \sz^{\gamma,R}(t)\right\|\leq
\left\| \sz^{\gamma,R}(t)-Z_\infty^0(z_0)(t)\right\| + R_0$ by the triangular inequality.
Therefore, for every $T,\delta>0$,
\begin{align*}%$$
  \bP\left(\sup_{t\in [0,T]}\left\| \sz^{\gamma,R}(t) - \sz^{\gamma}(t)\right\|>\delta \right)
  &\leq \bP\left(\sup_{t\in [0,T]}\left\| \sz^{\gamma,R}(t)\right\|\geq R \right)\\
  &\leq \bP\left(\sup_{t\in [0,T]}\left\| \sz^{\gamma,R}(t)-Z_\infty^0(z_0)(t)\right\| \geq R-R_0\right)\,.
\end{align*}%$$
By Eq.~(\ref{eq:cv-proba-R}), the RHS of the above inequality tends to zero as $\gamma\to 0$.
%This shows that Eq.~(\ref{eq:no-trunc}) holds true.
The proof is complete.

\subsection{Proof of Th.~\ref{th:longrun}}
\label{sec:longrun}

%\subsection{A General Convergence Result}\hfill\\

We start by stating a general result. Consider a Euclidean space $\sX$ equipped with its Borel $\sigma$-field $\cal X$.
Let $\bar \gamma_0>0$, and consider two families
$(P_{\gamma,n}:0<\gamma<\bar \gamma_0, n\in \bN^*)$ and $(\bar P_{\gamma}:0<\gamma<\bar \gamma_0)$
of Markov transition kernels on $\sX$.
Denote by $\cP(\sX)$ the set of probability measures on $\sX$.
Let $X=(X_n:n\in\bN)$ be the canonical process on $\sX$.
Let $(\bP^{\gamma,\nu}:0<\gamma<\bar \gamma_0,\nu\in \cP(\sX))$ and
$(\bar \bP^{\gamma,\nu}:0<\gamma<\bar \gamma_0,\nu\in \cP(\sX))$ be two families of measures on the canonical space
$(X^\bN,\cal X^{\otimes\bN})$ such that the following holds:
\begin{itemize}[leftmargin=*]
\item Under $\bP^{\gamma,\nu}$, $X$ is a non-homogeneous Markov chain with transition kernels $(P_{\gamma,n}:n\in \bN^*)$
and initial distribution $\nu$, that is, for each $n\in\bN^*$,
$
\bP^{\gamma,\nu}(X_{n}\in dx|X_{n-1}) = P_{\gamma,n}(X_{n-1},dx)\,.
$
\item Under $\bar\bP^{\gamma,\nu}$, $X$ is an homogeneous Markov chain with transition kernel $\bar P_{\gamma}$
and initial distribution $\nu$.
\end{itemize}
In the sequel, we will use the notation $\bar P^{\gamma,x}$ as a shorthand notation for
$\bar P^{\gamma,\delta_x}$ where $\delta_x$ is the Dirac measure at some point $x\in \sX$.
Finally, let $\Psi$ be a semiflow on $\sX$. A Markov kernel $P$ is \emph{Feller}
if $Pf$ is continuous for every bounded continuous  $f$.
\begin{assumption} Let $\nu\in \cP(\sX)$.
  \begin{enumerate}[{\sl i)}]
  \item For every $\gamma$, $\bar P_{\gamma}$ is Feller.
  \item $(\bP^{\gamma,\nu}X_n^{-1}:n\in \bN,0<\gamma<\bar \gamma_0)$ is a tight family of measures.
  \item For every $\gamma\in (0,\bar \gamma_0)$ and every bounded Lipschitz continuous function $f:\sX\to\bR$,
$P_{\gamma,n}f$ converges to $\bar P_\gamma f$ as $n\to\infty$, uniformly on compact sets.
\item For every $\delta>0$, for every compact set $K\subset \sX$, for every $t>0$,\\
$%$$
\lim_{\gamma\to 0}\sup_{x\in K}\bar P^{\gamma,x}\left(\|X_{\lfloor t/\gamma\rfloor} - \Psi_t(x)\|>\delta\right)=0\,.
$%$$
%for every $\delta>0$, for every compact set $K\subset \sX$, for every $t>0$.
  \end{enumerate}
\label{hyp:general}
\end{assumption}
Let $BC_\Psi$ be the Birkhof center of $\Psi$ \emph{i.e.}, the closure of the set of recurrent points.

\begin{theorem}
\label{longrun}
Consider $\nu\in \cP(\sX)$ s.t. Assumption~\ref{hyp:general} holds true. Then, for every $\delta>0$,
$
\lim_{\gamma\to 0}\limsup_{n\to\infty} \frac 1{n+1}\sum_{k=0}^n \bP^{\gamma,\nu}\left(d(X_k,BC_\Psi)>\delta\right)=0\,.
$
\end{theorem}
We omit the proof of this result which follows a similar reasoning to \cite[Th.~5.5 and Proof in section 8.4]{bianchi2019constant} and makes use of results from \cite{for-pag-99}. %which relies on

\noindent\textbf{End of the Proof of Th.~\ref{th:longrun}.}
We apply Th.~\ref{longrun} in the case where $P_{\gamma,n}$ is the kernel
of the non-homogeneous Markov chain $(z_n^\gamma)$ defined by~(\ref{eq:znT}) and
$\bar P_\gamma$ is the kernel of the homogeneous Markov chain $(\bar z_n^\gamma)$
given by %defined by the biased version of \adam\ defined by:
$\bar z_n^\gamma = \bar z_{n-1}^\gamma+\gamma H_\gamma(\infty,\bar z_{n-1}^\gamma,\xi_n)$
for every $n\in\bN^*$ and $\bar z_0 \in \cZ_+$ where $H_\gamma(\infty,\bar z_{n-1}^\gamma,\xi_n) \eqdef \lim_{k \to \infty} H_\gamma(k,\bar z_{n-1}^\gamma,\xi_n)$. The task is merely to verify Assumption~\ref{hyp:general}{\sl iii)}, the other assumptions being easily verifiable using Th.~\ref{th:weak-cv}, % stated for the biased version of \adam\,, Lemma~\ref{lem:UI}, Lemma~\ref{lem:tightness-in-C} and \cite[Lemma 6.2]{bianchi2019constant}.
Consider $\gamma\in (0,\bar \gamma_0)$. Let $f: \cZ \to \mathbb{R}$ be a bounded $M$-Lipschitz continuous function and $K$ a compact.
%Assumption~\ref{hyp:general}{\sl iii)} is verified by showing that
%$\sup_{z \in K}|P_{\gamma,n}(f)(z) - \bar P_{\gamma}(f)(z)| \to 0$ as $n \to +\infty$.
For all $z=(x,m,v) \in K$:
\begin{align*}
&|P_{\gamma,n}(f)(z) - \bar P_{\gamma}(f)(z)| \leq M \gamma \bE \left \| \frac{(1-\alpha^n)^{-1}\tilde{m}_\xi}{ \varepsilon+(1-\beta^n)^{-\frac{1}{2}}{\tilde{v}_\xi^{1/2}}} -\frac{\tilde{m}_\xi}{ \varepsilon+{\tilde{v}_\xi^{1/2}}}\right \| \\
&\resizebox{.99\hsize}{!}{$\leq \frac{M \gamma\alpha^n}{\varepsilon(1- \alpha^n)} \sup_{x,m}\left(\alpha ||m|| + (1-\alpha)\bE||\nabla f(x,\xi)|| \right)
 + \frac{M \gamma \bE ||\tilde{m}_\xi \odot \tilde{v}_\xi^{1/2}||}{\varepsilon^2}\left(1- \frac{1}{(1-\beta^n)^{1/2}}\right)$}\,
\end{align*}
where we write $\alpha=\bar \alpha(\gamma)$, $\beta=\bar \beta(\gamma)$,
$\tilde{m}_\xi \eqdef \alpha m+(1-\alpha)\nabla f(x,\xi)$ and
$\tilde{v}_\xi \eqdef  \beta v+(1- \beta)\nabla f(x,\xi)^{\odot 2}$.
Thus, condition~\ref{hyp:general}{\sl iii)} follows.
Finally, Cor.~\ref{coro:cv} implies $BC_\Phi=\cE$.
%Finally, the fact that $BC_\Phi=\cE$ follows from Cor.~\ref{coro:cv}.

\section{Proofs of Section~\ref{sec:discrete_decreasing}}
\label{sec:proofs_sec_discrete_decreasing}
In this section, we denote by $\bE_n = \bE(\cdot|\cF_n)$ the conditional expectation w.r.t. $\cF_n$.
We also use the notation $\nabla f_{n+1} \eqdef \nabla f(x_n,\xi_{n+1})$. %to simplify the equation.

The following lemma will be useful in the proofs.
\begin{lemma}
\label{lemma:r_n}
Let the sequence $(r_n)$ be defined as in Algorithm~\ref{alg:adam-decreasing}.
Assume that $0 \leq \alpha_n \leq 1$ for all $n$ and that $(1-\alpha_n)/\gamma_n \to a > 0$ as $n \to +\infty$.
Then,
\begin{enumerate}[{\sl i)}]
\item $\forall n \in \bN, r_n = 1 - \prod_{i=1}^n \alpha_i$,
\item The sequence $(r_n)$ is nondecreasing and converges to $1$.
\item Under Assumption~\ref{hyp:step-tcl}~\ref{step-tcl-i}, for every $\epsilon>0$, for sufficiently large $n$, we have
$r_n-1 \leq e^{-\frac{a\gamma_0}{2(1-\kappa)}n^{1-\kappa}}$ if $\kappa \in (0,1)$ and
$r_n-1 \leq n^{-a \gamma_0/(1+\epsilon)}$ if $\kappa = 1$.
% $
%      r_n-1 \leq
%      \begin{cases}
%         e^{-\frac{a\gamma_0}{2(1-\kappa)}n^{1-\kappa}} \quad \text{if}\,\,\kappa \in (0,1)\\
%         n^{-a \gamma_0/(1+\epsilon)} \quad \text{if}\,\,\kappa=1 \,.
%      \end{cases}
% $
%\item If $\gamma_n = \gamma_0 n^{-\alpha}$ with $\alpha \in (0,1)$, then
%for sufficiently large $n$, $1-r_n \leq e^{-\frac{a\gamma_0}{2(1-\alpha)}n^{1-\alpha}}$\,.
%\item If $\gamma_n = \gamma_0 n^{-1}$, then for sufficiently large $n$,
%$1-r_n \leq n^{-a \gamma_0/2}$\,.
\end{enumerate}
\end{lemma}
A similar lemma holds for the sequence $(\bar{r}_n)$.
\begin{proof}
i) stems from observing that $r_{n+1} -1 = \alpha_{n+1}(r_n -1)$ for every $n \in \bN$ and
iterating this relation ($r_0=0$). As a consequence, the sequence $(r_n)$ is nondecreasing.
%We now show iii) and iv).
We can write :
$
0 \leq 1-r_n \leq \exp(- \sum_{i=1}^n (1-\alpha_i))\,.
$
iii) As $\sum_{n\geq 1} \gamma_n = + \infty$ %for $\gamma_n = \gamma_0 n^{-\kappa}$ with $\kappa \in [0,1]$
and $(1-\alpha_n) \sim a\gamma_n$, we deduce that $\sum_{i=1}^n (1-\alpha_i) \sim \sum_{i=1}^n a\gamma_i$.
The results follow from the fact that $\sum_{i=1}^n \gamma_i \sim \frac{\gamma_0}{1-\kappa} n^{1-\kappa}$
when $\kappa \in (0,1)$ and $\sum_{i=1}^n \gamma_i \sim \gamma_0 \ln n$ for $\kappa=1$.
\end{proof}

\subsection{Proof of Th.~\ref{thm:as_conv_under_stab}}

We define
$\bar z_n = (x_{n-1},m_n,v_n)$ (note the shift in the index of the variable $x$).
We have
$$
\bar z_{n+1} = \bar z_n + \gamma_{n+1} h_\infty(\bar z_n) + \gamma_{n+1}\chi_{n+1} + \gamma_{n+1} \varsigma_{n+1}\,,
$$
where $h_\infty$ is defined in Eq.~(\ref{eq:h_infty}) and where we set
%$\chi_{n+1} = (\chi_{n+1}^x,\chi_{n+1}^m,\chi_{n+1}^v)$
$$
%\chi_{n+1} = \left(0,\frac{1-\alpha_{n+1}}{\gamma_{n+1}}\left(\nabla f_{n+1}-\nabla F(x_n)\right),\frac{1-\beta_{n+1}}{\gamma_{n+1}}\left(\nabla f_{n+1}^{\odot 2}-S(x_n)\right)\right)
\chi_{n+1} = \left(0,\gamma_{n+1}^{-1}(1-\alpha_{n+1})(\nabla f_{n+1}-\nabla F(x_n)),\gamma_{n+1}^{-1}(1-\beta_{n+1})(\nabla f_{n+1}^{\odot 2}-S(x_n))\right)
$$
and $\varsigma_{n+1} = (\varsigma_{n+1}^x,\varsigma_{n+1}^m,\varsigma_{n+1}^v)$ with the components defined by:
% $$
% \begin{cases}
%   \chi_{n+1}^x &= 0 \\
%   \chi_{n+1}^m &= \frac{1-\alpha_{n+1}}{\gamma_{n+1}}\left(\nabla f_{n+1}-\nabla F(x_n)\right) \\
%   \chi_{n+1}^v &= \frac{1-\beta_{n+1}}{\gamma_{n+1}}\left(\nabla f_{n+1}^{\odot 2}-S(x_n)\right)
% \end{cases}
% $$
%and
% $$
% \begin{cases}
%   \varsigma_{n+1}^x &= \frac{m_n}{\varepsilon + \sqrt{v_n}} - \frac{\gamma_n}{\gamma_{n+1}}\frac{\hat m_n}{\varepsilon + \sqrt{\hat v_n}} \\
%   \varsigma_{n+1}^m &= \left(\frac{1-\alpha_{n+1}}{\gamma_{n+1}}-a\right)(\nabla F(x_n)-m_n) + a(\nabla F(x_n) - \nabla F(x_{n-1}))  \\
%   \varsigma_{n+1}^v &= \left(\frac{1-\beta_{n+1}}{\gamma_{n+1}}-b\right)(S(x_n)-v_n) + b(S(x_n) - S(x_{n-1})) \,.
% \end{cases}
% $$
$\varsigma_{n+1}^x = \frac{m_n}{\varepsilon + \sqrt{v_n}} - \frac{\gamma_n}{\gamma_{n+1}}\frac{\hat m_n}{\varepsilon + \sqrt{\hat v_n}}$,
$\varsigma_{n+1}^m = \left(\frac{1-\alpha_{n+1}}{\gamma_{n+1}}-a\right)(\nabla F(x_n)-m_n) + a(\nabla F(x_n) - \nabla F(x_{n-1}))$ and
$\varsigma_{n+1}^v = \left(\frac{1-\beta_{n+1}}{\gamma_{n+1}}-b\right)(S(x_n)-v_n) + b(S(x_n) - S(x_{n-1}))$ \,.
We prove that $\varsigma_n\to 0$ a.s. %by considering the components separately.
Using the triangular inequality, %followed by the Cauchy-Schwarz inequality,
\begin{align*}
  \|\varsigma_n^x\| %&= \left\| \frac{m_n}{\varepsilon + \sqrt{v_n}} - \frac{\gamma_nr_n^{-1}}{\gamma_{n+1}\bar r_n^{-1/2}}\frac{m_n}{\bar r_n^{1/2}\varepsilon + \sqrt{v_n}}\right\| \\
&\leq \left\| \frac{m_n}{\varepsilon + \sqrt{v_n}} -\frac{m_n}{\bar r_n^{1/2}\varepsilon + \sqrt{v_n}}\right\|
+ \left|1-\frac{\gamma_nr_n^{-1}}{\gamma_{n+1}\bar r_n^{-1/2}}\right|\left\| \frac{m_n}{\bar r_n^{1/2}\varepsilon + \sqrt{v_n}}\right\|\\
%&\leq \sqrt d\varepsilon^{-1}|1-\bar r_n^{-1/2}|\|m_n\| +  \sqrt d\varepsilon^{-1}\left|\bar r_n^{-1/2}-\frac{\gamma_nr_n^{-1}}{\gamma_{n+1}}\right|\|m_n\|\,,
&\leq \varepsilon^{-1}|1-\bar r_n^{-1/2}|\|m_n\| + \varepsilon^{-1}\left|\bar r_n^{-1/2}-\frac{\gamma_nr_n^{-1}}{\gamma_{n+1}}\right|\|m_n\|\,,
\end{align*}
which converges a.s. to zero because of the boundedness of $(z_n)$ combined with Assumption~\ref{hyp:stepsizes} and Lemma~\ref{lemma:r_n} for $(\bar r_n)$.
The components $\varsigma_{n+1}^m$ and $\varsigma_{n+1}^v$ converge a.s. to zero,
as products of a bounded term and a term converging to zero.
Indeed, note that $\nabla F$ and $S$ are locally Lipschitz continuous under Assumption~\ref{hyp:model}.
Hence, there exists a constant $C$ s.t. $\|\nabla F(x_n) - \nabla F(x_{n-1})\| \leq C \|x_n-x_{n-1}\| \leq \frac{C}{\varepsilon}\gamma_n\|m_n\|$.
The same inequality holds when replacing $\nabla F$ by $S$.
Now consider the martingale increment sequence $(\chi_n)$, adapted to $\cF_n$.
Estimating the second order moments, it is easy to show using Assumption~\ref{hyp:moment-f}~\ref{momentegal} that
there exists a constant $C'$ s.t.
$\bE_n(\|\chi_{n+1}\|^2)\leq C'$. %, where $p(\cdot)$ is bounded on bounded sets.
Using that $\sum_k\gamma_k^2<\infty$, it follows that $\sum_n\bE_n(\|\gamma_{n+1}\chi_{n+1}\|^2)<\infty$ a.s.
By Doob's convergence theorem, $\lim_{n\to\infty} \sum_{k\leq n} \gamma_k\chi_k$ exists almost surely.
Using this result along with the fact that $\varsigma_n$ converges a.s. to zero, it follows from
usual stochastic approximation arguments \cite{ben-(cours)99} that the interpolated process
$\bar\sz: [0,+\infty)\to \cZ_+$ given by
\[
%\bar \sz(t) = \bar z_n + \left(t-\sum_{k=0}^n\gamma_k\right) \frac{\bar z_{n+1}-\bar z_n}{\gamma_{n+1}} \qquad \left(\forall n \in \bN\,,\, \forall t \in %\left[\sum_{k=0}^n\gamma_k,\sum_{k=0}^{n+1}\gamma_k\right)\right)
\bar \sz(t) = \bar z_n + (t-\tau_n) \frac{\bar z_{n+1}-\bar z_n}{\gamma_{n+1}} \qquad \left(\forall n \in \bN\,,\, \forall t \in [\tau_n,\tau_{n+1})\right)
\]
(where $\tau_n = \sum_{k=0}^n\gamma_k$), is almost surely a bounded APT of the semiflow $\bar\Phi$ defined by~(\ref{eq:ode-a}).
The proof is concluded by applying Prop.~\ref{prop:benaim} and Prop.~\ref{prop:Wstrict}.

%\subsection{Stability of the algorithm}
\subsection{Proof of Prop.~\ref{thm:stab}}
\label{sec:stability}

As $\inf F>-\infty$, one can assume without loss of generality that $F\geq 0$.
In the sequel, %we use the notation $\nabla f_{n+1}$ as a shorthand notation for $\nabla f(x_n,\xi_{n+1})$ and
$C$ denotes some positive constant which may change from line to line.
We define $a_n \eqdef (1-\alpha_{n+1})/\gamma_n$ and
$P_n\eqdef \frac 1{2a_nr_n}\ps{m_n^{\odot 2},\frac 1{\varepsilon+\sqrt{\hat v_n}}}$.
We have $a_n\to a$ and $r_n\to 1$.
%Denoting by $M$ the Lipschitz coefficient of $\nabla F$,
By Assumption~\ref{hyp:stab}-\ref{lipschitz},
\begin{align}
F(x_{n})  %&\leq F(x_{n-1}) - \gamma_{n} \ps{\nabla F(x_{n-1}),\frac{\hat m_n}{\varepsilon + \sqrt{\hat v_n}}} + M\gamma_n^2\left\|\frac{\hat m_n}{\varepsilon + \sqrt{\hat v_n}}\right\|^2 \nonumber \\
&\leq F(x_{n-1}) - \gamma_{n} \ps{\nabla F(x_{n}),\frac{\hat m_n}{\varepsilon + \sqrt{\hat v_n}}} + C\gamma_n^2 P_n\label{eq:lip}.
\end{align}
We set $u_n \eqdef 1-\frac{a_{n+1}}{a_{n}}$ and
$D_n \eqdef \frac {r_n^{-1}}{\varepsilon+\sqrt{\hat v_{n}}}$, so that $P_n = \frac 1{2a_n}\ps{D_n,m_n^{\odot 2}}$.
%We have the decomposition:
We can write:
\begin{equation}
P_{n+1}-P_n= u_nP_{n+1} +\ps{\frac{D_{n+1}-D_n}{2a_n},m_{n+1}^{\odot 2}}+\ps{\frac{D_n}{2a_n},m_{n+1}^{\odot 2}-m_n^{\odot 2}}.\label{eq:P-P}
\end{equation}
We estimate the vector $D_{n+1}-D_n$.
Using that $(r_n^{-1})$ is non-increasing,
$$
D_{n+1}-D_n \leq r_{n}^{-1} \frac{\sqrt{\hat v_n} - \sqrt{\hat v_{n+1}}}{(\varepsilon + \sqrt{\hat v_{n+1}})\odot(\varepsilon + \sqrt{\hat v_n})}\,.
$$
Remarking that $v_{n+1} \geq \beta_{n+1} v_n$, recalling that
$(\bar{r}_n)$ is nondecreasing and using the update rules of $v_n$ and
$\bar{r}_n$, we obtain after some algebra
\begin{align}
  \label{eq:subsubterm2}
\sqrt{\hat v_n} - \sqrt{\hat v_{n+1}} &=  \bar{r}_{n+1}^{-\frac 12}(1-\beta_{n+1})\frac{v_n-\nabla f_{n+1}^{\odot 2}}{\sqrt{v_n}+\sqrt{v_{n+1}}}%\\
                                          + \frac{\bar{r}_{n+1}-\bar{r}_n}{\sqrt{\bar{r}_n}(\sqrt{\bar{r}_n}+\sqrt{\bar{r}_{n+1}})} \sqrt{\frac{v_n}{\bar{r}_{n+1}}} \nonumber\\
                                          % &\leq  \left( \frac{1-\beta_{n+1}}{1+\sqrt{\beta_{n+1}}} + \frac{\bar{r}_{n+1}-\bar{r}_n}{2\bar{r}_n}  \right) \sqrt{\frac{v_n}{\bar{r}_{n+1}}} \nonumber\\
                                          &\leq c_{n+1} \sqrt{\hat v_{n+1}}  \,\, \text{where} \,\,
                                          c_{n+1} \eqdef \frac{1-\beta_{n+1}}{\sqrt{\beta_{n+1} }}\left( \frac{1}{1+\sqrt{\beta_{n+1}}} + \frac{1- \bar{r}_n}{2\bar{r}_n} \right)\,.
\end{align}
It is easy to see that $c_n/\gamma_n\to b/2$. Thus, for any $\delta>0$, $c_{n+1}\leq (b+2\delta)\gamma_{n}/2$ for all $n$ large enough.
Using also that $\sqrt{\hat v_{n+1}}/(\varepsilon +\sqrt{\hat v_{n+1}})\leq 1$, we obtain that
$%$$
D_{n+1}-D_n \leq %\frac{b+2\delta}{2(\varepsilon + \sqrt{\hat v_n})} =
\frac{b+2\delta}2 \gamma_n D_n\,.
$ %$$
%Substituting the above inequality in Eq.~(\ref{eq:P-P}), we obtain
Substituting this inequality in Eq.~(\ref{eq:P-P}), we get
\begin{align*}
  P_{n+1}-P_n&\leq u_nP_{n+1} +\gamma_n\ps{\frac{b+2\delta}{4a_n} D_n,m_{n+1}^{\odot 2}}+\ps{\frac{D_n}{2a_n},m_{n+1}^{\odot 2}-m_n^{\odot 2}}\,. %\\
%&=u_nP_{n+1} +\frac{b+2\delta}{2}\gamma_n P_n+\left(1+\gamma_n\frac{b+2\delta}2\right)\ps{\frac {D_n}{2a_n},m_{n+1}^{\odot 2}-m_n^{\odot 2}} \,.
\end{align*}
Using $m_{n+1}^{\odot 2} - m_n^{\odot 2} =  2 m_n
\odot (m_{n+1}-m_n)+(m_{n+1}-m_n)^{\odot 2} $, and noting that\\ $\bE_n(m_{n+1}-m_n) = a_n\gamma_n(\nabla F(x_n) -m_n)$,
\begin{equation*}
  \bE_n\ps{\frac {D_n}{2a_n},m_{n+1}^{\odot 2}-m_n^{\odot 2}}
 = \gamma_n\ps{\nabla F(x_n),\frac{\hat m_n}{\varepsilon+\sqrt{\hat v_n}}}-2a_n\gamma_nP_n+\ps{\frac {D_n}{2a_n}, \bE_n[(m_{n+1}-m_n)^{\odot 2}] }
\end{equation*}
%\begin{align*}
%  \bE_n\ps{\frac {D_n}{2a_n},m_{n+1}^{\odot 2}-&m_n^{\odot 2}} =
%  \gamma_n\ps{D_n, m_n\odot (\nabla F(x_n) -m_n) }+\ps{\frac {D_n}{2a_n}, \bE_n[(m_{n+1}-m_n)^{\odot 2}] } \\
%&= \gamma_n\ps{\nabla F(x_n),\frac{\hat m_n}{\varepsilon+\sqrt{\hat v_n}}}-2a_n\gamma_nP_n+\ps{\frac {D_n}{2a_n}, \bE_n[(m_{n+1}-m_n)^{\odot 2}] }
%\end{align*}
%where we used $\ps{D_n, m_n\odot \nabla F(x_n)  } = \ps{\nabla F(x_n),\hat m_n/(\varepsilon+\sqrt{\hat v_n})}$
%and $\ps{D_n, m_n^{\odot 2} } = 2a_nP_n$.
As $a_n\to a$, we have $a_n-\frac{b+2\delta}{4}\geq a-\frac{b+\delta}{4}$ for all $n$ large enough. Hence,
\begin{align*}
 \bE_n P_{n+1}-P_n&\leq
u_nP_{n+1} -2(a-\frac{b+\delta}{4})\gamma_n P_n + \gamma_n\ps{\nabla F(x_n),\frac{\hat m_n}{\varepsilon+\sqrt{\hat v_n}}}
\\
&+
\gamma_n^2\frac{b+2\delta}2\ps{\nabla F(x_n),\frac{\hat m_n}{\varepsilon+\sqrt{\hat v_n}}}
+C\ps{\frac {D_n}{2a_n}, \bE_n[(m_{n+1}-m_n)^{\odot 2}] }\,.
\end{align*}
Using the Cauchy-Schwartz inequality and Assumption~\ref{hyp:stab}~\ref{momentgrowth}, %Eq.~(\ref{eq:grows}),
it is easy to show the inequality
$\ps{\nabla F(x_n),\frac{\hat m_n}{\varepsilon+\sqrt{\hat v_n}}}\leq C(1+F(x_n)+P_n)$.
Moreover, using the componentwise inequality $(\nabla f_{n+1}-m_n)^{\odot 2} \leq 2 \nabla f_{n+1}^{\odot 2} + 2 m_n^{\odot 2}$
along with Assumption~\ref{hyp:stab}~\ref{momentgrowth}, we obtain
\begin{equation*}
  \resizebox{\hsize}{!}{$
  \ps{\frac {D_n}{2a_n}, \bE_n[(m_{n+1}-m_n)^{\odot 2}] } \leq
2(1-\alpha_{n+1})^2\ps{\frac {D_n}{2a_n},
\bE_n[ \nabla f_{n+1}^{\odot 2}]+ m_n^{\odot 2}
}
\leq C\gamma_n^2(1+F(x_n)+P_n)\,.
$}
\end{equation*}
%\begin{align*}
%  \ps{\frac {D_n}{2a_n}, \bE_n[(m_{n+1}-m_n)^{\odot 2}] }&\leq
%2(1-\alpha_{n+1})^2\ps{\frac {D_n}{2a_n},
%\bE_n[ \nabla f_{n+1}^{\odot 2}]+ m_n^{\odot 2}
%}\\
%&\leq C\gamma_n^2(1+F(x_n)+P_n)\,.
%\end{align*}
Putting all pieces together with Eq.~(\ref{eq:lip}),
\begin{equation}
\label{eq:F+P}
\resizebox{0.95\hsize}{!}{$
 \bE_n(F(x_n)+ P_{n+1}) \leq F(x_{n-1})+P_{n}
 +u_nP_{n+1} -2(a-\frac{b+\delta}{4})\gamma_n P_n
+C\gamma_n^2(1+F(x_n)+P_n)\,.
$}
\end{equation}
Define
$
V_n\eqdef (1-C\gamma_{n-1}^2)F(x_{n-1})+(1-u_{n-1})P_{n}
$
where the constant $C$ is fixed so that
%where $C$ is fixed so that
Eq.~(\ref{eq:F+P}) holds.
Then,
\begin{equation*}
  \bE_n(V_{n+1}) \leq V_n
 -\left(2a-\frac{b+\delta}{2}-\frac{u_{n-1}}{\gamma_n}\right)\gamma_n P_n
+C\gamma_n^2(1+P_n)+ C\gamma_{n-1}^2F(x_{n-1})\,.
\end{equation*}
%By Assumption~\ref{hyp:stab}, $\lim\sup_n u_{n-1}/\gamma_n< 2a-b/2$ and choosing $\delta$ small enough, we obtain
By Assumption~\ref{hyp:stab}, $\lim\sup_n u_{n-1}/\gamma_n< 2a-b/2$ and for $\delta$ small enough, we obtain
\begin{equation*}
  \bE_n(V_{n+1}) \leq V_n
+C\gamma_n^2(1+P_n)+ C\gamma_{n-1}^2F(x_{n-1})\leq (1+ C'\gamma_n^2)V_n +C\gamma_n^2\,.
\end{equation*}
By the Robbins-Siegmund's theorem  \cite{robbins1971convergence},
the sequence $(V_n)$ converges almost surely to a finite random variable $V_\infty \in \bR^+$.
In turn, the coercivity of $F$ implies that $(x_n)$ is almost surely bounded.
We now establish the almost sure boundedness of $(m_n)$.
Consider the martingale difference sequence $\Delta_{n+1}\eqdef \nabla f_{n+1} -\nabla F(x_n)$.
We decompose $m_{n} = \bar m_n + \tilde m_n$ where
$\bar m_{n+1} = \alpha_{n+1} \bar m_n + (1-\alpha_{n+1})\nabla F(x_n)$ and
$\tilde m_{n+1} = \alpha_{n+1} \tilde m_n + (1-\alpha_{n+1})\Delta_{n+1}$, setting $\bar m_0=\tilde m_0=0$.
We prove that both terms $\bar m_n$ and $\tilde m_n$ are bounded. Consider the first term:
$
\|\bar m_{n+1}\|\leq \alpha_{n+1} \|\bar m_n\| + (1-\alpha_{n+1}) \sup_k\|\nabla F(x_k)\|\,.
$
By continuity of $\nabla F$, the supremum in the above inequality is almost surely finite.
Thus, for every $n$, the ratio $\|\bar m_n\|/\sup_k\|\nabla F(x_k)\|$ is upperbounded by the bounded sequence
$r_n$. Hence, $(\bar m_n)$ is bounded w.p.1. Consider now the term $\tilde m_n$:
\begin{equation*}
    \resizebox{\hsize}{!}{$
  \bE_n(\|\tilde m_{n+1}\|^2)  = \alpha_{n+1}^2\|\tilde m_n\|^2 +  (1-\alpha_{n+1})^2\bE_n(\|\Delta_{n+1}\|^2)
                               \leq  (1+(1-\alpha_{n+1})^2)\|\tilde m_n\|^2 +  (1-\alpha_{n+1})^2C\,,
    $}
\end{equation*}
%\begin{align*}
%  \bE_n(\|\tilde m_{n+1}\|^2) &= \alpha_{n+1}^2\|\tilde m_n\|^2 +  (1-\alpha_{n+1})^2\bE_n(\|\Delta_{n+1}\|^2)\\
%&\leq  (1+(1-\alpha_{n+1})^2)\|\tilde m_n\|^2 +  (1-\alpha_{n+1})^2C\,,
%\end{align*}
where $C$ is a constant s.t. $\bE_n(\|\nabla f_{n+1}\|^2) \leq C$ by Assumption~\ref{hyp:moment-f}~\ref{momentegal}.
Here, we used $\alpha_{n+1}^2\leq (1+(1-\alpha_{n+1})^2)$ and the inequality
$\bE_n(\|\Delta_{n+1}\|^2) \leq \bE_n(\|\nabla f_{n+1}\|^2)$.
%where we recall that the function $p$, as defined in Assumption~\ref{hyp:moment-growth}, is bounded on bounded sets.
By Assumption~\ref{hyp:stepsizes}, $\sum_n(1-\alpha_{n+1})^2<\infty$. By the Robbins-Siegmund theorem,
it follows that $\sup_n\|\tilde m_n\|^2<\infty$ w.p.1. Finally, it can be shown that $(v_n)$ is
almost surely bounded using the same arguments.

\subsection{Proof of Th.~\ref{thm:clt}}
\label{sec:clt}

We use~\cite[Th.~1]{pelletier1998weak}.
All the assumptions in the latter can be verified in our case, at the exception of
a positive definiteness condition on the
limiting covariance matrix, which corresponds, in our case, to the matrix $Q$
given by Eq.~(\ref{eq:Q}). As $Q$ is not positive definite, it is strictly speaking not possible
to just cite and apply \cite[Th.~1]{pelletier1998weak}. Nevertheless,
a detailed inspection of the proofs of \cite{pelletier1998weak} shows that only a minor
adaptation is needed in order to cover the present case.
%Therefore, it is worthless reprove the convergence result of \cite{pelletier1998weak} from scratch.
Therefore, proving the convergence result of \cite{pelletier1998weak} from scratch is worthless.
It is sufficient to verify the assumptions of \cite[Th.~1]{pelletier1998weak}
(except the definiteness of $Q$) and then to point out the specific part of the proof of \cite{pelletier1998weak}
which requires some adaptation.
% The following lemma is a consequence of \cite[Part 3.4]{pelletier1998weak}.
% \begin{lemma}
%   Let $z^*\in \bR^k$ where $k\geq 1$. On some probability space equipped with a filtration $\cF=(\cF_n)_n$, consider a random sequence on $\bR^k$ given by
% $$
% Z_{n+1} = Z_n + \gamma_{n+1}\bar h(Z_n) + \gamma_{n+1}\eta_{n+1}+\gamma_{n+1}\epsilon_{n+1}\,,
% $$
% where $(\gamma_n)$ is a deterministic sequence, $\bar h:\bR^k\to\bR^k$ is measurable,
% $(\eta_n)$, $(\epsilon_n)$ are adapted to $\cF$. Assume the following conditions.
% \begin{itemize}[leftmargin=24pt]
% \item[\emph{(A1)}] It holds that $\bar h(z^*)=0$. Moreover, $\bar h$ is continuously differentiable in a neighborhood of $z^*$,
% and its Jacobian matrix $\bar H$ at $z^*$ is stable \emph{i.e.}, the largest real part of its eignevalues
% is $(-\bar L)$ with $\bar L>0$.
% \item[\emph{(A2)}] There exists
% \end{itemize}
% \end{lemma}

Let $z_n=(x_n,m_n,v_n)$ be the output of Algorithm~\ref{alg:adam-decreasing}.
Define $z^*=(x^*,0,S(x^*))$.
% Remark that the conditional central limit theorem stated by \cite{pelletier1998weak} shall be valid
% only given the event $\{z_n\to z^*\}$, whereas we seek to prove the convergence result stated in Th.~\ref{thm:clt}
% given the event $\{x_n\to x^*\}$. As a matter of fact, both events coincide up to a $\bP$-negligible set,
% so it is safe to prove the convergence on the event $\{z_n\to z^*\}$.
% The claim that the two events coincide trivially holds if one lets the assumptions of Th.~\ref{thm:as_conv_under_stab} be satisfied,
% because in that case, $m_n\to 0$ and $v_n- S(x_n)\to 0$ a.s., so that $z_n\to z^*$ a.s. on the event $\{x_n\to x^*\}$.
% But the claim is also satisfied if only the assumptions of Th.~\ref{thm:clt} are considered.
% As this point is somewhat technical and maybe non essential in a first reading, we postpone the discussion
% to the end of this section.
Define $\eta_{n+1} \eqdef (0, a(\nabla f_{n+1} - \nabla F(x_n)), b(\nabla f_{n+1}^{\odot 2}- S(x_n)))$.
We have
\begin{equation}
  \label{eq:zn_dec}
  z_{n+1} = z_n + \gamma_{n+1} h_\infty(z_n) + \gamma_{n+1}\eta_{n+1} + \gamma_{n+1} \epsilon_{n+1}\,,
\end{equation}
where $\epsilon_{n+1} \eqdef (\epsilon_{n+1}^1,\epsilon_{n+1}^2,\epsilon^3_{n+1})$, whose components are given by
\begin{equation*}
  \resizebox{\hsize}{!}{$
  \epsilon_{n+1}^1 =  \frac{m_n}{\varepsilon+\sqrt{v_n}} - \frac{\hat m_{n+1}}{\varepsilon+\sqrt{\hat v_{n+1}}};\,
  \epsilon_{n+1}^2 = \left(\frac{1-\alpha_{n+1}}{\gamma_{n+1}}-a\right)\left( \nabla f_{n+1} - m_n\right);\,
  \epsilon_{n+1}^3 =  \left(\frac{1-\beta_{n+1}}{\gamma_{n+1}}-b\right)\left( \nabla f_{n+1}^{\odot 2} - v_n\right)\,.
  $}
\end{equation*}
Here, $\eta_{n+1}$ is a martingale increment noise and
$\epsilon_{n+1} = (\epsilon_{n+1}^1,\epsilon_{n+1}^2,\epsilon^3_{n+1})$ is a remainder term.
The aim is to check the assumptions (A1.1) to (A1.3) of
\cite{pelletier1998weak}, where the role of the quantities ($h$,
$\varepsilon_n$, $r_n$, $\sigma_n$, $\alpha$, $\rho$, $\beta$) in
\cite{pelletier1998weak} is respectively played by the quantities
($h_\infty$, $\eta_n$, $\epsilon_n$, $\gamma_n$, $\kappa$, $1$, $1$) of
the present paper.

Let us first consider Assumption (A1.1) for $h_\infty$.
By construction, $h_\infty(z^*) = 0$. By Assumptions~\ref{hyp:mean_field_tcl} and \ref{hyp:S>0},
$h_\infty$ is continuously differentiable in the neighborhood of $z^*$ and its
Jacobian at $z^*$ coincides with the matrix $H$ given by Eq.~(\ref{eq:H}).
%As already discussed, after straightforward algebra,
As already discussed, after some algebra,
it can be shown that the largest real part of the eigenvalues of $H$ coincides with $-L$ where $L>0$
is given by Eq.~(\ref{eq:L}).
Hence, Assumption (A1.1) of \cite{pelletier1998weak}
is satisfied for $h_\infty$. Assumption (A1.3) is trivially satisfied
using Assumption~\ref{hyp:step-tcl}. The crux is therefore to verify Assumption (A1.2).
Clearly, $\bE(\eta_{n+1}|\cF_n)=0$. Using Assumption~\ref{hyp:moment-f}\ref{momentreinforce},
it follows from straightforward manipulations based on Jensen's inequality
that for any $M>0$, %s.t. $B(z^*,M)\subset \cV$,
there exists $\delta>0$ s.t.
$%$$
\sup_{n\geq 0}\bE_n\left(\|\eta_{n+1}\|^{2+\delta}\right) \1_{\{\|z_n-z^*\|\leq M\}}<\infty\,.
$ %$$
Next, we verify the condition
\begin{equation}
  \label{eq:cond-reste}
  \lim_{n\to\infty} \bE\left(\gamma_{n+1}^{-1}\|\epsilon_{n+1}\|^2\1_{\{\|z_n-z^*\|\leq M\}}\right) = 0\,.
\end{equation}
It is sufficient to verify the latter for $\epsilon^i_n$ ($i=1,2,3$) in place of $\epsilon_n$.
The map $(m,v)\mapsto m/(\varepsilon+\sqrt{v})$ is Lipschitz continuous in a neighborhood of
$(0,S(x^*))$ by Assumption~\ref{hyp:S>0}. Thus, for $M$ small enough, there exists a constant $C$ s.t.
if $ \|z_n-z^*\|\leq M$, then
%$$
%  \|\epsilon_{n+1}^1\| \leq C\left\|\frac{m_{n+1}}{r_{n+1}}-m_n\right\|+C\left\|\frac {v_{n+1}}{\bar r_{n+1}}-v_n\right\| \,.
%$$
$%$$
  \|\epsilon_{n+1}^1\| \leq C\left\|r_{n+1}^{-1}m_{n+1}-m_n\right\|+C\left\|\bar r_{n+1}^{-1} v_{n+1}-v_n\right\|\,.
$ %$$
Using the triangular inequality and the fact that $r_{n+1},\bar r_{n+1}$ are bounded sequences away from zero, there exists an other constant $C$ s.t.
\begin{align*}
  \|\epsilon_{n+1}^1\| &\leq C\left\|m_{n+1}-m_n\right\|+C\left\|v_{n+1}-v_{n}\right\|
+C|r_{n+1}-1|+C|\bar r_{n+1}-1|\,.
\end{align*}
%Using standard manipulations of series along with Assumption~\ref{hyp:step-tcl}, it can be shown that
Using Lemma~\ref{lemma:r_n} under Assumption~\ref{hyp:step-tcl} (note that $\gamma_0 > 1/2L \geq 1/a$ when $\kappa=1$),
we obtain that the sequence $|r_{n}-1|/\gamma_n$ is bounded, thus $|r_{n+1}-1|\leq C\gamma_{n+1}$.
%{\color{red} voir si on a la place de le mettre dans un lemme.}
 The sequence $(1-\alpha_{n})/\gamma_n$ being also bounded, it holds that
%$\|m_{n+1}-m_n\|\leq C\gamma_{n+1}\|\nabla f_{n+1}-m_n\|$, and if $ \|z_n-z^*\|\leq M$,
\begin{align*}
  \|m_{n+1}-m_n\|^2 \1_{\{\|z_n-z^*\|\leq M\}} \leq  C\gamma_{n+1}^2(1+\|\nabla f_{n+1}\|^2 )\1_{\{\|z_n-z^*\|\leq M\}}\,.
\end{align*}
By Assumption~\ref{hyp:moment-f}~\ref{momentreinforce}, %Assumption~\ref{hyp:mean_field_tcl},
$\bE_n(\|\nabla f_{n+1}\|^2|)$ is bounded by a deterministic constant on $\{\|z_n-z^*\|\leq M\}$.
Thus, $\bE_n(\|m_{n+1}-m_n\|^2\1_{\{\|z_n-z^*\|\leq M\}})\leq C\gamma_{n+1}^2$.
A similar result holds for $\|v_{n+1}-v_n\|^2$. We have thus shown that
$ \bE_n\left(\|\epsilon_{n+1}^1\|^2\1_{\{\|z_n-z^*\|\leq M\}}\right)\leq C\gamma_{n+1}^2$. Hence,
Eq.~(\ref{eq:cond-reste}) is proved for $\epsilon_{n+1}^1$ in place of $\epsilon_{n+1}$.
%As far as $\epsilon_{n+1}^2$, $\epsilon_{n+1}^3$ are concerned, the proof uses the same kind of arguments
%under Assumption~\ref{hyp:step-tcl} and is omitted.
Under Assumption~\ref{hyp:step-tcl}, the proof uses the same kind of arguments for $\epsilon_{n+1}^2$, $\epsilon_{n+1}^3$  and is omitted.
Finally, Eq.~(\ref{eq:cond-reste}) is proved.
Continuing the verification of Assumption (A1.2), we establish that
\begin{equation}
  \label{eq:lim-cov}
  \bE_n(\eta_{n+1}\eta_{n+1}^T) \to Q\textrm{ a.s. on } \{z_n\to z^*\}\,.
\end{equation}
Denote by $\bar Q(x)$ the matrix given by the righthand side of Eq.~(\ref{eq:Q}) when $x^*$
is replaced by an arbitrary $x\in \cV$. It is easily checked that $\bE_n(\eta_{n+1}\eta_{n+1}^T)=\bar Q(x_n)$
and by continuity, $\bar Q(x_n)\to Q$ a.s. on $\{z_n\to z^*\}$, which proves (\ref{eq:lim-cov}).
Therefore, Assumption (A1.2) is fulfilled, except for the point mentioned at the beginning of this section : %paragraph:
\cite{pelletier1998weak} puts the additional condition that the limit
matrix in Eq.~(\ref{eq:lim-cov}) is positive definite. This condition is not satisfied in our case,
but the proof can still be adapted. The specific part of the proof where the positive definiteness
comes into play is Th.~7 in \cite{pelletier1998weak}. The proof of \cite[Th.~1]{pelletier1998weak} can therefore
be adapted to the case of a positive semidefinite matrix.
%by substituting the following
%result to Th.~7 of the same paper.
In the proof of \cite[Th.~7]{pelletier1998weak}, we only substitute the inverse of the square root of $Q$ by the Moore-Penrose inverse.
Finally, the uniqueness of the stationary distribution $\mu$ and its expression follow from \cite[Th.~6.7, p. 357]{karatzas-shreve1991}.
\noindent\textbf{Proof of Eq.~(\ref{eq:cov})}.
We introduce the $d\times d$ blocks of the $3d\times 3d$ matrix %$\Sigma$ by
$\Sigma = \left(\Sigma_{i,j}\right)_{i,j=1,2,3}$
where $\Sigma_{i,j}$ is $d\times d$. We denote by $\tilde\Sigma$ the $2d\times 2d$ submatrix $\tilde\Sigma \eqdef\left(\Sigma_{i,j}\right)_{i,j=1,2}$.
By Th.~\ref{thm:clt}, we have the subsystem:
\begin{equation}
\tilde H \tilde\Sigma + \tilde\Sigma\tilde H^T =
\begin{pmatrix}
  0 & 0 \\
0 & -a^2 \tilde Q
\end{pmatrix}\qquad\text{where }\tilde H \eqdef
\begin{pmatrix}
  \zeta I_d & -D \\
a\nabla^2F(x^*) & (\zeta-a) I_d
\end{pmatrix}\label{eq:lyap-reduced}
\end{equation}
and where $\tilde Q \eqdef \textrm{Cov}\left(\nabla
  f(x^*,\xi)\right)$.  The next step is to triangularize the matrix
$\tilde H$ in order to decouple the blocks of $\tilde\Sigma$.  For
every $k=1,\dots,d$, set
$\nu_k^\pm \eqdef -\frac{a}{2}\pm\sqrt{a^2/4-a\lambda_k}$
with the convention that $\sqrt{-1} =
\imath$ (inspecting the characteristic polynomial of $\tilde H$, these %quantities
are the eigenvalues of $\tilde H$). Set
$M^\pm\eqdef\diag{(\nu_1^\pm,
\cdots,\nu_d^\pm)}$ and $R^\pm\eqdef
D^{-1/2}PM^\pm P^TD^{-1/2}$.  Using the identities $M^++M^-=-a I_d$ and
$M^+M^-=a\Lambda$ where
$\Lambda\eqdef\diag{(\lambda_1,\cdots,\lambda_d)}$, it can be checked
that
$$
\cR\tilde H =\begin{pmatrix}
  D R^+ + \zeta I_d & -D \\ 0 & R^-D + \zeta I_d
\end{pmatrix}\cR,\text{ where }\cR\eqdef
\begin{pmatrix}
  I_d & 0 \\ R^+ & I_d
\end{pmatrix}\,.
$$
Set $\check \Sigma \eqdef \cR\tilde\Sigma\cR^T$. Denote by
$(\check \Sigma_{i,j})_{i,j=1,2}$ the blocks of $\check\Sigma$.
Note that $\check\Sigma_{1,1} = \Sigma_{1,1}$. By left/right multiplication of Eq.~(\ref{eq:lyap-reduced})
respectively with $\cR$ and $\cR^T$, we obtain
%% $$
%% \cT\check\Sigma + \check\Sigma\cT^T =
%% \begin{pmatrix}
%%   0 & 0 \\
%% 0 & -a^2 \tilde Q
%% \end{pmatrix}\,.
%% $$
%% By expanding the above system, we obtain
\begin{align}
  &(DR^++\zeta I_d) \Sigma_{1,1}+\Sigma_{1,1}(R^+D+\zeta I_d) = \check\Sigma_{1,2} D+ D\check\Sigma_{1,2}^T \label{eq:lapremiere}\\
& (DR^++\zeta I_d) \check\Sigma_{1,2}+\check\Sigma_{1,2} (DR^-+\zeta I_d) = D\check\Sigma_{2,2} \label{eq:ladeuxieme}\\
& (R^-D+\zeta I_d) \check\Sigma_{2,2} + \check\Sigma_{2,2} (DR^-+\zeta I_d) = -a^2\tilde Q \label{eq:laderniere}
\end{align}
Set $\bar \Sigma_{2,2} = P^{-1}D^{1/2}\check\Sigma_{2,2} D^{1/2}P$.
Define $C\eqdef P^{-1}D^{1/2}\tilde Q D^{1/2}P$.
Eq.~(\ref{eq:laderniere}) yields $(M^-+\zeta I_d) \bar\Sigma_{2,2} + \bar\Sigma_{2,2} (M^-+\zeta I_d) = -a^2C$.
%% The component $(k,\ell)$ is given by
%% $$
%% \bar\Sigma_{2,2}^{k,\ell} = -a^2\frac{C_{k,\ell}}{\nu_k^-+\nu_\ell^-}\,.
%% $$
Set $\bar \Sigma_{1,2} = P^{-1}D^{-1/2}\check\Sigma_{1,2} D^{1/2}P$.
Eq.~(\ref{eq:ladeuxieme}) rewrites $(M^++\zeta I_d) \bar \Sigma_{1,2}+\bar \Sigma_{1,2} (M^-+\zeta I_d) = \bar \Sigma_{2,2}$.
We obtain that %The component $(k,\ell)$ is given by
$%$$
\bar\Sigma_{1,2}^{k,\ell} = (\nu_k^++\nu_\ell^-+2\zeta)^{-1}\bar\Sigma_{2,2}^{k,\ell} = \frac{-a^2C_{k,\ell}}{(\nu_k^++\nu_\ell^-+2\zeta)(\nu_k^-+\nu_\ell^-+2\zeta)}\,.
$ %$$
Set %finally
$\bar \Sigma_{1,1} = P^{-1}D^{-1/2}\Sigma_{1,1} D^{-1/2}P$.
Eq.~(\ref{eq:lapremiere}) becomes $(M^++\zeta I_d)\bar \Sigma_{1,1} + \bar \Sigma_{1,1}(M^++\zeta I_d) = \bar \Sigma_{1,2} + \bar \Sigma_{1,2}^T$. Thus,
%\begin{equation*}
%\resizebox{.99\hsize}{!}{$
\begin{align*}
\bar\Sigma_{1,1}^{k,\ell} &
\resizebox{.9\hsize}{!}{$
= \frac{\bar\Sigma_{1,2}^{k,\ell}+\bar\Sigma_{1,2}^{\ell,k}}{\nu_k^++\nu_\ell^++2\zeta}
=  \frac{-a^2C_{k,\ell}}{(\nu_k^++\nu_\ell^++2\zeta)(\nu_k^-+\nu_\ell^-+2\zeta)}\left(\frac{1}{\nu_k^++\nu_\ell^-+2\zeta}
+\frac{1}{\nu_k^-+\nu_\ell^++2\zeta}\right)
$}
\\ &
%\resizebox{.6\hsize}{!}{$= \frac{C_{k,\ell}}{ (1 - \frac{2\zeta}{a})(\lambda_k+\lambda_\ell-2\zeta + \frac 2a \zeta^2) +\frac 1{2(a-2\zeta)}(\lambda_k-\lambda_\ell)^2}\,.$}
= \frac{C_{k,\ell}}{ (1 - \frac{2\zeta}{a})(\lambda_k+\lambda_\ell-2\zeta + \frac 2a \zeta^2) +\frac 1{2(a-2\zeta)}(\lambda_k-\lambda_\ell)^2}\,,
%\,\,\text{and  the result is proved.}
\end{align*}
and  the result is proved.

\bibliographystyle{siamplain}
\bibliography{math}
\end{document}